\documentclass[11 pt]{article}
\usepackage[utf8]{inputenc}

\usepackage[a4paper, total={6in, 8in}]{geometry}
\usepackage{amsmath,amsthm,mathtools}
\usepackage{amssymb}
\usepackage{amsfonts}
\usepackage{color}
\usepackage[linktoc=page, colorlinks, linkcolor=blue, citecolor=blue]{hyperref}
\usepackage{bookmark}
\usepackage{xcolor}

\usepackage{float}

\usepackage{algorithmic}
\usepackage{algorithm, algorithmic}

\usepackage{bm}
\usepackage{fullpage}

\mathtoolsset{showonlyrefs}

\newtheorem{theorem}{Theorem}[section]
\newtheorem{lemma}[theorem]{Lemma} 
\newtheorem{remark}[theorem]{Remark}
\newtheorem{definition}[theorem]{Definition}
\newtheorem{proposition}[theorem]{Proposition}

\providecommand{\keywords}[1]
{
  \small	
  \textbf{{Keywords:}} #1
}

\newcommand{\R}{\mathbb{R}}

\newcommand{\N}{\mathbb{N}}
\newcommand{\EE}{\mathbb{E}}
\newcommand{\EEb}{\mathbf{E}}

\newcommand{\overleq}[1]{\overset{#1}{\le}}
\newcommand{\overeq}[1]{\overset{#1}{=}}
\newcommand{\bracing}[2]{\underset{#1}{\underbrace{#2}}}

\newcommand{\norm}[1]{\Vert #1 \Vert}
\newcommand{\twonorm}[1]{\big\Vert #1 \big\Vert_2}
\newcommand{\specnorm}[1]{\big\Vert #1 \big\Vert}
\newcommand{\fronorm}[1]{\big\Vert #1 \big\Vert_F}

\newcommand{\innerproduct}[1]{\langle #1 \rangle}

\makeatletter
\def\@tvsp{\mathchoice{{}\mkern-4.5mu}{{}\mkern-4.5mu}{{}\mkern-2.5mu}{}}
\def\ltrivert{\left|\@tvsp\left|\@tvsp\left|}
\def\rtrivert{\right|\@tvsp\right|\@tvsp\right|}
\makeatother
\newcommand\triplenorm[1]{\ltrivert #1 \rtrivert}

\newcommand{\AAi}{\mathbf{A}_i}

\newcommand{\AAiw}{\mathbf{A}_{i,\ww}}
\newcommand{\AAf}{\mathbf{A}}
\newcommand{\ZZ}{\mathbf{Z}}
\newcommand{\BB}{\mathbf{B}}
\newcommand{\UU}{\mathbf{U}}
\newcommand{\RR}{\mathbf{R}}
\newcommand{\OO}{\mathbf{O}}
\newcommand{\VV}{\mathbf{V}}
\newcommand{\WW}{\mathbf{W}}
\newcommand{\UUt}{\mathbf{U}_t}
\newcommand{\UUtw}{\mathbf{U}_{t,\ww} }
\newcommand{\UUtplusw}{\mathbf{U}_{t+1,\ww} }
\newcommand{\XXstar}{\mathbf{X}_{\star}}
\newcommand{\UUtT}{\mathbf{U}_t^\top}
\newcommand{\UUtplus}{\mathbf{U}_{t+1}}
\newcommand{\UUstar}{\UU_{\star}}

\newcommand{\XX}{\mathbf{X}}
\newcommand{\VXX}{\mathbf{V}_{\XXstar}}
\newcommand{\VUUt}{\mathbf{V}_{\UUt}}
\newcommand{\VUUtplus}{\mathbf{V}_{\UU_{t+1}}}
\newcommand{\VUUtP}{\mathbf{V}_{\UUt}^\top}
\newcommand{\VXXT}{\mathbf{V}_{\XXstar}^\top}
\newcommand{\VXXP}{\mathbf{V}_{\XXstar,\bot}}
\newcommand{\VXXPT}{\mathbf{V}_{\XXstar,\bot}^\top}
\newcommand{\MM}{\mathbf{M}}

\newcommand{\NN}{\mathbf{N}}
\newcommand{\LL}{\mathbf{L}}
\newcommand{\SSigma}{\mathbf{\Sigma}}
\newcommand{\DDelta}{\mathbf{\Delta}}
\newcommand{\DDeltat}{\mathbf{\Delta}_t}
\newcommand{\DDeltatw}{\mathbf{\Delta}_{t,\ww}}
\newcommand{\Id}{\textbf{Id}}
\newcommand{\IdOp}{\mathcal{I}}
\newcommand{\vv}{\mathbf{v}}
\newcommand{\uu}{\mathbf{u}}
\newcommand{\yy}{\mathbf{y}}
\newcommand{\xx}{\mathbf{x}}
\newcommand{\ww}{\mathbf{w}}
\newcommand{\yyi}{\mathbf{y}_i}
\newcommand{\LLambda}{\mathbf\Lambda}
\newcommand{\xxi}{\boldsymbol{\xi}}
\newcommand{\xxii}{\boldsymbol{\xi}}

\newcommand{\dist}{\text{dist}}

\newcommand{\Aop}{\mathcal{A}}
\newcommand{\Aopw}{\mathcal{A}_\ww}
\newcommand{\Aops}{\Aop^* \Aop}
\newcommand{\Aopws}{\Aopw^* \Aopw}

\newcommand{\Projw}{\mathcal{P}_{\ww \ww^\top}}
\newcommand{\Projwperp}{\mathcal{P}_{\ww \ww^\top,\bot}}

\newcommand{\epscover}{\mathcal{N}_\varepsilon}
\newcommand{\supw}{\underset{\ww \in \epscover}{\sup}}

\newcommand{\constone}{c_1}
\newcommand{\consttwo}{c_2}
\newcommand{\constthree}{c_3}
\newcommand{\constfour}{c_4}
\newcommand{\constfive}{c_5}
\newcommand{\constsix}{c_6}

\title{
Non-convex matrix sensing: 
Breaking the quadratic rank barrier in the sample complexity\footnote{ A preliminary version of this article was presented at the 38th Annual Conference on Learning Theory (COLT 2025).}
}
\author{Dominik St\"oger\thanks{
    MIDS
    (Mathematical Institute for Machine Learning and Data Science),
    KU Eichst\"att-Ingolstadt
} 
\and 
Yizhe Zhu\thanks{Department of Mathematics,
University of Southern California}}
\date{\today}

\begin{document}

\maketitle

\begin{abstract}
For the problem of reconstructing 
a low-rank matrix from a few linear measurements,
two classes of algorithms have been widely studied in the literature:
convex approaches based on nuclear norm minimization,
and non-convex approaches that use factorized gradient descent.
Under certain statistical model assumptions, 
it is known that nuclear norm minimization recovers the ground truth
as soon as the number of samples scales linearly with the number of degrees of freedom of the ground-truth.
In contrast, while non-convex approaches are computationally less expensive, 
existing recovery guarantees assume 
that the number of samples scales at least quadratically with the rank $r$ of the ground-truth matrix.
In this paper, we close this gap
by showing that the non-convex approaches can be as efficient as nuclear norm minimization in terms of sample complexity.
Namely, we consider the problem of reconstructing a positive semidefinite matrix from a few Gaussian measurements. 
We show that factorized gradient descent with spectral initialization
converges to the ground truth at a linear rate as soon as the number of samples scales with $ \Omega (rd\kappa^2)$, where $d$ is the dimension, and $\kappa$ is the condition number of the ground truth matrix.
This improves the previous rank-dependence in the sample complexity of non-convex matrix factorization from quadratic to linear. Furthermore, we extend our theory to the noisy setting, where we show that with noisy measurements, factorized gradient descent with spectral initialization converges to the minimax optimal error up to a factor linear in $\kappa$.
Our proof relies on a probabilistic decoupling argument,
where we show that the gradient descent iterates are only
weakly dependent on the individual entries of the measurement matrices.
We expect that our proof technique is of independent interest for other non-convex problems.
\end{abstract}



\section{Introduction}\label{sec:introduction}

Low-rank matrix recovery 
refers to the problem of reconstructing
an unknown matrix $\XXstar \in \R^{d_1 \times d_2}$ with 
$\text{rank} (\XXstar)=:r \ll \min \left\{ d_1; d_2 \right\}$
from an underdetermined linear set of equations of the form
\begin{equation*}
   \yy
   =
   \Aop (\XXstar) \in \R^m,
\end{equation*}
where $\Aop$ represents a known linear measurement operator
and $\yy \in \R^m$ are the observations.
This ill-posed inverse problem has been the topic of intense study over many years, given its relevance to a variety of questions in machine learning, signal processing, and statistics.
Notable applications include matrix completion \cite{candes2012exact}, phase retrieval \cite{candesPhaseLift}, robust PCA \cite{candes2011robust},
blind deconvolution \cite{ahmedblinddeconv} and its extension to blind demixing \cite{lingblindemixing}.
A major goal has been to develop methods 
which are \textit{sample-efficient};
that is, they can reconstruct the low-rank matrix $\XXstar$ 
if the number of observations $m$ is roughly of the same order as 
the number of degrees of freedom of $\XXstar$.
In addition, these methods should also be scalable,
meaning they remain computationally efficient as the
problem dimensions are increasing.

Several different algorithmic approaches to solve this problem have been proposed.
One line of research revolves around the idea of convex relaxation.
Here, the nuclear norm $\Vert \cdot \Vert_{\ast}$,
i.e., the sum of singular values,
is considered
as a convex proxy for the rank function.
For many problem classes, including 
matrix sensing \cite{rechtfazelnuclearnormmin}, matrix completion \cite{candesPower2010,grossmatrixcompletion}, and blind deconvolution and demixing \cite{jungblindemixing},
it has been shown that this approach is able to recover the unknown matrix $\XXstar$
as soon as the number of samples $m$ scales, up to logarithmic factors,
with the information-theoretically optimal sample complexity $r(d_1+d_2)$.
However, a drawback of these convex approaches 
is that they tend to be computationally prohibitive.

For this reason,
many studies have considered non-convex heuristics
where one minimizes an objective of the form
\begin{equation}\label{intro:equation}
   f
   (\UU, \VV)
   =
    \sum_{i=1}^m 
    \ell \left( \yyi, \left(\Aop  ( \UU \VV^\top ) \right)_i \right),
\end{equation}
with 
low-rank factors $\UU \in \R^{d_1 \times r}$ and $\VV \in \R^{d_2 \times r}$
and a loss function 
$\ell: \R \times \R \rightarrow \R$.
To minimize the objective function,
local search methods
such as gradient descent or alternating minimization with a suitable initialization are used.
An advantage of these approaches is that they are computationally less demanding
since there are only $r(d_1+d_2)$ optimization variables
instead of at least $d_1d_2$ optimization variables in the convex approaches.
However, due to the non-convexity of the objective function,
it might initially seem unclear that local search methods can find the global minimum
of the objective \eqref{intro:equation} efficiently.

Nevertheless, in recent years
a large body of literature has demonstrated
that under certain statistical assumptions, these methods converge
to the global minimum and are thus able to recover the unknown low-rank matrix $\XXstar$.
For instance, gradient descent with spectral initialization \cite{tu2016low}
and other variants of gradient descent \cite{TongScaledGD,LiZhuSoVidalNonconvex,CharisopoulosNonconvex} have been studied
for matrix sensing and related problems.
Similarly, numerous works have established convergence and recovery guarantees 
for matrix completion 
\cite{keshavanMatrixCompletion,SunMatrixCompletion,zheng2016convergenceNonconvex,ge2016matrix,MaCongNonconvex,ChenLiuLiNonconvex}
and 
blind deconvolution and demixing
\cite{LingDemixingNonconvex,DongShiDemixing}.
In addition, recent studies also analyzed overparameterized models,
where the exact rank $r$ is either not known or where the number of parameters exceeds the number of samples
\cite{li2018algorithmic,stoger2021small,jin2023understanding,xu2023powerOverparameterized,soltanolkotabi2023implicit,ma2024convergence,wind2023asymmetric}.
Beyond gradient descent, also alternating minimization \cite{jain2013low} and 
other non-convex methods based on matrix factorization such as GNMR \cite{ZilberNadlerNonConvex} 
have been proposed and studied.
For a more extensive overview of the literature, we refer the reader to \cite{ChenLiuLiNonconvex}.

Despite this significant body of literature, the existing theoretical guarantees for non-convex methods based on matrix factorization in the literature
are weaker than the corresponding guarantees for nuclear norm minimization in terms of sample complexity.
Namely, in all these results, it is required
that the number of samples $m$ scales at least quadratically with the rank $r$
and thus the total number of samples scales at least with $r^2 (d_1+d_2)$.
This raises the question of whether this quadratic rank-dependence is just an artifact of the proof
or whether it is inherent to the problem, see, e.g., \cite[p. 5264]{chi2019nonconvex}.

In this paper, we resolve this question in the context of symmetric matrix sensing. Under the assumption that $\mathcal{A}$ is a Gaussian measurement operator
and $\XXstar \in \R^{d \times d}$ is symmetric and positive semidefinite,
we show that factorized gradient descent with spectral initialization
is able to recover the unknown matrix $\XXstar$ 
if the number of samples scales with $rd$, 
which, in particular, is linear in the rank of $\XXstar$.
Our proof is based on a novel probabilistic decoupling argument.
Namely, we show that the trajectory of the gradient descent iterates 
depends only weakly on any given generalized entry of the measurement matrices in a suitable sense.
This allows us to prove stronger concentration bounds 
than what would be possible if one were to rely solely on uniform concentration bounds
(such as the Restricted Isometry Property, for example).
To establish this weak dependence, we construct auxiliary virtual sequences
and combine this with an $\varepsilon$-net argument. Furthermore, we extend our theory to the noisy setting, where we show that with noisy measurements, factorized gradient descent with spectral initialization converges to the minimax optimal error up to a factor linear in $\kappa$.
Our novel proof approach paves the way to improved sample complexity bounds for other non-convex algorithms
and beyond.

Finally, 
we note that there are also several non-convex algorithms 
for low-rank matrix recovery that are not explicitly based on matrix factorization formulation 
as in equation \eqref{intro:equation}.
This includes,
for example,
Singular Value Projection \cite{jain2010guaranteed,DingMatrixCompletion},
Normalized Iterative Hard Thresholding \cite{NIHT_Tanner},
Iteratively Reweighted Least Squares (IRLS), see, e.g., \cite{IRLS_fazel,IRLS_Fornasier,IRLS_Kuemmerle1,kummerle2021scalable}, and
Atomic  Decomposition for Minimum Rank Approximation (ADMiRA) \cite{Admira_Kiryung}.
However, since many of these algorithms operate in the full matrix space 
they are less computationally efficient than algorithms based on matrix factorization.
In the case of IRLS, only local convergence guarantees (with explicit convergence rates) are known.
There have also been algorithms studied
that are based on Riemannian optimization, see, e.g., \cite{wei2016guarantees,Riemannian_Vandereycken,Riemannian_olikier2023}.
However, these algorithms require that the sample complexity scales quadratically in the rank $r$.
We believe our work can lead to improved sample size guarantees for these methods as well.

\paragraph{Organization of the paper:}
This paper is structured as follows.
In the remainder of Section \ref{sec:introduction},
we will describe the formal setting and the algorithm, and we will state our main theoretical result, which is Theorem \ref{thm:main}.
In Section \ref{sec:preliminaries}, we discuss some technical preliminaries regarding the Restricted Isometry Property and perturbation bounds for eigenspaces.
In Section \ref{sec:outline}, we discuss the proof strategy, and we introduce the virtual sequences,
which are the main ingredient to establish
that the sample complexity depends only linearly on the rank.
Section \ref{sec:proof} contains the proof of the main result of this paper, Theorem \ref{thm:main}.
We discuss interesting directions for future research in Section \ref{sec:discussions}.

\paragraph{Notation:}
Before we state the problem formulation, we introduce some basic notation.
For a matrix $ \AAf \in \R^{d_1 \times d_2} $,
we denote its transpose by $ \AAf^\top $
and its trace by $\text{trace} (\AAf)$.
For matrices $\AAf,\BB \in \R^{d_1 \times d_2}$,
we define their inner product via 
$ \innerproduct{ \AAf, \BB } := \text{trace} \left( \AAf \BB^\top
 \right)$. 
The Frobenius norm $\fronorm{\cdot}$ denotes the norm induced by this inner product, i.e., $\fronorm{\AAf} := \sqrt{\innerproduct{\AAf, \AAf}} $.
By $\specnorm{\AAf}$ we denote the spectral norm of the matrix $\AAf$, i.e.,
the largest singular value of the matrix $\AAf$.
By $\twonorm{\vv}:= \sqrt{\sum_{i=1}^d \vv_i^2}$
we denote the Euclidean norm of a vector $\vv \in \R^d$.
The set $\mathcal{S}^d \subset \R^{d \times d}$ represents the set of all symmetric matrices. 
The matrix $\Id \in \mathcal{S}^d$ denotes the identity matrix.
Moreover, $\IdOp: \mathcal{S}^d \rightarrow \mathcal{S}^d$ represents the identity mapping.

Furthermore, for a matrix $\AAf \in \R^{d_1 \times d_2}$ of rank $r$
we denote its singular value decomposition by
$\AAf= \VV_{\AAf} \SSigma_{\AAf} \WW_{\AAf}^\top$.
The matrices $ \VV_{\AAf} \in \R^{d_1 \times r}$ and $\WW_{\AAf} \in \R^{d_2 \times r}$ contain the left-singular 
and right-singular vectors of the matrix $\AAf$.
The matrix $\SSigma_{\AAf} \in \R^{r \times r}$ contains the singular values of $\AAf$.
Moreover, $ \VV_{\AAf, \bot} \in \R^{(d_1-r) \times r} $
represents an orthogonal matrix 
whose column span is orthogonal to the column span of $\VV_{\AAf}$.

\subsection{Problem formulation}\label{sec:problemformulation}
In this paper,
we focus on symmetric matrix sensing.
More precisely,
we study the problem of reconstructing a symmetric, positive semidefinite matrix $\XXstar \in \mathbb{R}^{d \times d}$
with rank $r$ from $m$ linear observations of the form
\begin{equation}\label{equ:problemformulation}
    \yyi
    = \frac{1}{\sqrt{m}} \innerproduct{ \AAi, \XXstar }
    := \frac{1}{\sqrt{m}} \text{trace} \left(\AAi \XXstar \right)
    \quad \quad 
    \text{ for }
    i=1,2,\ldots, m.
\end{equation}

\begin{definition}[Measurement operator]\label{def:measurement}
We define the linear measurement operator $\Aop: \mathcal{S}^d \rightarrow \R^m$ by
\begin{align*}
    \left[\Aop (\XX) \right]_i := \frac{1}{\sqrt{m}} \innerproduct{\AAi, \XX}
    \quad \quad
    \text{ for }
    i=1,2,\ldots, m 
\end{align*}
for any matrix $ \XX \in \mathcal{S}^d$.
Recall that $\mathcal{S}^d \subset \R^{d \times d}$
denotes the set of symmetric matrices. The matrices $ \left\{ \AAi \right\}_{i=1}^m \subset \mathbb{R}^{d \times d}
 $ represent known, symmetric measurement matrices.
We assume that their entries are i.i.d. with distribution $ \mathcal{N} \left( 0,1 \right) $ on the diagonal 
and $ \mathcal{N} \left( 0, 1/2 \right)$ on the off-diagonal entries. Each $\AAi$ is also known as a Gaussian orthogonal ensemble \cite{anderson2010introduction}.
\end{definition}

This measurement model has been considered before in, e.g.,
\cite{tu2016low,li2018algorithmic}.
With this notation in place,
equation \eqref{equ:problemformulation} can be written more compactly as
$
    \yy
    =
    \Aop 
    \left(
        \XXstar
    \right)$.
To recover the ground-truth matrix $\XXstar$, we consider the non-convex objective function
\begin{equation}\label{equ:objectivefunction}
  \mathcal{L}( \UU ) 
  := \frac{1}{4}  \twonorm{ \yy- \mathcal{A} \left(\UU \UU^\top \right) }^2
  = \frac{1}{4}  \twonorm{ \mathcal{A} \left( \XXstar - \UU \UU^\top \right) }^2,
\end{equation}
where $\UU \in \R^{d \times r}$ is a matrix
and $\Vert \cdot \Vert_2$ denotes the $\ell_2$-norm of a vector.
To minimize this objective, 
we follow the two-stage approach introduced in \cite{keshavanMatrixCompletion}
for matrix completion,
which then subsequently was studied for matrix sensing in \cite{tu2016low}.
In the first stage, an initialization $\UU_0$ is constructed
via a so-called spectral initialization.
This initialization is subsequently used as a starting point for the gradient descent scheme in the second stage.
To precisely define the spectral initialization,
we denote by $\Aop^*: \R^m \rightarrow \mathcal{S}^d$ the adjoint operator of $\Aop$
with respect to the trace inner product defined in equation \eqref{equ:problemformulation}.

With this definition in place, we can consider the eigendecomposition of the matrix 
\begin{align}
    \Aop^* (\yy)
    &=: \widetilde\VV\widetilde\LLambda \widetilde\VV^\top,
\end{align}
where $\widetilde{\VV} \in \R^{d \times d}$ is an orthogonal matrix
and the matrix $\widetilde\LLambda \in \R^{d \times d}$ is diagonal matrix
which contains the eigenvalues of $\Aop^* (\yy)$ sorted by their magnitude, i.e., 
$ \vert \lambda_1 (\Aop^* (\yy)) \vert
 \ge \vert \lambda_2 \left( \Aop^* (\yy) \right) \vert
  \ge \ldots 
\ge \vert \lambda_d \left( \Aop^* (\yy) \right) \vert$. 

Since the measurement matrices $\AAi$ are Gaussian, we have that 
\begin{align*}
    \EE \left[ \Aop^* (\yy) \right]
    =
    \EE \left[ \left(\Aops\right) (\XXstar)  \right]
    =
    \XXstar.
\end{align*}
Since $\XXstar$ has rank $r$, for a large enough sample size $m$, one has
 that the truncated rank-$r$ eigendecomposition of $\Aop^* (\yy)$
fulfills
$ \widetilde{\VV}_r \widetilde{\LLambda}_r \widetilde{\VV}_r \approx \XXstar $.
Here, 
by $\widetilde{\VV}_r \in \R^{d \times r}$ we denote a matrix
which contains the first $r$ columns of $\widetilde\VV$
and by $\widetilde{\LLambda}_r$ we denote a diagonal matrix
which contains the largest $r$ eigenvalues 
of $\Aop^* \left( \yy \right)$ in decreasing order.
Motivated by this observation, the spectral initialization $\UU_0$ is defined as 
\begin{align*}
    \UU_0 
    :=\widetilde \VV_r{\widetilde\LLambda_r}^{1/2}.
\end{align*}
Here, the entries of the diagonal matrix $\widetilde\LLambda_r^{1/2}$ are given by
$  \sqrt{ \vert \lambda_i \left( \Aop^* (\yy) \right) \vert } $.
As we will see, all entries of $\widetilde\LLambda_r$
are positive with high probability.

After having computed the initialization $\UU_0$,
we use $\UU_0$ as a starting point of the gradient descent scheme in the second stage,
which is defined as follows
\begin{align*}
    \UU_{t+1}:=\UU_t - \mu \nabla \mathcal{L} (\UU_t)
    \quad
    \text{ for }
    t=0,1,\ldots, 
\end{align*} 
where $\mu >0$ denotes the step size.
A direct computation shows that
\begin{align}
    \UUtplus 
    &=
    \UUt 
    + \mu \left[ (\Aops) \left(  \XXstar - \UUt \UUt^\top \right)  \right] \UUt \label{equ:gradientdescentdefinition}\\
    &=
    \UUt 
    + \frac{\mu}{m} \sum_{i=1}^m \innerproduct{\AAi, \XXstar - \UUt \UUtT} \AAi \UUt. \nonumber
\end{align}
All steps of the two-stage approach are summarized below in Algorithm \ref{algorithm:twostage}.
\begin{algorithm}\label{algorithm:twostage}
    \caption{Two-Stage Approach for Low-Rank Matrix Recovery}
    \label{alg:algo1}
    \begin{algorithmic}
    \STATE{\underline{\textbf{Input:}}
     Measurement operator $\mathcal{A}: \mathcal{S}^d \to \R^{m}$,
     observations $\yy \in \R^m$, step size $\mu>0$
    }
    \STATE{
        \underline{\textbf{Stage 1 (Spectral Initialization):}}
        Compute the truncated eigendecomposition $ \widetilde\VV_r{\widetilde\LLambda_r} \widetilde\VV_r^\top  $
        of the data matrix $
        \mathbf{D} := \mathcal{A}^* (\yy)
            =
            \frac{1}{\sqrt{m}} \sum_{i=1}^m y_i \AAi$.
        Here, $\widetilde\LLambda_r \in \R^{d \times d}$ is the diagonal matrix
        which contains the $r$ largest eigenvalues of the data matrix $\mathbf{D}$ (in absolute value).
        The columns of $ \widetilde\LLambda_r \in \R^{d \times r} $ contain the corresponding eigenvectors.
        Define the initialization $ \UU_0 \in \R^{d \times r} $ via
       $
            \UU_0
            :=
            \widetilde \VV_r{\widetilde\LLambda_r}^{1/2}.
        $
    }
    \STATE{
        \underline{\textbf{Stage 2 (Gradient descent):}}
        \FOR{$t = 0, 1, 2,\ldots$}
        \STATE{
            $
                \UU_{t+1}
                :=
                \UUt
                -
                \mu \nabla \mathcal{L} \left( \UUt \right)
          $
        }
        \ENDFOR
    }
    \end{algorithmic}
\end{algorithm}

\subsection{Main result in the noiseless case}
To formulate our main result,
we need to introduce the condition number of $\XXstar$,
which is defined as
\begin{equation*}
    \kappa:= \frac{\specnorm{\XXstar}}{\sigma_{\min} (\XXstar)}.
\end{equation*}
Here, $\sigma_{\min} (\XXstar)$ denotes the smallest non-zero singular value of $\XXstar$.

Next, let $\UU_{\star} \in \R^{d \times r}$ be a matrix
such that $\XXstar= \UUstar \UUstar^\top$.
The matrix $\UUstar$ is uniquely defined only
up to an orthogonal transformation
$ \RR \in \R^{r \times r} $,
which is why we can only expect to be able to reconstruct $\UUstar$
up to this ambiguity.
To account for this, we will introduce the error metric
\begin{equation}\label{eq:def_dist}
    \text{dist} \left( \UUt, \UUstar \right)
    :=
    \underset{\RR \in \R^{r \times r}, \ 
    \RR^\top \RR = \Id_r}{\min}
    \fronorm{
        \UUt \RR - \UUstar 
    }.
\end{equation}
With this notation in place, we can state the main result of this paper.
\begin{theorem}\label{thm:main}
   Let $ \mathcal{A}: \mathcal{S}^d \to \R^m$ be a linear measurement operator as in Definition~\ref{def:measurement}
   with Gaussian measurement matrices.
   Moreover, let $\XXstar \in \mathcal{S}^d$ be a positive semidefinite matrix of rank $r$.
   Given observations $\yy = \Aop \left( \XXstar \right) \in \R^m $,
   let $\UU_0, \UU_1, \UU_2, \ldots$ 
   be the sequence of gradient descent iterates
   which are obtained via the two-stage approach 
   described in Algorithm \ref{alg:algo1}.
   Assume that the number of observations $m$ satisfies
   \begin{equation*}
    m \ge C rd \kappa^2, 
   \end{equation*} 
   and that the step size $\mu>0$ satisfies 
   \begin{align} \label{eq:range_mu}
   \frac{32}{ 6^d \sigma_{\min} (\XXstar)} \log \left( 16 r\right)\le \mu \le \frac{c_1}{\kappa \Vert \XXstar \Vert}.
   \end{align}
   Then, with probability at least $1-7\exp \left( -d \right)$, it holds 
   for all iterations $ t\ge 0$ that
    \begin{equation*}
        \dist^2 \left( \UUt, \UUstar \right)
        \le
        c_2 r
        \left(1 - c_3 \mu  \sigma_{\min} \left(\XXstar \right) \right)^{t}
        \sigma_{\min} \left( \XXstar \right).
    \end{equation*}
   Here, $C,c_1,c_2,c_3>0$ denote absolute constants.
\end{theorem}

\begin{remark} 
The lower bound in assumption \eqref{eq:range_mu} is rather mild 
since the left-hand side in this inequality converges to $0$ exponentially as the dimension $d$ increases.
If the dimension $d$ is larger than an absolute constant, 
then condition \eqref{eq:range_mu} can always be satisfied for some step size $\mu$.
\end{remark}
    Theorem \ref{thm:main} shows that factorized gradient descent with spectral initialization converges 
    to the ground truth with a linear rate as soon as the number of samples
    scales at least with $rd \kappa^2$.
    In particular, the bound on the sample complexity is linear in the rank $r$.
    This improves over previous results in the matrix sensing literature, which have a sample complexity 
    of order at least $r^2 d \kappa^2$, see, e.g., \cite{tu2016low} or \cite{TongScaledGD}.
    In particular, the sample complexity in Theorem \ref{thm:main} 
    is optimal with respect to the rank $r$ and dimension $d$.
    To the best of our knowledge, this is the first result in the literature 
    which achieves this optimal dependence in the rank for the non-convex low-rank matrix recovery.


        Compared to approaches based on nuclear norm or trace minimization,
    which only need $\Omega (rd)$ samples in the matrix sensing scenario,
    our result is still suboptimal by a factor of $\kappa^2$.
    However, all previous results in the literature on non-convex low-rank matrix recovery based on factorized gradient descent require having at least this quadratic dependence
    on the condition number,
    see, e.g., \cite{TongScaledGD,li2018algorithmic,LiMaChenChiNonconvex}.
    This is also the case for approaches based on alternating minimization
    \cite{jain2013low,hardt_alternatingMin}.
    A notable exception is the work \cite{hardt_alternatingMin}
     in the matrix completion setting, where a non-convex algorithm is carefully designed to only have a logarithmic dependence on the condition number $\kappa$.
    However, the sample complexity scales at least $r^9$ in terms of rank dependence.
    It remains an interesting open problem whether the dependence of
    our algorithm on the sample complexity on the condition number 
    is necessary or an artifact of the proof. 

    Our main result implies that 
    $ \dist \left( \UUt, \UUstar \right) \le \varepsilon  $
    after 
    $O \left(  \frac{ \log \left( r/(\varepsilon \sigma_{\min}(\XXstar))\right)}{\mu \sigma_{\min} \left( \XXstar \right) } \right)$
    iterations.
    Thus, if we choose the largest possible step size $ \mu \asymp 1/(\kappa \specnorm{\XXstar}) $
    we obtain that we reach $\varepsilon$-accuracy
    after
    $O \left(   \kappa^2 \log \left( r/(\varepsilon \sigma_{\min} (\XXstar) )\right) \right)$
    iterations.
    Previous work \cite{tu2016low} allows for a larger step size 
    $\mu \lesssim 1/(\kappa \specnorm{\XXstar})$
    which yields 
    that one can reach $\varepsilon$-accuracy after
    $O \left(   \kappa \log \left( r/(\varepsilon \sigma_{\min} (\XXstar) )\right) \right)$
    iterations,
    whereas Theorem \ref{thm:main} requires $ \mu \lesssim 1/(\kappa \specnorm{\XXstar})$.
    It remains an open problem whether this additional condition number in the step size bound
    can be removed.

\begin{remark}[Connection to other work] 
We compare Theorem~\ref{thm:main} to existing work in the literature:
\begin{itemize}
\item \textbf{Comparison with \cite{tu2016low}}: 
Note that \cite{tu2016low} actually establishes
that $\XXstar$ can be recovered with a nonconvex approach 
that  uses only $O (rd)$ measurements.
Namely, in their work, one performs $\log (r\kappa)$ steps of projected gradient descent in the lifted ($d^2$-dimensional) space after spectral initialization. After that, one performs successive refinements via factorized gradient descent. 
However, the motivation of our work lies in establishing optimal sample complexity for a method that runs with $O(rd)$ optimization variables
and uses matrix factorization.
Thus, this approach cannot be directly compared with ours.

In fact, in \cite{tu2016low}, it was established that after
$O(rd)$ steps of projected gradient descent, one has 
$ \fronorm{ \XXstar -\XX_t } \ll \sigma_{\min} (\XXstar)$,
where $\XX_t$ denotes the projected gradient descent iterate.
After that, the theoretical analysis of factorized gradient descent becomes easier.
By contrast, as can be seen in our proof, the main challenge in our work
is analyzing the first $T$ factorized gradient descent iterations until it holds that
$ \fronorm{ \XXstar -\UU_T \UU_T^\top } \ll \sigma_{\min} (\XXstar)$.
In other words, invoking projected gradient descent as in \cite{tu2016low} allows one to circumvent the initial phase in which the behavior of factorized gradient descent is difficult to analyze.

    \item \textbf{Landscape Analysis:} Several works
\cite{bhojanapalli2016global,park2017non,UschmajewSaddlePoints,zhangRIP}
have shown
that if $m \gtrsim rd$, 
then the loss landscape of 
the objective function $\mathcal{L}$ in \eqref{equ:objectivefunction} 
is benign
in the sense that $\mathcal{L}$ has no spurious local minima 
and all saddle points have at least one direction of strictly negative curvature.
It has been established
that in such a scenario
gradient descent starting from random initialization 
will converge to the ground truth
\cite{LeeSaddlePoints}.
However, these results do not imply any guarantees on the convergence rate or on the computational complexity.
In fact, there exist examples \cite{du2017gradient}
where gradient descent may take exponential time
to escape saddle points.
For this reason, 
the results mentioned above are not directly comparable
to our results.
\end{itemize}
\label{rem:landscapeanalysis}
\end{remark}


\subsection{Main result in the noisy case}

Next, we extend our main result to the case where the measurements are corrupted by noise. 
Assume
\begin{equation*}
    \yy
    =
    \Aop 
    \left(
        \XXstar
    \right)+\xxi,
\end{equation*}
where $\xxi\in \mathbb R^{m}$ is a random vector with i.i.d. entries with distribution $\mathcal N(0,\sigma^2)$. 

We consider the same gradient descent algorithm by inserting the new definition of $\yy$ in Algorithm~\ref{alg:algo1}. Theorem~\ref{thm:mainnoise} provides a recovery guarantee under the Frobenius norm in this setting, and the required number of measurements still scales linearly with $r$.

\begin{theorem}[Noisy measurements]\label{thm:mainnoise}
Let $ \mathcal{A}: \mathcal{S}^d \rightarrow \R^m$ and $\XXstar \in \R^{d \times d}$ be the same as  in Theorem~\ref{thm:main}.
   Given observations $\yy = \Aop \left( \XXstar \right)+\xxi \in \R^m $, where $\xxi\sim N(0, \sigma^2 \Id_m)$,
   let $\UU_0, \UU_1, \UU_2, \ldots$ 
   be the sequence of gradient descent iterates
   obtained in Algorithm \ref{alg:algo1}.
   Assume that the number of observations $m$ satisfies
   \begin{equation*}
    m \ge C rd \kappa^2, 
   \end{equation*} 
   and that the step size  satisfies 
 $\mu \le \frac{c_1}{\kappa \specnorm{\XXstar}}$, and $\sigma\leq \frac{c_2 \sigma_{\min}(\XXstar)}{\sqrt{d}}$ for  sufficiently small constants $c_1, c_2>0$.
 With probability at least $1-C_1\exp(-d)$, for every iteration $t$ with $0 \le t \leq \frac{1}{2} \cdot  6^d$, it holds that
    \begin{equation}\label{ineq:localconv10noise}
        \fronorm{\XXstar - \UUt \UUtT}
        \le
         C_2\sqrt{r} \kappa \sigma_{\min} (\XXstar)  \left(1 - \frac{\mu  \sigma_{\min} (\XXstar) }{16}\right)^{t}
        + C_3\sigma \kappa\sqrt{ rd},
    \end{equation}
   where 
   $C, C_1, C_2, C_3 >0$ are absolute constants.    
   Assume further that 
   $\sigma \gtrsim \exp ( -3^d \mu \sigma_{\min} (\XXstar) ) \frac{\sigma_{\min} (\XXstar)}{\kappa \sqrt{d}}$.
   Then,
   after 
   \begin{equation*}
   T \asymp
   \frac{ \log \left( \frac{\sigma_{\min} (\XXstar)}{\kappa\sqrt{d}\sigma } \right)  }{ \mu \sigma_{\min} (\XXstar)  }
   \end{equation*}
   iterations, it holds that 
   \begin{align}\label{eq:noisy_T}
       \fronorm{
       \XXstar - \UU_T \UU_T^\top
       }
       \lesssim
       \sigma \kappa \sqrt{rd}.
   \end{align}
\end{theorem}

\begin{remark}[Minimax optimal error rate]
 It was shown in \cite{candes2011tight} that for any estimator $\hat{\XX} (\yy)$
 there always exists a rank-$r$ $\XX_{\star}$
 such that
 \begin{align*}
     \fronorm{ \XXstar - \hat{\XX} (\yy) }
     \gtrsim
     \sigma \sqrt{rd}
 \end{align*}
 with probability at least $0.99$.
 Thus, the dependence on $r$, $d$, and $\sigma$  in Equation \eqref{eq:noisy_T} is minimax optimal. 
 In particular, Theorem~\ref{thm:mainnoise} shows that with noisy measurements, factorized gradient descent with spectral initialization converges to the minimax optimal error up to a linear factor of the condition number $\kappa$.
 We note that such a suboptimality of the estimation error in  $\kappa$
 was also observed in other works in non-convex low-rank matrix recovery \cite{chen2020noisy,xu2023powerOverparameterized}. 
\end{remark}

The lower bound on the noise level $\sigma$ is very mild (scaling only as a double‐exponential function of $d$).
It stems from the fact 
that in our proof, we can only control the gradient descent with noise 
up to $3^d$ iterations.
If the noise level is extremely small, one needs more iterations to reach this level of precision.
We believe that this lower bound can be removed with other proof techniques.

The proof of Theorem~\ref{thm:mainnoise} is an 
adaptation of the proof technique we developed for Theorem~\ref{thm:main}, 
and we defer the proof of Theorem~\ref{thm:mainnoise} to Appendix~\ref{appendix:noise}.

\section{Preliminaries}\label{sec:preliminaries}
In the following, we will discuss several technical preliminaries,
which are needed in our proof.
\subsection{The Restricted Isometry Property}\label{sec:RIP}
We first recall the Restricted Isometry Property (RIP).
\begin{definition}[Restricted Isometry Property]\label{definition:RIP}
   The linear measurement operator $\Aop: \mathcal{S}^d \to \R^m$ satisfies the Restricted Isometry Property (RIP),
   of rank $r$ with RIP-constant $\delta_r>0$,
   if it holds for all symmetric matrices $\ZZ \in \R^{d \times d}$ of rank at most $r$
   that 
   \begin{equation}\label{eq:RIP}
    \left(1-\delta_r\right) \fronorm{\ZZ}^2
    \le 
    \twonorm{\mathcal{A} (\ZZ)}^2
    \le 
    \left(1+\delta_r\right) \fronorm{\ZZ}^2.
   \end{equation} 
\end{definition}
In previous works, it was shown that
as soon as the measurement operator $\Aop$ has the RIP,
then convex approaches based on nuclear norm minimization
as well as non-convex approaches are able to recover the ground truth matrix, 
see, e.g., \cite{rechtfazelnuclearnormmin,tu2016low}.

It is well known that as soon as the number of samples $m$ satisfies $m \gtrsim rd$
then the measurement operator $\mathcal{A}$ has the RIP of order $r$ with high probability.
This fact is stated in the following lemma.

\begin{lemma}\label{lem:rank_RIP}
Let $\Aop: \mathcal{S}^d \to \R^m$ be a Gaussian measurement operator
as described in Section \ref{sec:problemformulation}.
Then the RIP constant $\delta_r$ satisfies $\delta_r\leq \delta\leq 1$ with probability $1-\varepsilon$ when 
\begin{align}\label{eq:lowerbound_m}
    m\geq C\delta^{-2} (rd+\log(2\varepsilon^{-1})),
\end{align}
where $C>0$ is a universal constant. In particular, we have with probability at least $1-\exp(-d)$, $m\geq C\delta^{-2} rd$.
\end{lemma}
This lemma differs from similar lemmas in the literature (see, e.g., \cite{candes2011tight}) 
by specifying how $m$ depends on the RIP-constant $\delta$.
A proof of this lemma is provided in Appendix \ref{sec:RIPproof}
together with a more detailed discussion of how this
lemma relates to previous work.

\begin{remark}
The works mentioned in Remark \ref{rem:landscapeanalysis}
have shown that the RIP implies that the optimization  landscape of $\mathcal{L}$ is benign
(in the sense of Remark \ref{rem:landscapeanalysis}).
Moreover, previous work such as \cite{tu2016low} or \cite{TongScaledGD},
which analyzed gradient descent with spectral initialization 
similar to the paper at hand,
relied on their analysis of gradient descent exclusively on the RIP property of the measurement operator $\mathcal{A}$.
As we will explain in Section \ref{sec:outline},
purely relying on the RIP will not suffice
to establish Theorem \ref{thm:main}.
For this reason, in addition to the RIP, we will use the orthogonal invariance of the Gaussian measurement operator $\mathcal{A}$. 
\end{remark}

The RIP has several important consequences, which we will need throughout our proof. We recall them in the following lemma.

\begin{lemma}\label{lemma: RIP}
     Let $ \mathcal{A}: \mathcal{S}^d \rightarrow \R^m$ be a linear measurement operator
     on the set of symmetric matrices as defined above.
    Denote by $\delta_r$ the RIP constant of the operator $\Aop$ of order $r$.
    Then the following statements hold.
    \begin{enumerate}
        \item Let $ \VV \in \R^{d \times r'} $ be any matrix with orthonormal columns, i.e., $ \VV^\top \VV = \Id$.
        Then it holds for any symmetric matrix $\ZZ \in \R^{d \times d}$ of rank at most $r$ that
        \begin{align}\label{ineq:RIPlemma1}
            \fronorm{ \left(\IdOp -\Aops \right) (\ZZ) \VV} \le \delta_{r+2r'} \fronorm{\ZZ}.
        \end{align}
        In particular, it holds that
        \begin{align}\label{ineq:RIPlemma2}
            \specnorm{ \left(\IdOp -\Aops \right) (\ZZ) } \le  \delta_{r+2} \fronorm{\ZZ}.
        \end{align}
        \item Let $\ww \in \R^d$ such that $\twonorm{\ww}=1$. Define the orthogonal projection operators
\begin{align}
    \Projw (\ZZ) 
    &:= \innerproduct{\ww \ww^\top, \ZZ} \ww \ww^\top,  \label{eq:def_projw}\\
    \Projwperp (\ZZ) 
    &:= \ZZ - \innerproduct{\ww \ww^\top, \ZZ} \ww \ww^\top.
\end{align}
        Then it holds for any symmetric matrix $\ZZ \in \R^{d \times d}$ of rank at most $r$ that  
        \begin{align}\label{ineq:RIPlemma3}
            \vert \innerproduct{ \Aop (\ww \ww^\top), \Aop \left( \Projwperp (\ZZ) \right) } \vert 
            \le 
            \delta_{
            r+2} \fronorm{\ZZ}.
        \end{align}
    \end{enumerate}
\end{lemma}
Some variants of these inequalities appeared in the literature already before; see, e.g., \cite{stoger2021small}. 
For completeness, we decided to include a proof in Appendix \ref{sec:proofRIPLemma}.

\begin{remark}
To keep the notation more concise,
we will sometimes drop the subscript
and just use the notation $\delta$ for the RIP constant.
For all results below, the choices of $\delta$ satisfy
$\delta\le \delta_{6r}$
due to the monotonicity of the RIP constant with respect to the rank.
\end{remark}

\subsection{Perturbation bounds for eigenspaces}
The Davis-Kahan $\sin \theta$-theorem \cite{daviskahanbound} states that the eigenspaces of a symmetric matrix
are stable under perturbations of that matrix. 
Among others, we will need this result in order to show that 
the spectral initialization recovers the eigenspace of the ground truth matrix sufficiently well.
We also will need it in order to show that $\UU_{0,\ww}$ is sufficiently close to $\UU_{0} $.

To state this theorem, 
recall that for a symmetric matrix $\ZZ \in \R^{n \times n}$
with eigendecomposition $\ZZ = \UU_{\ZZ} \LLambda_{\ZZ} \UU_{\ZZ}^\top$
the matrix $ \UU_{\ZZ,r} \in \R^{ n \times r }$ consists of the first $r$ columns of $\UU_{\ZZ}$
and the matrix $ \UU_{\ZZ,r,\bot} \in \R^{n \times (n-r)}$ consists of the remaining $n-r$ columns.
Moreover, recall that the eigenvalues of $\ZZ$ are ordered
such that their magnitude is decreasing,
i.e., $ \vert \lambda_1 (\ZZ) \vert 
\ge \vert \lambda_2 (\ZZ) \vert 
\ge \ldots \ge \vert \lambda_n (\ZZ) \vert$.

\begin{lemma}[Davis-Kahan inequality, Corollary 2.8 in \cite{chen_spectralMethods}]\label{lem:DavisKahan}
    Set $\triplenorm{\cdot}=\specnorm{\cdot}$ 
    or $\triplenorm{\cdot}=\fronorm{\cdot}$.
    Let $\ZZ_1 \in \R^{d \times d}$ and $\ZZ_2 \in \R^{d \times d}$ be two symmetric matrices,
    such that the eigenvalues of $\ZZ_1$ satisfy
   $ \vert \lambda_r (\ZZ_1) \vert > \vert \lambda_{r+1} \left( \ZZ_{1}\right) \vert$ for an integer $1\le r < d$.
    Let the eigendecompositions of $\ZZ_1$ and $\ZZ_2$ be given by
    $\ZZ_1=\UU_1\LLambda_1\UU_1^\top$, respectively $\ZZ_2=\UU_2\LLambda_2\UU_2^\top$. 
    Then, if the assumption 
    $$\specnorm{ \ZZ_1-\ZZ_2 } 
    \leq \left(1-1/\sqrt{2} \right)\left( \vert \lambda_r(\ZZ_1) \vert - \vert \lambda_{r+1}(\ZZ_1) \vert \right)$$
    is fulfilled, 
    it holds that
    \begin{align}
        \triplenorm{\UU_{2,r,\perp}^\top \UU_{1,r}}
        \leq  
        \frac{\sqrt{2} \triplenorm{(\ZZ_1-\ZZ_2)\UU_{1,r}}}{\vert \lambda_r(\ZZ_1)\vert-\vert\lambda_{r+1}(\ZZ_1) \vert}.
    \end{align}
\end{lemma}

\section{Proof ideas}\label{sec:outline}

\subsection{A fundamental barrier in previous work}\label{subsec:barrier}

Before we give an outline of our proof approach,
we want to explain 
why in previous work the additional $r$-factor appeared in the sample complexity.
As Lemma \ref{lemma:spectralinitialization} below shows, 
it holds for the spectral initialization $\UU_0$  
with high probability that
\begin{equation*}
    \specnorm{\XXstar - \UU_0 \UU_0^\top}
    \le 
     C\kappa\sigma_{\min}(\XXstar) \sqrt{\frac{rd}{m}}.
\end{equation*}
In particular, for $ m \gg \kappa^2 rd $
we have that
\begin{equation*}
    \specnorm{\XXstar - \UU_0 \UU_0^\top}
    \ll
    \sigma_{\min} (\XXstar).
\end{equation*}
Thus, the spectral initialization ensures
that the initialization $\UU_0$ is in a neighborhood of the ground truth.
We aim to establish that within this neighborhood, gradient descent converges
with a linear rate.
To show this, we note first that the gradient of our objective function $\mathcal{L}$ depends on the random matrices $\left( \AAi \right)_{i=1}^m$. 
To deal with this, a common technique that has been used in previous works 
is to decompose the gradient of the objective function $\mathcal{L}$ into 
a sum of two terms:
\begin{align}\label{equ:gradientdecomposition}
    \nabla \mathcal{L} \left( \UU \right) 
    =
    \EE_{(\AAi)_{i=1}^m} \left[ \nabla \mathcal{L} (\UU) \right] 
    +
    \left[
    \nabla \mathcal{L} (\UU) 
    - 
    \EE_{(\AAi)_{i=1}^m} \left[ \nabla \mathcal{L} (\UU) \right] 
    \right].
\end{align}
The first term is the gradient of the population risk, i.e.,
the objective function one obtains in the limit case that the sample size $m$ goes to infinity.
The second term can be interpreted as a perturbation term
that measures the deviation of the gradient of the empirical risk from the gradient of the population risk.
In particular, this term converges to zero as the sample size $m$ increases.
For this reason, a major task in our proof is to show
that the second summand is small with respect to a suitable norm
as soon as the sample size $m$ is sufficiently large.
A direct computation shows that
\begin{align*}
\nabla \mathcal{L} (\UU) 
- 
\EE_{(\AAi)_{i=1}^m} \left[ \nabla \mathcal{L} (\UU) \right] 
=
&\left[
\left( \Aops - \IdOp \right)
\left(  \UU \UU^\top - \XXstar \right)
\right]
\UU\\
=
&\frac{1}{m}
\sum_{i=1}^m
\innerproduct{\AAi, \UUt \UUtT - \XXstar} \AAi
-\left( \UUt \UUtT - \XXstar\right).
\end{align*}
To deal with this deviation term,
in previous works, bounds of the type
\begin{align}\label{ineq:goal}
    \specnorm{
    \left(
    \Aops - \IdOp
    \right)
    \left( \XXstar - \UUt \UUtT \right)
    }
    \ll
    \specnorm{\XXstar - \UUt \UUtT}
\end{align}
needed to be established.
A major challenge in establishing such bounds
is that the gradient descent iterates $(\UUt)_{t}$ depend on the measurement matrices
$\left( \AAi \right)_{i=1}^m$ in an intricate way.
For this reason,
standard matrix concentration inequalities are not directly applicable.
To circumvent this issue, 
previous work establishes \textit{uniform} bounds
for the quantity
\begin{align*}
    \underset{ \ZZ \in \mathcal{T}_{2r} }{\sup}   
    \specnorm{
    \left(
    \Aops - \IdOp
    \right)
    \left( \ZZ \right)
    }
\end{align*}
where
\begin{align}
    \mathcal{T}_r:=\left\{ \ZZ \in \R^{d \times d} : \ZZ=\ZZ^\top, \text{rank} \left(\ZZ\right) \le r, \specnorm{\ZZ} \le 1 \right\},
\end{align}
denotes the collection of matrices with rank at most $r$
and bounded operator norm.
Indeed, such a bound can be directly derived from the Restricted Isometry Property. 
Namely, when $\mathcal{A}$ has the RIP of order $2r+2$ with constant $\delta_{2r+2}$
then Lemma \ref{lemma: RIP} implies that
\begin{align*}
    \underset{ \ZZ \in \mathcal{T}_{2r} }{\sup}   
    \specnorm{
    \left(
    \Aops - \IdOp
    \right)
    \left( \ZZ \right)
    }
    \le
    \delta_{2r+2}
    \underset{ \ZZ \in \mathcal{T}_{2r} }{\sup}   
    \fronorm{\ZZ}
    \le 
    \delta_{2r+2} \sqrt{2r},
\end{align*}
where in the second inequality, we used that the matrix $\ZZ$ has rank at most $2r$
and that $\specnorm{\ZZ}=1$.
Thus, it follows from Lemma \ref{lem:rank_RIP} that whenever $m \gg rd$ that with high probability we have that
\begin{equation}\label{equ:uniformbound}
    \underset{ \ZZ \in \mathcal{T}_{2r} }{\sup}   
    \specnorm{
    \left(
    \Aops - \IdOp
    \right)
    \left( \ZZ \right)
    }
    \lesssim
    \sqrt{\frac{r^2d}{m}}.
\end{equation}
This shows that if we want to deduce inequality \eqref{ineq:goal} from the uniform bound \eqref{equ:uniformbound}
we must assume that $ m \gg r^2 d $.
Indeed, several works, e.g., \cite{li2018algorithmic,stoger2021small,zhuo2024computational}, relied precisely on this bound.

This leads to the question of whether the bound \eqref{equ:uniformbound} can be sharpened.
For example, in \cite[p. 9]{zhuo2024computational}, it was conjectured
that using more refined techniques from empirical process theory, one may be able to refine \eqref{equ:uniformbound}.
However, as the following result shows, inequality \eqref{equ:uniformbound} is tight up to absolute numerical constants
and thus cannot be improved further.
 \begin{theorem}\label{thm:lowerbound}
Let $(\AAf_i)_{i \in [m]}$ be independent $d\times d$ symmetric random matrices, where each $\AAi$ has independent entries  with distribution $ \mathcal{N} \left( 0,1 \right) $ on the diagonal 
and $ \mathcal{N} \left( 0, 1/2 \right)$ on the off-diagonal entries. 
Assume $d\geq 6$, $m\geq C_0$ for some universal constant $C_0>0$,  and $r\leq \frac{d}{16}$.  
Then, with probability at least $1-2\exp(-\frac{m}{32})-2\exp(-\frac{d}{32})$,
it holds that
\begin{equation*}
  \sup_{\ZZ\in \mathcal{T}_r } 
  \specnorm{ \left( \Aops - \IdOp \right) \left( \ZZ \right)  }
  \ge \frac{1}{16}\sqrt{\frac{r^2d}{m}}.
 \end{equation*}
\end{theorem}
Theorem \ref{thm:lowerbound} shows that we will need to
use different proof techniques to establish a bound similar to
\eqref{ineq:goal}.
In particular, we cannot rely on uniform concentration inequalities.
These novel techniques will be introduced in Section \ref{subsec:virtualsequences} below.
Before that, we want to prove Theorem \ref{thm:lowerbound}.
\begin{proof}
First, we note that 
\begin{align*}
  \sup_{\ZZ\in \mathcal{T}_r } 
  \specnorm{ \left( \Aops - \IdOp \right) \left( \ZZ \right)  }
  =
    \underset{\ZZ \in \mathcal{T}_r}{ \sup}   \specnorm{ \frac{1}{m} \sum_{i=1}^m \innerproduct{\AAi, \ZZ} \AAi -\ZZ  }
    =
\underset{  \norm{\uu}=1}{\sup}~ \underset{\ZZ \in \mathcal{T}_r }{ \sup} \ 
\Big\vert \innerproduct{ \frac{1}{m} \sum_{i=1}^m \innerproduct{\AAi, \ZZ} \AAi -\ZZ   , \uu \uu^\top }
\Big\vert
.   
\end{align*}
Now for any fixed $\uu \in \R^{d}$  with $\twonorm{\uu} =1$, define 
\begin{equation*}
    \mathcal T_{\uu}:= \left\{ \ZZ \in \R^{d \times d} : \ZZ=\ZZ^\top, \text{rank} \left(\ZZ\right) \le r, \specnorm{\ZZ} \le 1,   \ZZ \uu =0  \right\},
\end{equation*}
i.e., the set consisting of matrices in $\mathcal{T}_r$, 
whose row space is orthogonal to $\uu$.
It follows that
\begin{align}
    \underset{\ZZ \in \mathcal{T}_r}{ \sup}   \specnorm{ \frac{1}{m} \sum_{i=1}^m \innerproduct{\AAi, \ZZ} \AAi -\ZZ  } \nonumber
&\ge   \underset{\ZZ \in \mathcal T_{\uu}}{ \sup}   \innerproduct{ \frac{1}{m} \sum_{i=1}^m \innerproduct{\AAi, \ZZ} \AAi -\ZZ   , \uu \uu^\top } \nonumber\\ 
&=   \underset{\ZZ \in \mathcal T_{\uu}}{ \sup}   \innerproduct{ \frac{1}{m} \sum_{i=1}^m \innerproduct{\AAi, \ZZ} \AAi    , \uu \uu^\top } \nonumber\\ 
&=   \underset{\ZZ \in \mathcal T_{\uu}}{ \sup} \frac{1}{m}   \sum_{i=1}^m  \innerproduct{ \innerproduct{ \AAi    , \uu \uu^\top } \AAi, \ZZ} . \label{ineq:aux1}
\end{align}

 Now note that $\innerproduct{ \AAi    , \uu \uu^\top } $ is independent of $\left(\innerproduct{\AAi , \ZZ } \right)_{\ZZ \in \mathcal{T}_{\uu}} $. 
Let $\AAf \in \R^{d \times d} $ be a matrix with the same distribution as $\AAi$ and which is independent of $(\AAi)_{i=1}^m$. 
We claim that conditional on $ \left\{ \innerproduct{ \AAi, \uu \uu^\top }  \right\}_{i=1}^m $ we have that  the following two random variables are equal in distribution:
\begin{align}\label{eq:GP}
    \underset{\ZZ \in \mathcal T_{\uu}}{ \sup} \frac{1}{ m}   \sum_{i=1}^m  \innerproduct{ \AAi    , \uu \uu^\top }\innerproduct{  \AAi, \ZZ} 
    &\stackrel{d}{=} \frac{1}{\sqrt m}    \sqrt{ \frac{1}{m} \sum_{i=1}^m \innerproduct{ \AAi, \uu \uu^\top  }^2  } \ \underset{\ZZ \in \mathcal T_{\uu}}{ \sup} \innerproduct{\AAf, \ZZ}.
\end{align}
To show \eqref{eq:GP}, one can check that conditional on $ \left\{ \innerproduct{ \AAi, \uu \uu^\top }  \right\}_{i=1}^m $, the random variables on both sides of \eqref{eq:GP} are the supremum of  Gaussian processes indexed by $\mathcal T_{\uu}$ with the same covariance structure, so they have the same distribution.

In the following, we set 
\begin{align}\label{eq:ed}
\uu:=(0,\dots, 0, 1)^\top \in \mathbb R^d.
\end{align}
It follows that 
\newcommand{\AAidd}{ \left(\AAi \right)_{d,d} }
\begin{align}
    \sum_{i=1}^m \innerproduct{ \AAi, \uu \uu^\top  }^2=\sum_{i=1}^m \AAidd^2.
\end{align}
By Lipschitz concentration for Gaussian random variables \cite[Theorem 5.6]{boucheron2013concentration}, we obtain 
\begin{align}
    \mathbb P \left(\left|\sqrt{\sum_{i=1}^m \AAidd^2}-\mathbb E \sqrt{\sum_{i=1}^m \AAidd^2 } \right| \geq \sqrt{m}/4\right)\leq 2\exp(-m/32). 
\end{align}
This shows that with probability at least $1-2\exp(-m/32)$, 
\begin{align}\label{eq:Lipchitz_bound}
\sqrt{\sum_{i=1}^m \AAidd^2}\geq \mathbb E \sqrt{\sum_{i=1}^m \AAidd^2 } -\frac{\sqrt{m}}{4}\geq \sqrt{m}/2
\end{align}
for sufficiently large $m$, where we have used that the expectation of chi-distribution with parameter $m$ has asymptotic value $\sqrt{m-\frac{1}{2}}$ (see, e.g., \cite{johnson1994chi}).
In addition, with $\uu$ given in \eqref{eq:ed}, 
all entries in the $d$-th row and $d$-th column of the matrix $\ZZ\in \mathcal T_{\uu}$ are equal to zero. 
Let $\tilde{\AAf}\in \mathbb R^{(d-1)\times (d-1)}$ 
be the submatrix $\AAf$  where the last row and column of $\AAf$ are removed, 
and define $\tilde{\ZZ}$ in the same way.
Then we have
\begin{align}
    \underset{\ZZ \in \mathcal T_{\uu}}{ \sup} \innerproduct{\AAf, \ZZ}=\underset{ \norm{\tilde\ZZ}\leq 1, \tilde{\ZZ}=\tilde{\ZZ}^\top, \ \mathrm{rank}(\tilde\ZZ)\leq r }{ \sup} \innerproduct{\tilde\AAf, \tilde \ZZ}=\sum_{i=1}^r \sigma_i(\tilde \AAf).
\end{align}
Our goal is to bound the sum of singular values on the right-hand side from below.
For that, 
we define the matrix
\begin{equation*}
\hat{\AAf}
:=
\begin{pmatrix}
    \mathbf{0}_{ (\lceil (d-1)/2 \rceil-1 )  \times r} & \mathbf{0}_{\lceil (d-1)/2 \rceil  \times (d-r)} \\
    \tilde{\AAf}_{\lceil (d-1)/2 \rceil:(d-1), 1:r} & \mathbf{0}_{(d-1-\lceil (d-1)/2 \rceil) \times (d-r)}
\end{pmatrix}
\in \R^{(d-1) \times (d-1)}.
\end{equation*}
Here, $\tilde{\AAf}_{\lceil (d-1)/2 \rceil:(d-1), 1:r}$ denotes the submatrix of $\AAf$ 
obtained by restricting $\AAf$ 
to the  $\lceil (d-1)/2 \rceil$-th to $(d-1)$-th rows
and the first $r$ columns. 
By $ \mathbf{0}_{a \times b} $ we denote the zero matrix of size $a$ times $b$.
To relate the singular values of $ \tilde{\AAf}$ with the singular values of $ \hat{\AAf} $, we will use the following lemma.
\begin{lemma}[Corollary 3.1.3 in \cite{horn1994topics}] \label{lem:interlacing}
Let $\AAf\in \mathbb R^{ (d-1) \times (d-1)}$ 
and  let $\BB \in \R^{(d-1) \times (d-1)}$ 
be a matrix which is obtained from the matrix $\AAf$
by setting the entries of one row or one column to zero.
Then it holds that $\sigma_i(\BB)\leq \sigma_i(\AAf)$ for all $i=1,\dots, d-1$.
\end{lemma}
By repeatedly applying  Lemma \ref{lem:interlacing}, 
we find 
\begin{align}
    \sum_{i=1}^r \sigma_i(\hat{\AAf})\leq \sum_{i=1}^r \sigma_i(\tilde \AAf).
\end{align}
On the other hand, 
we can identify the $r$ largest singular singular values of $\hat{\AAf}$ 
with the singular values of a Gaussian matrix of size $\lfloor \frac{d-1}{2} \rfloor \times r$. 
By standard concentration inequalities for the singular values of Gaussian matrices,
see, e.g., \cite[Corollary 5.35]{vershynin2010introduction}, 
we find that with probability at least $1-2\exp(-t^2/2),$ 
\begin{align}
  \sigma_{r}(\hat \AAf)\ge   \sqrt{\Big\lfloor \frac{d-1}{2} \Big\rfloor }-\sqrt{r}-t .
\end{align}
Taking $t=\frac{\sqrt d}{8}$, and using the assumption that $r\leq \frac{d}{16}$, we find for $d\geq 6$,
\begin{align}\label{eq:RMT_bound}
    \sum_{i=1}^r \sigma_i(\tilde \AAf)\geq \frac{r\sqrt{d}}{8}
\end{align}
with probability at least $1-2\exp(-d/32)$. Combining \eqref{eq:RMT_bound} and \eqref{eq:Lipchitz_bound} finishes the proof. 
\end{proof}
Note that the key idea in this proof was to fix a vector $\uu \in \R^d$
and to pick a matrix $\ZZ \in \mathcal{T}_r$ based on eigenvectors 
corresponding to the largest eigenvalues 
(of a submatrix) of
\begin{equation*}
    \AAf
    =
    \frac{1}{m}
    \sum_{i=1}^m
    \innerproduct{\AAi, \uu \uu^\top} \AAi.
\end{equation*}
By design, this implies that the matrix $\ZZ$
was chosen in a way which strongly depends on $ \left( \innerproduct{\AAi, \uu \uu^\top} \right)_{i=1}^m $.
This observation leads to the key idea in our proof.
Namely, we will show that
our gradient descent iterates $\UUt$ 
depend, in a suitable sense, only weakly 
$\left( \innerproduct{\AAi, \uu \uu^\top} \right)_{i=1}^m  $ for fixed $\uu \in \R^d$.
This will allow us to prove stronger upper bounds  
for the term 
$ \specnorm{ \left( \Aops - \IdOp \right) \left( \XXstar - \UUt \UUtT \right)} $
than what can be achieved using uniform concentration inequalities.

\subsection{Virtual sequences}\label{subsec:virtualsequences}
As explained at the end of Section \ref{subsec:barrier},
we aim to establish that the gradient descent iterates 
$ \left( \UUt \right)_{t} $ depend only weakly on
$ \left( \innerproduct{\AAi, \ww \ww^\top} \right)_{i=1}^m $
in a suitable sense.
For this aim, we will use so-called \textit{virtual sequences} $\left(\UUtw \right)_{t \in \mathbb N} \subset \mathcal{S}^d$.
The central idea is to introduce for
$ \ww \in S^{d-1}:= \left\{ \xx \in \R^d: \twonorm{\xx} =1 \right\}$
a sequence with the following two properties.
\begin{enumerate}
    \item The sequence $\left(\UUtw \right)_{t \in \mathbb N}$ is stochastically independent of $ \left( \innerproduct{ \AAi, \ww \ww^\top } \right)_{i=1}^m $.
    \item The sequence $\left(\UUtw \right)_{t \in \mathbb N}$ stays sufficiently close to the sequence $\left( \UUt \right)_{t \in \mathbb N}$.
    More precisely, 
    we require that $ \fronorm{ \UUt \UUtT -\UUtw \UUtw^\top }$ stays sufficiently small.
\end{enumerate}
The sequences $\left( \UUtw \right)_{t\in \mathbb N}$ are called \textit{virtual}
since they are introduced solely for proof purposes.
\begin{remark}[Related work]
In the context of non-convex optimization, 
the use of virtual sequences has been pioneered in the influential works \cite{MaCongNonconvex} and \cite{DingMatrixCompletion}.
In these works, 
\textit{leave-one-out sequences},
which can be seen as a special case of virtual sequences,
were introduced to show 
that the gradient descent iterates depend only weakly on the individual samples or measurements.
These works lead to a number of follow-up works.
For example,
several works used virtual sequences to establish
convergence from random initialization for gradient descent in phase retrieval \cite{ChenGlobalPhaseRetrieval}
or for alternating minimization in rank-one matrix sensing \cite{KiryungALS}.
In \cite{ma2024convergence},
leave-one-out sequences were used to establish that 
in overparameterized matrix completion
gradient descent with small random initialization converges to the ground truth.
Similar to the paper at hand,
the virtual sequence argument was combined with an $\varepsilon$-net argument.
However, the technical details are arguably quite different.
\end{remark}
Before defining the virtual sequences
we recall the notion of an $\varepsilon$-net.
\begin{definition}[$\varepsilon$-net]
    Let $A \subset \R^d$. 
    A subset $B \subset A$ is called $\varepsilon$-net of $A$
    if for every $\xx \in A$
    there is a point $\xx_0 \in B$
    such that $ \twonorm{\xx - \xx_0} \le \varepsilon $.
\end{definition}
It is well-known that for 
$S^{d-1}= \left\{ \xx \in \R^d: \twonorm{\xx} =1 \right\}$
there exists an $\varepsilon$-net $\epscover \subset S^{d-1} $
with  cardinality $ \vert \epscover \vert \le \left( 3/\varepsilon \right)^d  $
\cite{vershynin2018high}.
In the remainder of this paper, we will assume that $\epscover$ is a fixed $\varepsilon$-net of $S^{d-1}$ 
with $\varepsilon=1/2$ such that $\vert \epscover \vert \le 6^d$.
We will define one virtual sequence $ (\UUtw)_{t} $
for each $\ww \in \epscover$.

Recall from equation \eqref{eq:def_projw} that for $\ww \in \epscover$ the orthogonal projection operators $\Projw$ and $\Projwperp$ were defined for $ \ZZ \in \mathcal{S}^d$ via
\begin{align*}
    \Projw (\ZZ) 
    = \innerproduct{\ww \ww^\top, \ZZ} \ww \ww^\top,
    \quad
    \Projwperp (\ZZ) 
    = \ZZ - \innerproduct{\ww \ww^\top, \ZZ} \ww \ww^\top.
\end{align*}
Next, for $\ww \in  \epscover$ we define the modified measurement matrices via
\begin{equation*}
    \AAiw := \Projwperp (\AAi)  = \AAi -\innerproduct{\ww \ww^\top,  \AAi} \ww \ww^\top.
\end{equation*}
Thus, the matrix $\AAiw$ is obtained from the matrix $\AAi$ by setting the generalized entry $\innerproduct{\AAi, \ww \ww^\top}$ equal to $0$. 
We observe that by definition the matrices $ \left( \AAiw \right)_{i=1}^m $
are stochastically independent of 
$ \left( \innerproduct{\AAi, \ww \ww^\top} \right)_{i=1}^m $.
We define the virtual measurement operator $\Aopw : \mathcal{S}^d \rightarrow \R^{m+1}  $
via
\begin{equation*}
    [\Aopw (\ZZ)]_i := \frac{1}{\sqrt{m}} \innerproduct{\AAiw, \ZZ}
\end{equation*}
for $i \in [m]$
and 
\begin{equation*}
    [\Aopw (\ZZ)]_{m+1} :=  \innerproduct{\ww \ww^\top, \ZZ}.
\end{equation*}
Again, we observe that by construction, the measurement operator $\Aopw$ is independent of 
$ \left( \innerproduct{\AAi, \ww \ww^\top} \right)_{i=1}^m $.
As a next step, analogously to the definition of the objective function $\mathcal{L}$,
we can define the modified objective function
$\mathcal{L}_{\ww}: \mathcal{S}^d \rightarrow \R $ via
\begin{equation*}
    \mathcal{L}_{\ww}
    \left( \UU \right)
    :=
    \frac{1}{4}
    \twonorm{
    \Aopw
    \left(
    \XXstar - \UU \UU^\top
    \right)
    }^2.
\end{equation*}
With these definitions in place,
the virtual sequence $ \left( \UUtw \right)_{t} $ can be defined analogously to the original sequence
$\left( \UUt \right)_t$.
Namely, to define the spectral initialization, we consider the eigendecomposition
\begin{align}\label{eq:defVw}
    \left( \Aopws \right) \left(\XXstar \right)&=:\widetilde\VV_{\ww}\widetilde\LLambda_{\ww} \widetilde\VV_{\ww}^\top.
\end{align}
Then, analogously as for the original spectral initialization $\UU_0$,
the matrix $\UU_{0,\ww}$ is defined as
\begin{align}
    \UU_{0,\ww}&=:\widetilde \VV_{r,\ww}{\widetilde\LLambda_{r,\ww}}^{1/2}. \label{eq:defU0w}
\end{align}
Then the virtual sequence $ \left\{ \UUtw \right\}_{t \in \N} $ via
\begin{equation*}
\UUtplusw
:= 
\UUtw 
-
\mu \nabla \mathcal{L}_{\ww} \left( \UUtw \right)
=
\UUtw + \mu \left[ \left( \Aopws \right) \left( \XXstar - \UUtw \UUtw^\top \right) \right] \UUtw.
\end{equation*}
It follows directly from the definition of $ \left( \UUtw \right)_{t}$
that this sequence is stochastically independent of
$ \left( \innerproduct{\AAi, \ww \ww^\top} \right)_{i=1}^m $.
At the end of this section, we state and prove the following lemma,
which is a direct consequence of the definition of $\Aopw$.
This lemma will be useful in the convergence analysis 
where we establish that 
$ \fronorm{ \UUt \UUtT -\UUtw \UUtw^\top } $
stays sufficiently small.
\begin{lemma}\label{lemma:AopIdentities}
For any symmetric matrix $\ZZ \in \R^{d \times d}$ it holds that
\begin{align*}
    \left(\Aopws \right) \left( \Projw (\ZZ) \right)
    &=
    \Projw (\ZZ), \\
    \left( \Aopws \right) \left( \Projwperp (\ZZ) \right)
    &=
    \left( \Aops \right) \left( \Projwperp (\ZZ) \right) 
    - \innerproduct{\Aop (\ww \ww^\top), \Aop \left( \Projwperp (\ZZ)\right) } \ww \ww^\top.
\end{align*}
\end{lemma}
\begin{proof}[Proof of Lemma \ref{lemma:AopIdentities}]
    To prove the first inequality we note first that it follows directly from the definition of $\AAiw$ that
    $\innerproduct{\AAiw, \Projw (\ZZ)}
    =0$.
    It follows that
    \begin{align*}
    \left( \Aopws \right) \left( \Projw (\ZZ) \right)
    &=
    \frac{1}{\sqrt{m}} \sum_{i=1}^m \left[\Aopw \left(\Projw (\ZZ) \right) \right]_i \AAiw
    +
    \left( \Aopw \left(\Projw (\ZZ) \right) \right)_{m+1} \ww \ww^\top \\
    &=
    \frac{1}{m} \sum_{i=1}^{m}  \innerproduct{\AAiw, \Projw (\ZZ)}\AAiw
    +\innerproduct{\ww \ww^\top, \ZZ}\ww \ww^\top\\
    &=
    \innerproduct{\ww \ww^\top, \ZZ}\ww \ww^\top.
    \end{align*}
This proves the first equation.
In order to prove the second equation, we note that
\begin{align*}
    \left( \Aopws \right) \left( \Projwperp (\ZZ) \right)
    &=
    \frac{1}{m} \sum_{i=1}^m  \innerproduct{\AAiw, \Projwperp (\ZZ)} \AAiw 
    +
    \innerproduct{\ww \ww^\top, \Projwperp (\ZZ) } \ww \ww^\top\\
    &=
    \frac{1}{m} \sum_{i=1}^m  \innerproduct{\AAiw, \Projwperp (\ZZ)}\AAiw\\
    &=
    \frac{1}{m} \sum_{i=1}^m  \innerproduct{\AAi, \Projwperp (\ZZ)}\AAiw\\
    &=
    \frac{1}{m} \sum_{i=1}^m  \innerproduct{\AAi, \Projwperp (\ZZ)}\AAi
    -
    \frac{1}{m} \sum_{i=1}^m  \innerproduct{\AAi, \Projwperp (\ZZ)} \innerproduct{\ww \ww^\top, \AAi} \ww \ww^\top\\
    &=
    (\Aops) \left( \Projwperp (\ZZ)\right) 
    - \innerproduct{ \Aop (\ww \ww^\top), \Aop ( \Projwperp (\XX) ) } \ww \ww^\top.
\end{align*}
This proves the second equation.
\end{proof}

\subsection{Upper bounds for the spectral norm of the deviation term}

Recall that by construction, it holds for any $\ww \in \epscover$
that the sequence $ \left( \UUtw \right)_{t=0,1,\ldots, T} $ 
is independent of $ \left( \langle \ww \ww^\top, \AAi \rangle  \right)_{i=1}^m $.
This property allows us to establish the following key lemma
which we will use several times throughout our proof.
\begin{lemma}\label{lemma:independencebound}
    Let $\epscover$ be the $\varepsilon$-net with $\varepsilon=1/2$
    introduced in Section \ref{subsec:virtualsequences} 
    which we used to construct the virtual sequences 
    $ \left( \UUtw \right)_{t} $.
    Assume that for the cardinality of $\epscover$, we have that
    $\vert \epscover \vert \le 6^d$.
    Moreover, let $T \in \mathbb{N}$ such that $2T \le 6^d$.
    Then, with probability at least $1-2\exp \left(-10d \right)$,
    it holds
    for all $\ww \in \epscover$ and all $1\le t \le T$ that
    \begin{equation*}
    \vert   \innerproduct{\ww \ww^\top, \left(\Aops \right) \left( \Projwperp \left( \XXstar - \UUtw \UUtw^\top  \right)  \right)  }\vert \\
    \le 
    4 \sqrt{\frac{d}{m}} \twonorm{ \Aop \left( \Projwperp \left( \XXstar - \UUtw \UUtw^\top \right) \right) }.
    \end{equation*}
\end{lemma} 
\begin{proof}
We introduce the shorthand
\begin{equation*}
\DDeltatw := \XXstar - \UUtw \UUtw^\top.
\end{equation*}
Due to the definition of $\AAiw $ and
due to the rotation invariance of the Gaussian distribution,
$\{\AAiw\}_{i=1}^m$ and  $\{\innerproduct{\ww \ww^\top, \AAi }\}_{i=1}^m$ are independent.
Moreover, note that by construction $\DDeltatw$ is independent of $ \left\{ \innerproduct{\ww \ww^\top, \AAi } \right\}_{i=1}^m$.
Thus, it follows that $ \left\{ \innerproduct{\ww \ww^\top, \AAi } \right\}_{i=1}^m$ 
is independent of $ \left\{ \innerproduct{\AAi, \Projwperp \left( \DDeltatw \right) } \right\}_{i=1}^m$. 
Moreover, the vector $ \left( \innerproduct{\ww \ww^\top, \AAi} \right)_{i=1}^m $ has i.i.d. entries with distribution $\mathcal{N}(0,1)$.
Thus, 
we have for all $x>0$ with probability at least $ 1 - 2\exp \left( -x^2/2 \right)$ (see \cite[Proposition 2.1.2]{vershynin2018high}) that 
\begin{align}
 \big\vert   \innerproduct{\ww \ww^\top, (\Aops ) \left(\Projwperp (\DDeltatw) \right) }  \big\vert 
&= \big\vert  \frac{1}{m} \sum_{i=1}^m \innerproduct{\ww \ww^\top, \AAi } \innerproduct{\AAi, \Projwperp (\DDeltatw) } \big\vert \\
&\le
\frac{x}{m} \sqrt{ \sum_{i=1}^m \innerproduct{ \AAi, \Projwperp \left( \DDeltatw \right) }^2 } \\
&=
\frac{x}{\sqrt{m}} \twonorm{\Aop \left(\Projwperp (\DDeltatw)\right)}.\label{ineq:intern1}
\end{align}

Then, by applying inequality \eqref{ineq:intern1} with $x=C \sqrt{d}$ and by taking a union bound, it follows that with probability at least $1-\xi$ (over the whole probability space), we have for all $\ww \in \epscover$ and all $t \in [T]$ that
\begin{equation*}
 \big\vert   \innerproduct{\ww \ww^\top, (\Aops) \left( \Projwperp (\DDeltatw) \right) } \big\vert \\
 \le 
\frac{C \sqrt{d}}{\sqrt{m}} 
\twonorm{ \Aop \left( \Projwperp (\DDeltatw) \right) },
\end{equation*}
where
\begin{align*}
    \xi 
    \le 2 T \vert \epscover \vert \exp \left( - C^2 d \right)
    \le  6^{2d} \exp \left( - C^2 d \right)
    =   \exp \left( 2d \log(6) - C^2 d \right).
\end{align*}
The claim follows from choosing $C=4$.
\end{proof}
Recall that our goal was to derive an upper bound for the expression
$\specnorm{\left(\Aops - \IdOp \right) \left( \XXstar - \UUt \UUtT \right)}$.
The following lemma provides such a bound for $1\le t \le T$.
Here, $T \in \mathbb{N}$ is some fixed number of iterations,
which will be specified later in the proof of our main result.

\begin{proposition}\label{proposition:key}
    Let $\epscover$ be the $\varepsilon$-net from above with $\varepsilon = 1/2$
    which we used to construct the virtual sequences $ \left( \UUtw \right)_{t=0,1,\ldots,T} $.
    Assume that the conclusion of Lemma \ref{lemma:independencebound} holds.
    Moreover, assume that the linear measurement operator $\mathcal{A}$
    has the Restricted Isometry Property of order $2r+2$ with constant $\delta =\delta_{2r+2} \le 1$.
    Then it holds that for all $0\le t\le T$,
    \begin{equation*}
        \begin{split}
        \specnorm{ \left( \Aops -\IdOp \right) \left( \XXstar - \UUt \UUtT \right)  }
        \le 
        &\left( 16 \sqrt{\frac{2rd}{m}} + 2\delta  \right)  \specnorm{ \XXstar - \UUt \UUtT }\\
        &+
        4\left( \delta + 4 \sqrt{\frac{d}{m}} \right)  \supw  \fronorm{\UUt \UUtT - \UUtw \UUtw^\top}.
        \end{split}
    \end{equation*}
\end{proposition}
As already mentioned, in previous literature, the quantity $ \specnorm{ \left( \Aops -\IdOp \right) \left( \XXstar - \UUt \UUtT \right)  }$ was controlled via an upper bound of $\sup_{\ZZ\in \mathcal T_{2r}}\specnorm{ \left( \Aops -\IdOp \right) \left( \ZZ \right)  }$, where $\mathcal{T}_{2r}$ is a set of all rank-$2r$ matrices with bounded operator norm. 
This requires a uniform concentration bound for all matrices of rank at most $2r$ with bounded spectral norm.
As we have seen in Theorem~\ref{thm:lowerbound}, this argument necessarily leads to a multiplicative factor of $ \sqrt{r^2d/m} $.

In contrast, Proposition~\ref{proposition:key} bounds $ \specnorm{ \left( \Aops -\IdOp \right) \left( \XXstar - \UUt \UUtT \right)  }$ by a sum of two terms.
The first term can be controlled with sample complexity $m \gtrsim rd \kappa^2$ 
since we also have $\delta \lesssim \sqrt{rd/ m }$, see Lemma \ref{lem:rank_RIP}.
The second term is a uniform bound on the deviation of the ``true'' sequence from the ``virtual'' sequences.
This term can be interpreted as a measure of how stable the sequence 
$(\UUt)_t$ are under perturbation of the generalized entries $(\innerproduct{\AAi, \ww \ww^\top})_{i=1}^m$ of the symmetric measurement matrices.

\begin{proof}[Proof of Proposition~\ref{proposition:key}]
We use the shorthand notation
\begin{align*}
    \DDeltat &:= \XXstar - \UUt \UUtT,\\
    \DDeltatw &:= \XXstar - \UUtw \UUtw^\top.
\end{align*}
Since $\epscover$ is an $\varepsilon$-net of $S^{d-1}$ with $\varepsilon = 1/2$
we obtain that 
\begin{equation}\label{ineq:intern124}
    \specnorm{ (\Aops - \IdOp ) (\DDeltat)}
    \le 
    2 \supw \vert \innerproduct{\ww \ww^\top, (\Aops - \IdOp) (\DDeltat)} \vert,
\end{equation}
(see, e.g. \cite[Lemma 4.4.1]{vershynin2018high}).
Then, for every $\ww \in \epscover$ using the triangle inequality we obtain that
\begin{align}
   \vert \innerproduct{\ww \ww^\top, (\Aops - \IdOp) (\DDeltat) } \vert
   \le
   &\vert \innerproduct{\ww \ww^\top, (\Aops - \IdOp) (\DDeltatw) }\vert 
   + \vert \innerproduct{\ww \ww^\top, (\Aops - \IdOp) (\DDeltatw -\DDeltat) }\vert \\
   \le
   &\vert \innerproduct{\ww \ww^\top, (\Aops - \IdOp) (\DDeltatw)}\vert 
   + \specnorm{ (\Aops - \IdOp ) (\DDeltatw - \DDeltat)  } \\
   \le
   & \vert \innerproduct{\ww \ww^\top, (\Aops - \IdOp) (\DDeltatw) }\vert
   +
   \delta  \fronorm{\DDeltat-\DDeltatw}. \label{ineq:intern125}
\end{align}
The last line is a consequence of the Restricted Isometry Property
and Lemma \ref{lemma: RIP}, see inequality \eqref{ineq:RIPlemma2}.
To estimate the first summand further, we use the triangle inequality again, and we obtain that
\begin{align*}
    &\vert \innerproduct{\ww \ww^\top, (\Aops - \IdOp) (\DDeltatw) }\vert\\
    \le
    &\vert \innerproduct{\ww \ww^\top, (\Aops - \IdOp) \left( \Projwperp (\DDeltatw) \right) } \vert 
    +
    \vert \innerproduct{\ww \ww^\top, (\Aops - \IdOp) \left( \Projw (\DDeltatw) \right) } \vert \\
    \overeq{(a)}
    &\vert \innerproduct{\ww \ww^\top, (\Aops ) \left( \Projwperp (\DDeltatw) \right) } \vert 
    +
    \Big\vert \left( \twonorm{\Aop \left(\ww \ww^\top \right)}^2 -1 \right) \innerproduct{\ww \ww^\top, \DDeltatw} \Big\vert\\
    \overleq{(b)} 
    &\vert \innerproduct{\ww \ww^\top, (\Aops ) \left( \Projwperp (\DDeltatw) \right) } \vert 
    +
    \delta  \vert  \innerproduct{\ww \ww^\top, \DDeltatw} \vert\\
    \le 
    &\vert \innerproduct{\ww \ww^\top, (\Aops ) \left( \Projwperp (\DDeltatw) \right) } \vert 
    +
    \delta  \specnorm{\DDeltatw}.
\end{align*}
Equation $(a)$ follows from the definition of $ \Projw$ and $\Projwperp$
and in inequality $(b)$ 
we used the Restricted Isometry Property; see Definition \ref{definition:RIP}.
Thus, 
by combining the last estimate with inequalities \eqref{ineq:intern124} and \eqref{ineq:intern125}
and
taking the supremum over all $\ww \in \epscover$ 
we obtain that 
\begin{align}
    &\specnorm{ (\Aops-\IdOp) ( \DDeltat )  } \nonumber \\
    \le 
    &2 \supw \vert \innerproduct{ \ww \ww^\top, (\Aops ) \left( \Projwperp ( \DDeltatw) \right) }  \vert 
    +
    2 \delta \supw  \fronorm{\DDeltat-\DDeltatw}
    +
    2 \delta \supw  \specnorm{\DDeltatw} \nonumber \\
    \le
    &2 \supw \vert \innerproduct{ \ww \ww^\top, (\Aops ) \left( \Projwperp ( \DDeltatw) \right) }  \vert 
    +
    4 \delta  \supw  \fronorm{\DDeltat-\DDeltatw}
    +
    2 \delta  \specnorm{\DDeltat}.
    \label{ineq:ref1}
\end{align}
Since we assumed that the conclusion of Lemma \ref{lemma:independencebound} holds 
we obtain for the first summand that
\begin{align*}
    \supw \vert \innerproduct{\ww \ww^\top, (\Aops) \left( \Projwperp \left(\DDeltatw \right) \right)}\vert
    &\le 
    4 \sqrt{\frac{d}{m}} \ \supw \twonorm{ \Aop \left( \Projwperp (\DDeltatw) \right)}\\
    &\overleq{(a)}
    8 \sqrt{\frac{d}{m}} \ \supw \fronorm{\Projwperp (\DDeltatw)}\\
    &\le 
    8 \sqrt{\frac{d}{m}} \ \supw \ \fronorm{\DDeltatw}\\
    &\le 
    8 \sqrt{\frac{d}{m}}  \fronorm{\DDeltat}
    +
    8 \sqrt{\frac{d}{m}} \ \supw \fronorm{\DDeltat - \DDeltatw}
    \\
    &\overleq{(b)}
    8 \sqrt{\frac{2rd}{m}}  \specnorm{\DDeltat}
    +
    8 \sqrt{\frac{d}{m}} \ \supw \fronorm{\DDeltat - \DDeltatw}.
\end{align*}
Inequality $(a)$ follows from the assumption that the operator $\mathcal{A}$ has 
the Restricted Isometry Property of order $2r+2$ 
with an RIP-constant $\delta \le 1$.
To obtain inequality $(b)$, we have used that the rank of $\DDeltat$ is at most $2r$.
Inserting the last estimate into \eqref{ineq:ref1}, we obtain 
\begin{align*}
\specnorm{ \left( \Aops -\IdOp \right) (\DDeltat)  }
    \le 
    &\left( 16 \sqrt{\frac{2rd}{m}} + 2\delta  \right)  \specnorm{\DDeltat} 
    +
    4\left( \delta + 4 \sqrt{\frac{d}{m}} \right) \supw  \fronorm{\DDeltat-\DDeltatw}.
\end{align*}
Inserting the definition of $\DDeltat$ and $\DDeltatw$ yields the claim.
\end{proof}

\section{Proof of the main result}\label{sec:proof}

\subsection{Spectral Initialization}

We provide the following lemma to show that both the original sequence and the virtual sequences are close to the ground truth $\XXstar$ at the spectral initialization. 
Moreover, this lemma guarantees that
$ \fronorm{ \UU_0 \UU_0^\top - \UU_{0,\ww} \UU_{0,\ww}^\top } $
is sufficiently small.
The proof of Lemma~\ref{lemma:spectralinitialization} is deferred to Appendix~\ref{appendix:spectral}.

\begin{lemma}\label{lemma:spectralinitialization}
  There exists an absolute constant $C>0$  such that the following holds:
\begin{enumerate}
    \item  With probability at least $1-\exp(-4d)$, 
 if $m > C^2\kappa^2rd$ is satisfied, it holds that
 \begin{equation} \label{eq:XstartU0U0}
        \specnorm{
        \XXstar - \UU_0 \UU_0^\top
        }
    \le C\kappa\sigma_{\min}(\XXstar) \sqrt{\frac{rd}{m}}  .
    \end{equation}
\item  With probability at least $1-\exp(-2d)$, if $m > 4C^2\kappa^2rd$ is satisfied, it holds for every $\ww\in \mathcal N_{\varepsilon}$ that
 \begin{equation}\label{eq:XstarUwUw}
        \specnorm{
        \XXstar - \UU_{0,\ww} \UU_{0,\ww}^\top
        }
    \le 2C\kappa\sigma_{\min}(\XXstar)\sqrt{\frac{rd}{m}}.
    \end{equation}
    Consequently, if $m > 4C^2\kappa^2rd$, with probability at least $1-2\exp(-2d)$,
    it holds for every $\ww\in \mathcal N_{\varepsilon}$ that
    \begin{align}\label{eq:diff_U0_U0w}
        \specnorm{\UU_0 \UU_0^\top-\UU_{0,\ww} \UU_{0,\ww}^\top}\leq 3C\kappa\sigma_{\min}(\XXstar) \sqrt{\frac{rd}{m}}.
    \end{align}    
    \item   For any $\alpha\in (0,1)$, assume $m\geq \left(51C^2+\frac{C_1}{\alpha^2}\right)\kappa^2 rd$ for an absolute constant $C_1>0$. With probability at least $1-4\exp(-d)$, for every $\ww\in \mathcal N_{\varepsilon}$,  
    \begin{align}\label{eq:diff_r_Awops_Aops}
        \fronorm{
        \UU_{0} \UU_{0}^\top
        -
        \UU_{0,\ww} \UU_{0,\ww}^\top
        }
        \le  \left(  2\alpha +C\kappa\sqrt{\frac{rd}{m}}\right)\left(2\sigma_{\min}(\XXstar)+3\sqrt 2 C\kappa \sqrt{\frac{rd}{m}}\sigma_{\min}(\XXstar)\right).
    \end{align}
\end{enumerate}
 \end{lemma}

\subsection{Convergence Analysis}

\subsubsection{Outline of proof strategy}

Before we explain our proof strategy,
we want to recall the following convergence lemma
which was proven in \cite[Theorem 3.2]{tu2016low} and \cite{zheng2015convergent}.
It states that as soon as $\dist (\UUt, \UUstar)$ is small enough
then $\dist(\UUt, \UUstar)$ converges to zero with linear rate.
We state it in the version of the overview article \cite[Theorem 4]{chi2019nonconvex}.
\begin{lemma}\label{lemma:Phase2}
   Assume that the measurement operator $\Aop$ satisfies the Restricted 
   Isometry Property for all matrices of rank at most $6r$ 
   with constant $\delta_{6r}<1/10$.
   Let $\UU_0, \UU_1, \UU_2, \ldots$ 
   be a sequence of gradient descent iterates
   defined via equation \eqref{equ:gradientdescentdefinition}.
   Assume that the step size satisfies $\mu \le \frac{c_1}{\specnorm{\XXstar}}$
   and 
   \begin{equation}\label{ineq:closenesscondition}
    \dist^2 \left( \UU_T, \UUstar \right)
    \le 
    \frac{1}{16}  \sigma_{\min} (\XXstar)
   \end{equation} 
   for some iteration number $T$.
   Then it holds for all $t \ge T$ that 
   \begin{equation*}
    \dist^2 \left( \UUt, \UUstar \right)
    \le 
    \left(1- c_2 \mu \sigma_{\min} (\XXstar) \right)^{t-T}
    \dist^2 (\UU_T, \UUstar).
   \end{equation*}
   Here, $c_1, c_2 >0$  are absolute numerical constants chosen small enough.
\end{lemma}
Note that the condition $\delta_{6r} < 1/10$ holds with high probability
if the sample size satisfies $ m \gtrsim rd $. 
However, condition \eqref{ineq:closenesscondition} cannot be guaranteed
for the spectral initialization, i.e., for $T=0$, when $m \asymp rd \kappa^2$.
For this reason, Lemma \ref{lemma:Phase2} is not directly applicable in our proof.
To deal with this, we consider two different phases in our convergence analysis.
Namely, we set
\begin{equation*}
    T
    :=  
    \Big\lceil \frac{8}{\mu \sigma_{\min} \left( \XXstar \right) } 
    \log \left( 16r  \right)      
    \Big\rceil.
\end{equation*}
We will show that at the end of the first phase, 
which consists of the iterations $t=0,1,\ldots,T$,
condition \eqref{ineq:closenesscondition} holds.
The second phase starts at iteration $T$.
For the second phase, we have established 
that condition \eqref{ineq:closenesscondition} already holds
we can directly apply Lemma \ref{lemma:Phase2}
and we obtain linear convergence.
Thus, our main focus in this section
will be to analyze the first convergence phase.

In the following, we will give an outline of the analysis of this first phase.
As is typical in the analysis of non-convex optimization algorithms,
we will control several quantities simultaneously in each iteration via an induction argument.
The following list contains an overview of these.
\begin{enumerate}
    \item[a)] We will show that $\fronorm{\UUt \UUtT - \UUtw \UUtw^\top}$
    and $ \fronorm{ \VXXT \left(\UUt \UUtT - \UUtw \UUtw^\top \right) } $
    stay sufficiently small for each $\ww \in \epscover$.
    Together with Proposition \ref{proposition:key}, this allows us to control 
    the deviation term 
    $ \Vert  \left(\IdOp - \Aops\right) \left( \XXstar - \UUt \UUtT \right) \Vert .$ 
    \item[b)] We will show that for each iteration $t \in [T]$ it holds that
    $\specnorm{\XXstar - \UUt \UUtT} \leq c \sigma_{\min} (\XXstar)$ for some small constant $c>0$. \label{quantity1}
    This ensures that the gradient descent iterates stay in the basin of attraction,
    in which we can establish linear convergence.
    \item[c)] We will establish that $\fronorm{ \VXXT \left( \XXstar - \UUt \UUtT \right)}$ 
    decays linearly in each iteration. 
    Combined with the result from b) 
    this will allow us to establish linear convergence of 
    $\dist \left(\UUt, \UU_{\star} \right) $.
\end{enumerate} 

The remainder of this section is structured as follows.
In Section \ref{subsection:distanceauxsequence} we will provide the technical lemmas to control 
$\fronorm{\UUt \UUtT - \UUtw \UUtw^\top}$
and $ \fronorm{ \VXXT \left(\UUt \UUtT - \UUtw \UUtw^\top \right) } $
as described in a) above.
In Section \ref{subsection:localconvergence}, we will provide the technical lemmas  
which allow us to control the quantities described above in b) and c).
In Section \ref{subsection:mainconvergencelemma}, 
we will combine these ingredients to prove Proposition \ref{lemma:phase1},
which is our main result describing the convergence of the iterates $ (\UUt)_{0 \le t \le T} $ in the first convergence phase.
\subsubsection{Lemmas for controlling the distance between the virtual sequences and the original sequence}\label{subsection:distanceauxsequence}

The goal of this section is to show that the virtual sequence iterates 
$(\UUtw)_{t }$ stay sufficiently close to the original sequence $(\UUt)_{t}$.
This will be established via induction.
In the following, we will state all key lemmas.
To keep the presentation concise, we have moved the proofs, 
which may be of independent interest,
to Section \ref{sec:auxsequenceproofs}.

The first lemma in this section provides an a priori estimate.
Its proof can be found in Section \ref{subsec:proofauxdistanceweak}.
\begin{lemma}\label{lemma:auxdistanceweak}
    For absolute constants $\constone, \consttwo, \constthree >0$ chosen small enough
    the following statement is true.
    Let $ \ww \in \epscover $ and assume that
     \begin{align}
      \specnorm{\UUt} &\le \sqrt{2 \specnorm{\XXstar}}, \label{ineq:weakboundintern1}\\
      \specnorm{ \left( \Aops - \IdOp \right) \left( \XXstar - \UUt \UUtT \right) }
      &\le
      \constone \sigma_{\min} \left( \XXstar \right) ,
      \label{ineq:weakboundintern2}\\
      \specnorm{\XXstar - \UUt \UUtT}
      &\le 
      \sigma_{\min} (\XXstar),
      \label{ineq:weakboundintern3}\\
      \fronorm{\UUt \UUtT - \UUtw \UUtw^\top}
      &\le 
      \frac{\sigma_{\min} \left(\XXstar \right)}{80},
      \label{ineq:weakboundintern4}
     \end{align}
     and that the step size $\mu>0$ satisfies $\mu \le \frac{\consttwo}{\kappa \specnorm{\XXstar}}$.
     In addition, assume  that the conclusions of Lemma \ref{lemma:independencebound} hold
     and that
     \begin{align}
      \max
      \left\{ 
        \delta; 8 \sqrt{\frac{rd}{m}}
      \right\}  
      &\le \frac{\constthree}{\kappa}, \label{ineq:weakboundintern11}
     \end{align}
     where $\delta= \delta_{4r+1}$ denotes the Restricted Isometry Property of rank $4r+1$.
     Then it holds that
     \begin{equation*}
      \fronorm{\UUtplus \UUtplus^\top - \UUtplusw \UUtplusw^\top}
      \le 
     \frac{\sqrt{\sqrt{2} - 1}}{40} \sigma_{\min} (\XXstar).
     \end{equation*} 
  \end{lemma}
Under the assumption
that this a priori estimate holds,
the next lemma
shows that the quantity
$ \fronorm{ \UUt \UUt -\UUtw \UUtw^\top }  $ 
can be bounded from above
by the quantity
$\fronorm{\VXXT \left( \UUt \UUtT -\UUtw \UUtw^\top \right) } $.
The proof of this lemma has been deferred to Section \ref{subsec:proofauxcloseness1}.
\begin{lemma}\label{lemma:auxcloseness1}
Let $\ww \in \epscover$ and assume that   
\begin{align}
    \specnorm{ \UUt \UUt^\top - \XXstar}
    &\le
    \frac{ \sigma_{\min} \left( \XXstar \right)}{1600},\label{assump:cloneness3}\\
    \fronorm{ \UUt \UUt^\top - \UUtw \UUtw^\top }
    &\le 
    \frac{\sqrt{ 3 \left(\sqrt{2}-1 \right)} \cdot \sigma_{\min} \left( \XXstar \right)}
    {40  }.\label{assump:cloneness4}
\end{align}
Then it holds that 
\begin{align}\label{ineq:lemmaequ1}
    \fronorm{ \VXXPT \left(\UUt \UUt^\top - \UUtw \UUtw^\top \right) \VXXP}
    \le
    \frac{3 \fronorm{ \VXXT \left(\UUt \UUt^\top - \UUtw \UUtw^\top \right)}}{5}.
\end{align}
Moreover, it holds that 
\begin{align}\label{ineq:lemmaequ2}
    \fronorm{\UUt \UUt^\top - \UUtw \UUtw^\top}
    \le
    3
    \fronorm{\VXXT \left(\UUt \UUt^\top - \UUtw \UUtw^\top \right)}.
\end{align}
\end{lemma}
The following key lemma allows us to control 
$\fronorm{\VXXT \left( \UUt \UUtT -\UUtw \UUtw^{\top} \right)}$ iteratively.
Its proof can be found in Section \ref{subsec:proofauxsequencecloseness}.
\begin{lemma}\label{lemma:auxsequencecloseness}
For sufficiently small absolute constants 
$\constone, \consttwo, \constthree, \constfour, \constfive, \constsix >0$ the following statement holds.
Let $\ww \in \epscover$ and assume that
\begin{align}
    \specnorm{ \VXXPT \VV_{\UUt} }
    &\le 
   \constone, \label{assump:closeness5} \\
    \norm{\UUt} &\le \sqrt{2 \norm{\XXstar}},\label{assump:closeness6}\\
    \specnorm{ \UUt \UUt^\top - \XXstar}
    &\le
    \consttwo \sigma_{\min} (\XXstar),\label{assump:closeness7}\\
    \fronorm{ \UUt \UUt^\top - \UUtw \UUtw^\top }
    &\le 
    \constthree \sigma_{\min} \left(\XXstar \right).\label{assump:closeness8}
\end{align}
Moreover, assume that the step size satisfies $ \mu \le \frac{\constfour}{ \kappa \specnorm{\XXstar}} $.
Furthermore, assume that the conclusion of Lemma \ref{lemma:independencebound} holds
and that
\begin{align}
    \specnorm{  \left(  \Aops - \IdOp \right) \left(\XXstar- \UUt \UUt^\top \right) }
    &\le 
    \constfive  \sigma_{\min} (\XXstar), \label{assump:closenessRIP} \\
    \max \left\{ 
        \delta;
        8 \sqrt{\frac{2rd}{m}}
    \right\}
    &\le
    \frac{\constsix}{\kappa}, \label{assump:closenessDelta}
\end{align}
where $\delta=\delta_{4r+2}$ denotes the Restricted Isometry Constant of rank $4r+2$.
Then, it holds that
\begin{align*}
    &\fronorm{\VXXT \left( \UUtplus \UUtplus^\top - \UUtplusw \UUtplusw^\top \right)}\\
    \le 
    &\left( 1- \frac{\mu \sigma_{\min} (\XXstar)}{16} \right)
    \fronorm{ \VXXT \left( \UUt \UUt^\top - \UUtw \UUtw^\top \right) }
    +
    \mu   \sigma_{\min}(\XXstar)\specnorm{\XXstar - \UUt \UUtT}.
\end{align*}
\end{lemma}

\subsubsection{Lemmas controlling the distance between $\XXstar$ and $\UUt \UUtT$}\label{subsection:localconvergence}
In the following, let $\triplenorm{\cdot}$ denote any matrix norm, which satisfies the inequality
\begin{equation}\label{ineq:triplenorm}
   \triplenorm{ \mathbf{X} \mathbf{Y} \mathbf{Z}  } 
   \le 
   \specnorm{\mathbf{X}}
   \triplenorm{ \mathbf{Y} }
   \specnorm{ \mathbf{Z} }
\end{equation}
for all matrices $\mathbf{X}$, $\mathbf{Y}$, and $\mathbf{Z}$ 
with dimensions such that the matrix product $  \mathbf{X} \mathbf{Y} \mathbf{Z}$ is well-defined.
Note that all Schatten-$p$ norms
have this property. 
In particular, this includes the spectral norm $\specnorm{\cdot}$
and the Frobenius norm $\fronorm{\cdot}$.

In the following, we are interested
in establishing upper bounds for
$\triplenorm{ \XXstar - \UUt \UUtT }$, 
where either $\triplenorm{\cdot}=\fronorm{\cdot} $ or $ \triplenorm{\cdot}=\specnorm{\cdot} $.
Instead of estimating these quantities directly,
we will instead derive upper bounds for the quantity
\begin{equation}\label{ineq:intern123}
    \triplenorm{\VXXT \left( \XXstar - \UUt \UUtT \right)}.
\end{equation}
To be able to relate this quantity with
$\triplenorm{  \XXstar - \UUt \UUtT } $
one can then use the following lemma.
\begin{lemma}\label{lemma:localconvaux}
Let $\triplenorm{\cdot}$ be a norm for which inequality \eqref{ineq:triplenorm} holds.
Assume that 
\begin{align}
    \specnorm{\VXXPT \VUUt} \le \frac{1}{\sqrt{2}}.\label{ineq:localconv1}
\end{align}    
Then the following inequalities hold: 
\begin{align}
    \triplenorm{\VXXPT \UUt \UUtT \VXXP}
    &\le 
    2 \specnorm{\VXXPT \VUUt} 
    \triplenorm{ \VXXT  \left( \UUt \UUtT - \XXstar \right) \VXXP},\label{ineq:localconv2}\\
    \triplenorm{ \UUt \UUtT - \XXstar }
    &\le 
    2 \left( 1+ \specnorm{ \VXXPT \VUUt } \right)
    \triplenorm{ \VXXT \left( \UUt \UUtT - \XXstar \right) }.\label{ineq:localconv3}
\end{align}
\end{lemma}
A comparable lemma was proven in \cite{stoger2021small}
in a more general setting but with less explicit constants.
For the sake of completeness, we included in Appendix \ref{sec:prooflocalconv_a}.

The following lemma allows us to control the quantity
$\triplenorm{\VXXT \left(\XXstar - \UUt \UUtT\right)}$ iteratively.
We note that  
a similar lemma has already been proven in \cite{stoger2021small}
in a more general setting with less explicit constants.
For the sake of completeness, we again included a proof in Appendix \ref{sec:prooflocalconv_b}.
\begin{lemma}\label{lemma:localconv}
    Let $\triplenorm{\cdot}$ be a norm
    which is submultiplicative in the sense of inequality \eqref{ineq:triplenorm}.
    Assume that
    \begin{align}
        \specnorm{\VXXPT \VUUt}
        &\le 
        \frac{1}{2}, \label{ineq:localconv8} \\
        \specnorm{\UUt} 
        &\le \sqrt{2 \specnorm{\XXstar}}, \nonumber \\
        \specnorm{\XXstar - \UUt \UUtT} 
        &\le \frac{\sigma_{\min} (\XXstar)}{48}, \label{ineq:localconv5}\\
        \specnorm{\left(\Aops - \IdOp \right) \left( \XXstar - \UUt \UUtT \right)}
        &\le 
        \frac{1}{48}
        \sigma_{\min} \left( \XXstar \right),\label{ineq:localconv4}
    \end{align}
    and that the step size satisfies 
    $ \mu \le \frac{1}{1024 \kappa \specnorm{\XXstar}} $.
    Then it holds that
    \begin{align*}
        &\triplenorm{\VXXT \left( \UUtplus \UUtplus^\top - \XXstar \right)}\\
        \le 
        &\left(1 - \frac{\mu}{8} \sigma_{\min} \left( \XXstar \right) \right) 
        \triplenorm{\VXXT \left( \XXstar - \UUt \UUtT \right) }
        +
        5 \mu \specnorm{\XXstar}
        \triplenorm{
            \left[ \left(\Aops - \IdOp \right)
            \left(\XXstar - \UUt \UUtT\right)
            \right]
            \VUUt
        }
        .
    \end{align*}
\end{lemma}

Given an upper bound for $ \fronorm{\XXstar - \UUt \UUtT}  $
we can obtain an estimate for $\dist \left(\UUt, \UU_{\star}\right) $ 
by using the following technical lemma.
\begin{lemma}[Lemma 5.4 in \cite{tu2016low}]\label{lemma:procrustebound}
Let $\UU, \VV \in \R^{d \times r}$ be two matrices 
and assume that $\text{rank} (\UU)= \min \left\{ r;d \right\}$.
Then it holds that
\begin{align*}
    \dist^2 \left( \UU, \VV \right)
    \le
    \frac{1}{2(\sqrt{2}-1) \sigma_{\min}^2 (\UU)} 
    \fronorm{\UU \UU^\top - \VV \VV^\top}^2,
\end{align*}
where $\dist \left( \UU, \VV \right)$ is defined in \eqref{eq:def_dist}.
\end{lemma}

To check the prerequisite of the Davis-Kahan inequality (Lemma~\ref{lem:DavisKahan}) in our proof,
we will also need the following auxiliary lemma,
which provides us with an a priori bound for 
$ \specnorm{\XXstar - \UUtplus \UUtplus^\top} $.
Its proof can be found in Appendix \ref{sec:convaprioribound}.

\begin{lemma}\label{lemma:convaprioribound}
There are absolute constants $c_1,c_2,c_3>0$ such that the following holds.
Assume that  $  \mu \le \frac{c_1}{ \specnorm{\XXstar}}$
and
\begin{align}
    \specnorm{\UUt} &\le \sqrt{2\specnorm{\XXstar}}, \label{eq:UUt1}\\
    \specnorm{\XXstar - \UUt \UUtT}
    &\le 
    c_2 \sigma_{\min} (\XXstar),\label{eq:UUt3}\\
        \specnorm{\left( \Aops - \IdOp \right)
    \left( \XXstar - \UUt\UUtT \right)}
    &\le 
    c_3 \sigma_{\min} \left( \XXstar \right). \label{eq:UUt2}\\
\end{align}
Then it holds that
\begin{align*}
    \specnorm{ \XXstar - \UUtplus \UUtplus^\top  }
    &\le 
    \left(
    1-\frac{1}{\sqrt{2}}
    \right)
    \sigma_{\min} \left( \XXstar \right).
\end{align*}   
\end{lemma}

\subsubsection{Statement and proof of the main convergence lemma}\label{subsection:mainconvergencelemma}

We now have all the ingredients in place to prove the main lemma in this section,
which is stated below. 
\begin{lemma}\label{lemma:phase1}  
    There are absolute constants $c_1,c_2,c_3, c_4 >0$ 
    chosen sufficiently small
    such that the following statement holds.
    Assume that the spectral initialization $\UU_0$ satisfies 
    \begin{align}
        \specnorm{
        \XXstar - \UU_0 \UU_0^\top
        }
    \le 
     c_1 \sigma_{\min} \left( \XXstar \right)
    \label{assump:localconv1}
    \end{align}
    and that for every $\ww \in \epscover$ we have that
    \begin{align}
        \fronorm{
        \UU_{0} \UU_{0}^\top
        -
        \UU_{0,\ww} \UU_{0,\ww}^\top
        }
        \le c_2 \sigma_{\min} \left( \XXstar \right)
        .
    \label{assump:localconv2}
    \end{align}
    Moreover, we assume that the conclusion of Lemma \ref{lemma:independencebound} 
    holds for
    \begin{equation*}
        T=  \Big\lceil \frac{8}{\mu \sigma_{\min} \left( \XXstar \right) } 
        \log \left( 16r  \right)      
        \Big\rceil.
    \end{equation*}
    Furthermore, we assume that 
   \begin{align}
    \max 
    \left\{
        \delta;  8 \sqrt{\frac{2rd}{m}}
    \right\}
    &\le \frac{\constthree}{\kappa}, 
    \label{assump:localconv4}
   \end{align}
   where $\delta= \delta_{4r+2}$ denotes the Restricted Isometry Property of order $4r+2$.
   In addition, assume that $\mu \le \frac{c_4}{\kappa \specnorm{\XXstar}}$.
 Then for every iteration $t$ with $0 \le t \le T$ it holds that
    \begin{equation}\label{ineq:localconv10}
        \dist^2 \left( \UUt, \UUstar \right)
        \le
        r \left(1 - \frac{\mu  \sigma_{\min} (\XXstar) }{16}\right)^{2t}
        \specnorm{\XXstar - \UU_0 \UU_0^{\top}}.
    \end{equation}
    In particular, we have that
    \begin{equation}\label{ineq:localconvT}
        \dist^2 \left( \UU_T, \UUstar \right)
        \le 
        \frac{1}{16} \sigma_{\min} (\XXstar),
    \end{equation}
   where $\UUstar \in \R^{n \times r}$ denotes a matrix which satisfies $\UUstar \UUstar^\top =\XXstar$.
\end{lemma}
\begin{proof}[Proof of Lemma \ref{lemma:phase1}]
We prove by induction that for all iterations $t$ with $ 0 \le t \le T$ the following inequalities hold:    
\begin{align}
    \fronorm{ \VXXT \left( \XXstar - \UUt \UUtT \right)} 
    \le
    &\left(1 - \frac{\mu}{16} \sigma_{\min} (\XXstar) \right)^{t} 
    \fronorm{\VXXT \left( \XXstar - \UU_0 \UU_0^\top  \right)}
    ,\label{ineq:induction3}\\
    \specnorm{\VXXT \left(\XXstar - \UUt \UUtT \right)} 
    \le
    & c_1 \sigma_{\min} \left( \XXstar \right)  ,\label{ineq:induction4}\\
    \specnorm{\VXXPT \VUUt}
    \le 
    &   \sqrt{2} c_1 ,
    \label{ineq:induction6}\\
    \specnorm{\XXstar - \UUt \UUtT} 
    \le
    & 3 c_1 \sigma_{\min} \left( \XXstar \right) ,\label{ineq:induction5}
\end{align}
and, for every $\ww \in \epscover$,
\begin{align}
    \fronorm{
        \VXXT
        \left(\UUt \UUtT
        -
        \UUtw \UUtw^\top\right)
    }
    \le
    &  c_2 \sigma_{\min} \left( \XXstar \right) , \label{ineq:induction1}\\
    \fronorm{ \UUt \UUtT - \UUtw \UUtw^\top }
    \le 
    &  3 c_2 \sigma_{\min} (\XXstar) .\label{ineq:induction2}
\end{align}
The constants $c_1, c_2 >0$ are the same as in assumptions \eqref{assump:localconv1}
and \eqref{assump:localconv2}
and are thus, in particular, independent of the iteration number $t$.

First, we check that these inequalities hold for $t=0$.
Inequality \eqref{ineq:induction3} is immediate.
Inequalities \eqref{ineq:induction4} and \eqref{ineq:induction5} follow from assumption \eqref{assump:localconv1}.
Inequalities \eqref{ineq:induction1} and \eqref{ineq:induction2} are due to assumption \eqref{assump:localconv2}.
It remains to establish inequality \eqref{ineq:induction6} for $t=0$.
Using the Davis-Kahan inequality (see Lemma \ref{lem:DavisKahan})
and assumption \eqref{assump:localconv1} it follows that 
\begin{align*}
    \specnorm{\VXXT \VV_{\UU_0}}
    \le
    \frac{\sqrt{2} \specnorm{ \VXXT \left( \XXstar - \UU_0 \UU_0^\top \right)}  }{\sigma_{\min} \left( \XXstar \right)}
    \le \sqrt{2}  c_1.
\end{align*}
This shows that the above inequalities hold for $t=0$.\\

For the induction step, assume now that these inequalities hold for some $t$.
First, we observe that it follows from the induction assumption \eqref{ineq:induction5} and Weyl's inequalities that 
$\specnorm{\UUt} \le \sqrt{2 \specnorm{\XXstar}}$ for $c_1<1/3$.
Moreover, note that since we assumed that the conclusion of Lemma \ref{lemma:independencebound} holds
we obtain from Proposition \ref{proposition:key} that
\begin{align}
    & \specnorm{\left(\Aops - \IdOp \right) \left(\XXstar - \UUt \UUtT \right) }\\
    \le
    &\left( 16 \sqrt{\frac{2rd}{m}} +2 \delta \right) \specnorm{\XXstar - \UUt \UUtT}
    +
    4 \left( \delta + 4 \sqrt{\frac{d}{m}} \right)
    \fronorm{ \UUt \UUtT - \UUtw \UUtw^\top} \nonumber \\
    \overleq{(a)}
    & \frac{4 c_3}{\kappa} \specnorm{\XXstar - \UUt \UUtT}
    +
    \frac{6 c_3}{\kappa} \fronorm{\UUt \UUtT - \UUtw \UUtw^{\top} }\nonumber \\
    \overleq{(b)} 
    & \frac{10 c_3}{\kappa} \sigma_{\min} \left( \XXstar \right), \label{ineq:intern87}
\end{align}
where inequality $(a)$ follows from assumption \eqref{assump:localconv4}.
Inequality $(b)$ is due to the induction hypotheses \eqref{ineq:induction5} and \eqref{ineq:induction2} with $c_1 \le 1/3$ and $c_2 \le 1/3$.
Next, we note that from Lemma \ref{lemma:localconv} applied with 
$\triplenorm{\cdot}= \fronorm{\cdot} $ it follows that 
\begin{align*}
    &\fronorm{\VXXT \left(
        \UUtplus \UUtplus^\top - \XXstar
    \right)}\\
    \le
    &\left(1 - \frac{\mu}{8} \sigma_{\min} \left( \XXstar \right) \right) 
        \fronorm{\VXXT \left( \XXstar - \UUt \UUtT \right) }
        +
        5 \mu \specnorm{\XXstar}
        \fronorm{
            \left[\left(\Aops - \IdOp \right)
            \left(\XXstar - \UUt \UUtT\right)\right]
            \VUUt
        } \\
    \overleq{(a)}
    &\left(1 - \frac{\mu}{8} \sigma_{\min} \left( \XXstar \right) \right) 
        \fronorm{\VXXT \left( \XXstar - \UUt \UUtT \right) }
        +
        5 \mu \delta \specnorm{\XXstar}
        \fronorm{
            \XXstar - \UUt \UUtT
        }\\
    \overleq{(b)}
    &\left(1 - \frac{\mu}{8} \sigma_{\min} \left( \XXstar \right) \right) 
        \fronorm{\VXXT \left( \XXstar - \UUt \UUtT \right) }
        +
        15 \mu \delta \specnorm{\XXstar}
        \fronorm{
          \VXXT \left( \XXstar - \UUt \UUtT \right)
        }\\
    \overleq{(c)}
    &\left(1 - \frac{\mu}{8} \sigma_{\min} \left( \XXstar \right) \right) 
        \fronorm{\VXXT \left( \XXstar - \UUt \UUtT \right) }
        +
        \frac{15  \mu c_3 \specnorm{\XXstar}}{\kappa}
        \fronorm{
          \VXXT \left( \XXstar - \UUt \UUtT \right)
        }\\
    \overleq{(d)}
    &\left(1 - \frac{\mu}{16} \sigma_{\min} \left( \XXstar \right) \right) 
        \fronorm{\VXXT \left( \XXstar - \UUt \UUtT \right) }.
\end{align*}
Inequality $(a)$ follows from the Restricted Isometry Property combined 
with Lemma \ref{lemma: RIP}.
Inequality $(b)$ is due to Lemma \ref{lemma:localconvaux} and
inequality \eqref{ineq:induction6}.
Inequality $(c)$ follows from assumption \eqref{assump:localconv4}
and inequality $(d)$ is due to the fact we can choose $c_3 \le \frac{1}{240}$.
Thus, using the induction assumption, 
we see that inequality \eqref{ineq:induction3} holds for $t+1$.

Next, our goal is to prove inequality \eqref{ineq:induction4} for $t+1$.
For that, we note that it follows from Lemma \ref{lemma:localconv}
with $ \triplenorm{\cdot} = \specnorm{\cdot} $ that
\begin{align}
    &\specnorm{\VXXT \left(
        \UUtplus \UUtplus^\top - \XXstar
    \right)}\\
    \le
    &\left(1 - \frac{\mu}{8} \sigma_{\min} \left( \XXstar \right) \right) 
    \specnorm{\VXXT \left( \XXstar - \UUt \UUtT \right) }
    +
     5 \mu \specnorm{\XXstar}
    \specnorm{
         \left(\Aops - \IdOp \right)
            \left(\XXstar - \UUt \UUtT\right)
        }\\
    \overleq{(a)}
    &\left(1 - \frac{\mu}{8} \sigma_{\min} \left( \XXstar \right) \right) 
    c_1 \sigma_{\min} (\XXstar) 
    +
    50 c_3 \mu  \sigma_{\min}^2 \left( \XXstar \right)\\
    \overleq{(b)}
    &c_1 \sigma_{\min} \left( \XXstar \right),
    \label{ineq:inter789}
\end{align}
where inequality $(a)$ follows from the induction hypothesis \eqref{ineq:induction4} and inequality \eqref{ineq:intern87}.
Inequality $(b)$ holds
since we can choose $c_1$ and $c_3$ in such a way that $c_3 \le \frac{c_1}{400} $.
This proves inequality \eqref{ineq:induction4} for $t+1$.

We observe that Lemma \ref{lemma:convaprioribound}  yields the 
a-priori bound
\begin{equation*}
    \specnorm{\XXstar- \UUtplus \UUtplus^\top}
    \le
    \left( 1 - \frac{1}{\sqrt{2}} \right)
    \sigma_{\min} \left( \XXstar \right).
\end{equation*}
Thus, we can apply the Davis-Kahan inequality 
(see Lemma \ref{lem:DavisKahan})
which together with
inequality \eqref{ineq:inter789}
yields that
\begin{align}\label{ineq:intern67}
    \specnorm{\VXXT \VV_{\UUtplus}}
    \le
    \frac{\sqrt{2} \specnorm{ \VXXT \left( \UUtplus \UUtplus^\top -\XXstar   \right)}  }{\sigma_{\min} \left( \XXstar \right)}
    \le 
      \sqrt{2} c_1 .
\end{align}
This proves inequality \eqref{ineq:induction6} for $t+1$.
Next, we apply Lemma \ref{lemma:localconvaux}
and \eqref{ineq:inter789} to obtain that
\begin{align*}
    \specnorm{\XXstar - \UUtplus \UUtplus^\top}
    &\le 
    2 \left(1 + \specnorm{ \VXXPT \VUUtplus } \right)
    \specnorm{\VXXT \left( \XXstar - \UUtplus \UUtplus^\top \right)}\\
    &\le 
    3 \specnorm{\VXXT \left(\XXstar - \UUtplus \UUtplus^\top \right)}
    \le  3c_1 \sigma_{\min} (\XXstar)  ,
\end{align*}
which proves inequality \eqref{ineq:induction5} for $t+1$.

Next, we can apply Lemma \ref{lemma:auxsequencecloseness} since all assumptions are satisfied 
and it follows that
\begin{align}
    &\fronorm{
        \VXXT
        \left(\UUtplus \UUtplus^\top
        -
        \UUtplusw \UUtplusw^\top\right)
    }\\
    \le
    &\left( 1- \frac{\mu \sigma_{\min} (\XXstar)}{16} \right)
    \fronorm{ \VXXT \left( \UUt \UUt^\top - \UUtw \UUtw^\top \right) }
    +
    \mu  
    \sigma_{\min}(\XXstar)\specnorm{\XXstar - \UUt \UUtT}\\
    \overleq{(a)} 
    &\left( 1- \frac{\mu \sigma_{\min} (\XXstar)}{16} \right)
   c_2 \sigma_{\min} \left( \XXstar \right)
    +
    3 c_1 \mu   \sigma^2_{\min}(\XXstar)\\
    \overleq{(b)} 
    & c_2 \sigma_{\min} \left( \XXstar \right) .
    \label{ineq:intern786}
\end{align}
Inequality $(a)$ is due to inequalities \eqref{ineq:induction5} and \eqref{ineq:induction1}.
Inequality $(b)$ holds since we can choose that $ c_1 \le  \frac{c_2}{48} $.
This proves inequality \eqref{ineq:induction1}.

Next, we want to prove inequality \eqref{ineq:induction2} for $t+1$.
First, we apply Lemma \ref{lemma:auxdistanceweak} and we obtain 
for all $ \ww \in \epscover $ the a-priori bound
\begin{align*}
\fronorm{
    \UUtplus \UUtplus^\top - \UUtplusw \UUtplusw^\top
}    
\le 
\frac{\sqrt{\sqrt{2}-1}}{40}
\cdot
\sigma_{\min} \left(\XXstar \right).
\end{align*}
This allows us to apply Lemma \ref{lemma:auxcloseness1}
and we obtain for all $ \ww \in \epscover $ the sharper bound
\begin{align*}
\fronorm{
    \UUtplus \UUtplus^\top - \UUtplusw \UUtplusw^\top
}    
\le 
3 \fronorm{
    \VXXT \left(\UUtplus \UUtplus^\top - \UUtplusw \UUtplusw^\top \right)
}    
\overleq{\eqref{ineq:intern786}} 
3c_2 \sigma_{\min} \left( \XXstar \right) ,
\end{align*}
which shows inequality \eqref{ineq:induction2} for $t+1$.
This completes the induction step.
\\

To complete the proof of Lemma \ref{lemma:phase1}
it remains to prove inequalities \eqref{ineq:localconv10} and \eqref{ineq:localconvT}.
For that, we first observe that
\begin{align*}
    \fronorm{\XXstar - \UUt \UUtT}
    \overleq{(a)} 
    & 3 \fronorm{\VXXT \left( \XXstar - \UUt \UUtT \right)}\\
    \overleq{(b)}
    &3\left(1 - \frac{\mu  \sigma_{\min} (\XXstar) }{16}\right)^t
    \fronorm{
       \VXXT \left( \XXstar - \UU_0 \UU_0^\top \right)
    }\\
    \overleq{(c)} 
    &3 \sqrt{2r} \left(1 - \frac{\mu  \sigma_{\min} (\XXstar) }{16}\right)^t
    \specnorm{
       \XXstar - \UU_0 \UU_0^\top
    }.
\end{align*}
Inequality $(a)$ follows from Lemma \ref{lemma:localconvaux} 
with $ \triplenorm{\cdot} = \fronorm{\cdot}  $
which is applicable since we have shown by induction that \eqref{ineq:induction6} holds for $0 \le t \le T$.
Inequality $(b)$ holds since we have proven \eqref{ineq:induction3} for all $0 \le t \le T$.
Inequality $(c)$ holds since $\XXstar - \UUt \UUtT$ has rank at most $2r$.
Thus, we can apply Lemma \ref{lemma:procrustebound} and obtain that
\begin{align*}
    \dist^2 \left( \UUt, \UUstar \right)
    &\le 
    \frac{\fronorm{\XXstar - \UU_t \UU_t^\top}^2}{ 2 \left(\sqrt{2} -1 \right) \sigma_{\min} \left(\XXstar \right) }\\
    &\le 
    18 r \left(1 - \frac{\mu  \sigma_{\min} (\XXstar) }{16}\right)^{2t}
    \cdot \frac{\specnorm{\XXstar - \UU_0 \UU_0^\top}^2}{ 2 \left(\sqrt{2}-1 \right) \sigma_{\min} (\XXstar) }\\
    &\le 
    \frac{9 c_{1} r }{ \left(\sqrt{2}-1 \right)}  \left(1 - \frac{\mu  \sigma_{\min} (\XXstar) }{16}\right)^{2t}
    \specnorm{\XXstar - \UU_0 \UU_0^{\top}},
\end{align*}
where in the last inequality, we have used assumption \eqref{assump:localconv1}. 
This proves inequality \eqref{ineq:localconv10}
since $c_1 \le \frac{\sqrt{2}-1}{9} $.
Next, we note that for $t=T$, the above inequality yields that
\begin{align*}
    \dist^2 \left( \UU_T, \UUstar \right)
    \overleq{(a)}
    &\frac{9 c^2_{1} r}{ \left( \sqrt{2}-1 \right)}  \left(1 - \frac{\mu  \sigma_{\min} (\XXstar) }{16}\right)^{2T}
    \sigma_{\min} (\XXstar) \\
    \overleq{(b)}
    &\frac{9 c^2_{1} r}{  \left(\sqrt{2}-1 \right)} 
    \exp \left( \frac{- T\mu  \sigma_{\min} (\XXstar) }{8}\right)
    \sigma_{\min} (\XXstar) \\
    \overleq{(c)}
    & \frac{ \sigma_{\min} (\XXstar)}{16}.
\end{align*}
In inequality $(a)$, we have used again assumption \eqref{assump:localconv1}.
Inequality $(b)$ is due to the elementary inequality $ \ln (1+x) \le x $ for $-1<x$ and the assumption $\mu<\frac{c_4}{\kappa \specnorm{\XXstar}}$ for sufficiently small $c_4>0$.
Inequality $(c)$ follows from 
$T = \Big\lceil \frac{8}{\mu \sigma_{\min} \left( \XXstar \right) } \log \left(  16r  \right)  \Big\rceil $
(and from the fact that we can choose $c_1 \le \frac{\sqrt{ \sqrt{2}-1 }}{3}$).
This proves inequality \eqref{ineq:localconvT}.
Thus, the proof of Lemma \ref{lemma:phase1} is complete.
\end{proof}

\subsection{Proof of Theorem \ref{thm:main}}\label{secproofmainresult}
Now we have all the ingredients in place to prove the main result of this paper,
Theorem \ref{thm:main}.
\begin{proof}[Proof of Theorem \ref{thm:main}]
In the following 
$c>0$ denotes a sufficiently small absolute constant.
First, by Lemma \ref{lem:rank_RIP} we know that due to our assumption $m \gtrsim rd \kappa^2$, with probability $1-\exp(-d)$
the measurement operator $\mathcal{A}$ satisfies the Restricted Isometry Property
of order $6r$ with a constant $\delta=\delta_{6r} \le \frac{c}{\kappa}$,
where $c>0$ is a sufficiently small absolute constant.

Set
\begin{equation*}
    T:=
  \bigg\lceil
    \frac{8}{\mu \sigma_{\min} (\XXstar)}
    \log
    \left(16r
    \right)
    \bigg\rceil
.
\end{equation*} 
Note that since $r\geq 1$ and the assumption $\mu\leq \frac{c_1}{\sigma_{\min}(\XXstar)}$ for small $c_1>0$, we have $T\geq 1$.
Let $\epscover$ be an $\varepsilon$-net  of the unit sphere in $\R^d$
with $\varepsilon=1/2$
such that $\vert \mathcal{N}_{\varepsilon} \vert \le 6^d$. Now note that $2 T\leq 6^d$,
where we have used the assumption 
$\mu \ge \frac{32}{ \sigma_{\min} (\XXstar) 6^d } \log \left( 16 r\right)$.
Thus, it follows from Lemma \ref{lemma:independencebound}
that
with probability at least $1-2\exp (-10d)$
it holds that
\begin{equation*}
    \vert   \innerproduct{\ww \ww^\top, \left(\Aops \right) \left( \Projwperp \left( \XXstar - \UUtw \UUtw^\top  \right)  \right)  }\vert \\
    \le 
    4 \sqrt{\frac{d}{m}} \twonorm{ \Aop \left( \Projwperp \left( \XXstar - \UUtw \UUtw^\top \right) \right) }
    \end{equation*}
    for all $\ww \in \epscover$
    and for all $ 0 \le t \le T $.
    Next, we know from Lemma \ref{lemma:spectralinitialization}
    and due to our assumption $m \gtrsim rd \kappa^2$ 
    that 
    with probability at least $1- 5\exp(-d)$, the inequalities 
    \begin{align}
        \specnorm{\XXstar - \UU_0 \UU_0^\top} 
        &\le  c \sigma_{\min} \left(\XXstar \right)  , \label{ineq:proofmain3}\\
        \fronorm{\UU_0 \UU_0^{\top}- \UU_{0,\ww} \UU_{0,\ww}^\top }
        &\le  c \sigma_{\min} \left(\XXstar \right)   \nonumber
    \end{align}
    hold for a sufficiently small constant $c>0$.
    Thus, all the assumptions of Lemma \ref{lemma:phase1}
    are fulfilled.
    It follows that
    \begin{equation}\label{ineq:proofmain1}
        \dist^2 \left( \UUt, \UUstar \right)
        \le
        r
        \left(
            1- \frac{\mu \sigma_{\min} (\XXstar) }{16}
        \right)^{2t}
        \specnorm{\XXstar - \UU_0 \UU_0^\top}
    \end{equation}
    for all $0 \le t \le T$ 
    and
    \begin{equation}\label{ineq:proofmain2}
        \dist
        \left( \UU_T, \UU_{\star} \right)
        \le
        \frac{\sigma_{\min} (\XXstar)}{16}.
    \end{equation}
    Due to inequality \eqref{ineq:proofmain2} and since $\delta_{6r}<1/10$
    we can apply Lemma \ref{lemma:Phase2}
    which yields that for $t\geq T$,
    \begin{align}\label{ineq:proofmain4}
     \dist^2 \left( \UUt, \UUstar \right)
    \le
    \left( 1 - c \mu \sigma_{\min} \left(\XXstar \right) \right)^{t-T} 
    \dist^2 \left( \UU_T, \UU_{\star} \right).
    \end{align}
    Thus, by combining  \eqref{ineq:proofmain3}, 
    \eqref{ineq:proofmain1}, and \eqref{ineq:proofmain4} 
    we obtain the conclusion of Theorem \ref{thm:main}.\end{proof}

\section{Discussions}\label{sec:discussions}

In this paper,
we have shown that for symmetric matrix sensing, factorized gradient descent can recover the ground truth matrix 
as soon as the number of samples satisfies $m \gtrsim rd \kappa^2$.
This improves over previous results in the literature with a quadratic rank dependence.
The key ingredient in our proof is a combination of a virtual sequence argument
with an $\varepsilon$-net argument.

Going forward, our work opens up a number of exciting research directions.
In the following, we highlight a few of these.

\paragraph{Breaking the quadratic rank barrier in related non-convex
matrix sensing problems:}
We expect that our novel proof technique will pave the way
to break the quadratic rank barrier in the sample complexity in various related non-convex matrix sensing problems.
This includes matrix sensing with an asymmetric ground truth matrix 
or overparameterized matrix sensing with small random initialization \cite{li2018algorithmic}.
One might also examine whether our new proof technique can be used to remove the additional rank factor in the sample complexity
in related algorithms such as \textit{scaled gradient descent} \cite{TongScaledGD} or \textit{GSMR} \cite{ZilberNadlerNonConvex}.

\paragraph{Convergence from random initialization:}
While our paper analyzes spectral initialization,
practitioners typically prefer random initialization.
To the best of our knowledge, 
establishing convergence from random initialization
remains an open problem in low-rank matrix recovery, 
even when allowing for polynomial rank-dependency in the sample complexity.
A notable exception is the rank-one case,
where in \cite{ChenGlobalPhaseRetrieval}
convergence of gradient descent without sample splitting
from random initialization was established
in the phase retrieval setting.
We believe that our proof techniques might
be helpful in solving this problem for the case with the rank greater than one.

\paragraph{Beyond Gaussian measurement matrices:}
Our proof crucially uses that the generalized entry 
$\langle \AAi, \ww \ww^{\top} \rangle$ of the measurement matrix $\AAi$
is independent of the matrix $\AAiw$,
i.e., the matrix which is obtained by deleting the generalized entry $\AAiw$. To satisfy this independence property,  the Darmois–Skitovich theorem \cite{darmois1953analyse}  implies that $\AAi$ has to have  Gaussian entries.

It would also be interesting to examine whether our argument can be
adapted to scenarios where the measurement matrices are no longer Gaussian, e.g., the matrix completion problem. 
Since as we mentioned the proof presented in this paper heavily relies on
the orthogonal invariance of the Gaussian distribution,
new insights are likely required to handle scenarios where this property is no longer available.
We believe that this is an exciting research direction.

\subsection*{Acknowledgements}
D.S. is grateful to Mahdi Soltanolkotabi for fruitful discussions,
in particular regarding Theorem \ref{thm:lowerbound},
and to Felix Krahmer for helpful comments.
Y.Z. was partially supported by NSF-Simons Research Collaborations on the Mathematical and Scientific
Foundations of Deep Learning and an AMS-Simons Travel Grant.
\bibliographystyle{plain}
\bibliography{ref.bib}

\newpage

\appendix

\section{Proof for the Spectral Initialization
(Proof of Lemma \ref{lemma:spectralinitialization})}\label{appendix:spectral}

 \begin{proof}[Proof of Lemma \ref{lemma:spectralinitialization}]
(1) We write \begin{align} \left(\Aops \right)(\XXstar)-\XXstar=\frac{1}{m}\sum_{i=1}^m \left( \innerproduct{\AAi, \XXstar} \AAi  -  \XXstar\right). 
\end{align}
Let $\widetilde{\mathcal N_{\varepsilon}}$ be any $\varepsilon$-net on $S^{d-1}$ with $\varepsilon=\frac{1}{2}$ of size at most $6^d$. 
Then we have 
\begin{align}
    \specnorm{\left(\Aops \right)(\XXstar)-\XXstar}\leq &2\sup_{\xx\in \widetilde{\mathcal N_{\varepsilon}}} \frac{1}{m}\sum_{i=1}^m \xx^\top \left( \innerproduct{\AAi, \XXstar} \AAi  -  \XXstar\right)\xx\\
    =&2\sup_{\xx\in \widetilde{\mathcal N_{\varepsilon}}} \frac{1}{m}\sum_{i=1}^m  \left( \innerproduct{\AAi, \XXstar}\xx^\top \AAi \xx  -  \xx^\top \XXstar \xx\right).
\end{align}
For each $i\in [m]$, 
we have that $\EE \left[  \innerproduct{\AAi, \XXstar}\xx^\top \AAi \xx   \right] = \xx^\top \XXstar \xx$.
Moreover,
the inner product $\innerproduct{\AAi, \XXstar}$ is a centered Gaussian random variable with variance $\|\XXstar\|_F^2$ and $\xx^\top \AAi \xx$ is a centered Gaussian random variable with variance $1$.
Thus, for each fixed $\xx$, 
$\sum_{i=1}^m  \left( \innerproduct{\AAi, \XXstar}\xx^\top \AAi \xx  -  \xx^\top \XXstar \xx\right)$ 
is a sum of $m$ independent and centered sub-exponential random variables with subexponential norm bounded by  $K\|\XXstar\|_F$, 
where $K$ is an absolute constant (see \cite[Lemma 2.7.7]{vershynin2018high}). 
Therefore, by Bernstein's inequality (see, for example, \cite[Theorem 2.8.1]{vershynin2018high}),
it holds that 
\begin{align}
\mathbb P\left(\left| \frac{1}{m}\sum_{i=1}^m  \left( \innerproduct{\AAi, \XXstar}\xx^\top \AAi \xx  -  \xx^\top \XXstar \xx\right) \right| \geq  t \right)\leq \exp \left( -C' \min \left\{ \frac{mt^2}{\|\XXstar\|_F^2}, \frac{mt}{\|\XXstar\|_F}\right\} \right),
\end{align}
where $C'>0$ is some absolute constant.
Taking $t=\frac{1}{8}C\|\XXstar\|_F\left(\sqrt{\frac{d}{m}}+\frac{d}{m}\right)$ and a union bound over all points $\xx$ on $\widetilde{\mathcal N_{\varepsilon}}$, we obtain 
\begin{align}\label{eq:bernsteinAA}
\specnorm{(\Aops)(\XXstar)-\XXstar}\leq \frac{1}{4} C \|\XXstar\|_F \left(\sqrt{\frac{d}{m}}+\frac{d}{m}\right)\leq \frac{1}{4}C\kappa\sigma_{\min}(\XXstar) \sqrt{r}\left(\sqrt{\frac{d}{m}}+\frac{d}{m}\right)
\end{align}
with probability at least $1-\exp(d\log(6)-C'C^2d)\geq 1-\exp(-4d)$  for some sufficiently large constant $C>0$.  

We assume that \eqref{eq:bernsteinAA} holds and that $m>  C^2\kappa^2rd$.
Then Weyl's inequalities imply that
\begin{align}
\lambda_r((\Aops)(\XXstar))>\frac{1}{2} \sigma_{\min}(\XXstar), \quad  |\lambda_{r+1}((\Aops)(\XXstar))|<  \frac{1}{2} \sigma_{\min}(\XXstar).
\end{align}
Since $\widetilde{\LLambda}_r$ is a diagonal matrix with entries $\lambda_1((\Aops)(\XXstar)),\dots, \lambda_r((\Aops)(\XXstar))$,
it follows from the definition of $\UU_0= \widetilde{\VV}_r \widetilde{\LLambda}_r^{1/2}$
that $ \UU_0 \UU_0^\top $ is the best rank-$r$ approximation of $(\Aops) (\XXstar)$. Consequently, we obtain that
\begin{align}
    \specnorm{
        \XXstar - \UU_0 \UU_0^\top
        } &\leq \specnorm{
        \XXstar - (\Aops)(\XXstar)
        }  +\specnorm{(\Aops)(\XXstar)- \UU_0 \UU_0^\top}\\
        &\leq \specnorm{
        \XXstar - (\Aops)(\XXstar)
        }  +\specnorm{(\Aops) (\XXstar)- \XXstar} \leq  C\kappa\sigma_{\min}(\XXstar)\sqrt{\frac{rd}{m}},
\end{align}
where in the second inequality, we used the Eckart-Young-Mirsky theorem.

(2)  Due to Lemma~\ref{lemma:AopIdentities} we have
\begin{align}\label{eq:Awops_X}
    (\Aopws-\IdOp)(\XXstar)=(\Aops-\IdOp)(\Projwperp (\XXstar))-  \innerproduct{ \Aop (\ww \ww^\top), \Aop \left( \Projwperp (\XXstar) \right)  } \ww \ww^\top.
\end{align}
It follows that
\begin{align}\label{eq:AwAwbound}
  \norm{(\Aopws-\IdOp)(\XXstar)}\leq  \norm{(\Aops-\IdOp)(\Projwperp (\XXstar))}+|\innerproduct{ \Aop (\ww \ww^\top), \Aop \left( \Projwperp (\XXstar) \right)  } |.
\end{align}
For a fixed $\ww\in \mathcal N_{\varepsilon}$,
we obtain with an analogous argument as for \eqref{eq:bernsteinAA} that with probability at least $1-\exp(-4d)$,
\begin{align}
    \norm{(\Aops-\IdOp)(\Projwperp (\XXstar))}\leq C \fronorm{ \Projwperp (\XXstar)}\left(\sqrt{\frac{d}{m}}+\frac{d}{m}\right)\leq \frac{1}{4}C\kappa\sigma_{\min}(\XXstar) \sqrt{r}\left(\sqrt{\frac{d}{m}}+\frac{d}{m}\right).
\end{align}
The second term in \eqref{eq:AwAwbound} can be rewritten as 
\begin{align}
    \innerproduct{ \Aop (\ww \ww^\top), \Aop \left( \Projwperp (\XXstar) \right)  } 
    &= \frac{1}{m} \sum_{i=1}^m  \innerproduct{\ww\ww^\top, \AAi} \innerproduct{\AAi,\Projwperp (\XXstar)}.
\end{align}
Here, $\sum_{i=1}^m  \innerproduct{\ww\ww^\top, \AAi} \innerproduct{\AAi,\Projwperp (\XXstar)}$
 is a sum of $m$  independent sub-exponential random variables with mean zero due to the rotation invariance of the Gaussian measure. 
Moreover, each term has sub-exponential norm $K\fronorm{\XXstar}$.
Applying Bernstein's inequality as in the proof of  \eqref{eq:bernsteinAA},
we obtain that for each fixed $\ww$ with probability at least $1-\exp(-4d)$,
\begin{align}\label{eq:AwwB}
     \innerproduct{ \Aop (\ww \ww^\top), \Aop \left( \Projwperp (\XXstar) \right)  }\leq \frac{1}{4}C\kappa\sigma_{\min}(\XXstar) \sqrt{r}\left(\sqrt{\frac{d}{m}}+\frac{d}{m}\right). 
\end{align}
Then, by taking a union bound over $\ww\in \mathcal N_{\varepsilon}$,  
it follows from \eqref{eq:AwAwbound}
that with probability at least $1-\exp(-2d)$ 
that for all $\ww\in \mathcal N_{\varepsilon}$ it holds that
\begin{align}\label{eq:Aopsw_I}
    \specnorm{(\Aopws-\IdOp)(\XXstar)}\leq \frac{1}{2}C\kappa\sigma_{\min}(\XXstar) \sqrt{r}\left(\sqrt{\frac{d}{m}}+\frac{d}{m}\right).
\end{align}

We now assume that \eqref{eq:Aopsw_I} holds and that $m> 4C^2\kappa^2rd$. 
Then it follows from Weyl's inequalities that
\begin{align}
\lambda_r((\Aopws)(\XXstar))>\frac{1}{2} \sigma_{\min}(\XXstar), \quad  |\lambda_{r+1}((\Aopws)(\XXstar))|<  \frac{1}{2} \sigma_{\min}(\XXstar).
\end{align}
It follows from the Eckart-Mirsky-Young theorem and the definition of $\UU_{0,\ww}$ 
that $ \UU_{0,\ww} \UU_{0,\ww}^\top $ is the best rank-$r$ approximation of $(\Aopws) (\XXstar)$. Therefore,
\begin{align}
    \specnorm{
        \XXstar - \UU_{0,\ww} \UU_{0,\ww}^\top
        } &\leq \specnorm{
        \XXstar - (\Aopws)(\XXstar)
        }  +\specnorm{(\Aopws)(\XXstar)- \UU_{0,\ww} \UU_{0,\ww}^\top}\\
        &\leq 2\specnorm{
        \XXstar - (\Aopws)(\XXstar)
        } \leq 2C\kappa\sigma_{\min}(\XXstar) \sqrt{\frac{rd}{m}} .
\end{align}
This finishes the proof of inequality \eqref{eq:XstarUwUw}.
Finally,   \eqref{eq:diff_U0_U0w} follows from \eqref{eq:XstartU0U0}  and \eqref{eq:XstarUwUw}  via the triangle inequality.

(3)  From \eqref{eq:Awops_X}, we have 
\begin{align}
 \left(\Aops \right)(\XXstar)
 - \left(\Aopws \right)(\XXstar)
 &=(\Aops-\IdOp)(\XXstar)- (\Aopws-\IdOp)(\XXstar)\\
 &=\langle \ww\ww^\top,\XXstar\rangle (\Aops-\IdOp)(\ww\ww^\top)+ \innerproduct{\Aop(\ww\ww^\top), \Aop(\Projwperp(\XXstar))}\ww\ww^\top.
\end{align}

It follows from Lemma~\ref{lem:rank_RIP}
that there exists an absolute constant $C_1>0$ 
such that for any $\alpha\in (0,1)$ and $m\geq \frac{C_1}{\alpha^2} \kappa^2 rd$, with probability at least $1-\exp(-d)$,
the measurement operator $\mathcal{A}$ satisfies the 
Restricted Isometry Property of order $6r$ with constant 
\begin{align}\label{eq:delta_c}
\delta:=\delta_{6r}\leq \frac{\alpha}{\kappa}.
\end{align}
Then for any $\VV\in \mathbb R^{d\times r}$ with orthonormal columns and for all $\ww\in\mathcal N_{\varepsilon}$, when $m\geq  \frac{C_1}{\alpha^2} \kappa^2 rd$, with probability at least $1-2\exp(-d)$,
\begin{align}
 &\fronorm{ (\Aops- \Aopws)(\XXstar)\VV}\\
 \leq & |\langle \ww\ww^\top,\XXstar\rangle|  \fronorm{(\Aops-\IdOp)(\ww\ww^\top)\VV}+|\innerproduct{\Aop(\ww\ww^\top), \Aop(\Projwperp(\XXstar))}| \fronorm{\ww\ww^\top\VV}\\
 \overleq{(a)} & \delta\specnorm{\XXstar}  \fronorm{\ww\ww^\top}+ |\innerproduct{\Aop(\ww\ww^\top), \Aop(\Projwperp(\XXstar))}| \\
 \overleq{(b)} & \alpha \sigma_{\min}(\XXstar)+ \frac{1}{2}C\kappa\sigma_{\min}(\XXstar) \sqrt{\frac{rd}{m}}. \label{eq:diff_Awops_Aops}
\end{align}
Here in (a) we use property \eqref{ineq:RIPlemma1} in Lemma \ref{lemma: RIP} and the fact that $\ww\ww^\top \VV$ is of rank 1, 
and in (b) we use  \eqref{eq:delta_c} and, moreover,  \eqref{eq:AwwB} with a union bound over $\ww\in\mathcal N_{\varepsilon}$.

We now proceed under the assumption that the inequalities in parts (1) and (2) hold. We use the following notations for spectral initialization: 
\begin{align}
    \left(\Aops \right)(\XXstar)
    &= \widetilde\VV\widetilde\LLambda \widetilde\VV^\top, \quad 
    \UU_{0}=\widetilde \VV_{r}{\widetilde\LLambda_{r}}^{1/2}, \label{eq:defU0}\\
    \left(\Aopws \right)(\XXstar)&= \widetilde\VV_{\ww}\widetilde\LLambda_{\ww} \widetilde\VV_{\ww}^\top, \quad 
    \UU_{0,\ww}=\widetilde \VV_{r,\ww}{\widetilde\LLambda_{r,\ww}}^{1/2}.
\end{align}
Denote  
\[\ZZ_{1}:= (\Aops)(\XXstar), \quad \ZZ_2:= (\Aopws)(\XXstar),\]
and 
\[ \ZZ_{1,r}:=\UU_0\UU_0^\top, \quad \ZZ_{2,r}:=\UU_{0,\ww}\UU_{0,\ww}^\top.
\]
Recall the definition of $\widetilde{\VV}_r$ and $\widetilde{\VV}_{r,\ww}$ in \eqref{eq:defU0} and \eqref{eq:defVw}.
We have 
\begin{align}
    \|\ZZ_{1,r}-\ZZ_{2,r}\|_F&=\|\UU_0\UU_0^\top-\UU_{0,\ww}\UU_{0,\ww}^\top\|_F\\
    &\leq \| \left( \UU_0\UU_0^\top-\UU_{0,\ww}\UU_{0,\ww}^\top\right) \widetilde{\VV}_{r}\|_F+\| \left( \UU_0\UU_0^\top-\UU_{0,\ww}\UU_{0,\ww}^\top\right) \widetilde{\VV}_{r,\perp}\|_F .\label{eq:twotermsZ1Z2}
\end{align}
For the first term in \eqref{eq:twotermsZ1Z2}, we have 
\begin{align}
     &\| \left( \UU_0\UU_0^\top-\UU_{0,\ww}\UU_{0,\ww}^\top\right) \widetilde{\VV}_{r}\|_F\\
     =&\| (\ZZ_1-\ZZ_{2,r})\widetilde{\VV}_r\|_F\\
     \leq & \| (\ZZ_1-\ZZ_{2})\widetilde{\VV}_r\|_F+\|(\ZZ_2-\ZZ_{2,r})\widetilde{\VV}_r\|_F\\
     = &\| (\ZZ_1-\ZZ_{2})\widetilde{\VV}_r\|_F+\|(\widetilde{\VV}_{r,\ww,\perp}\LLambda_{r,\ww,\perp}\widetilde{\VV}_{r,\ww,\perp}^\top)\widetilde{\VV}_r\|_F\\
     \leq & \| (\ZZ_1-\ZZ_{2})\widetilde{\VV}_r\|_F+\sigma_{r+1}(\ZZ_2) \| \widetilde{\VV}_{r,\ww,\perp}^\top\widetilde{\VV}_r\|_F\\
     \leq & \| (\ZZ_1-\ZZ_{2})\widetilde{\VV}_r\|_F+C\kappa\sigma_{\min}(\XXstar)  \sqrt{\frac{rd}{m}} \| \widetilde{\VV}_{r,\ww,\perp}^\top\widetilde{\VV}_r\|_F, \label{eq:UUUwUw}
\end{align}
where in the last inequality we used Weyl's inequality and \eqref{eq:bernsteinAA}, 
which implies 
\begin{align}\label{eq:sigma_r_Z2}
    \sigma_{r+1}(\ZZ_2) = |\sigma_{r+1}(\ZZ_2)-\sigma_{r+1}(\XXstar)|\leq \| \ZZ_2-\XXstar\| \leq C\kappa\sigma_{\min}(\XXstar)  \sqrt{\frac{rd}{m}}.
\end{align}
From \eqref{eq:Aopsw_I}  and \eqref{eq:bernsteinAA},
it follows that when $m\geq C^2\kappa^2 rd$,
\begin{align}\label{eq:Z1Z2}
    \|\ZZ_1-\ZZ_2\| \leq \frac{3C}{2}\kappa\sigma_{\min}(\XXstar) \sqrt{\frac{rd}{m}}.
\end{align}
Similar to \eqref{eq:sigma_r_Z2}, using \eqref{eq:bernsteinAA} and Weyl's inequalities
we obtain that
\begin{align}
    |\sigma_r(\ZZ_1)-\sigma_{\min}(\XXstar)| &\leq C\kappa\sigma_{\min}(\XXstar) \sqrt{\frac{rd}{m}},\\
   \sigma_{r+1}(\ZZ_1) &\leq C\kappa\sigma_{\min}(\XXstar) \sqrt{\frac{rd}{m}}.
\end{align}
Therefore, if $m>16C^2\kappa^2rd$, 
the spectral gap between $\sigma_r (\ZZ_1)$ and $\sigma_{r+1} (\ZZ_1)$ 
can be bounded from below by
\begin{align}
    \sigma_r(\ZZ_1)-\sigma_{r+1}(\ZZ_1) \geq \left(1-2C\kappa \sqrt{\frac{rd}{m}}\right)\sigma_{\min}(\XXstar)\geq \frac{1}{2}\sigma_{\min}(\XXstar). \label{eq:Z1Z22}
\end{align}
When  $m\geq 51C^2\kappa^2 rd$,
we have from \eqref{eq:Z1Z2} and \eqref{eq:Z1Z22},
\begin{align}
    \specnorm{\ZZ_1-\ZZ_2}&\leq \frac{3C}{2}\kappa \sqrt{\frac{rd}{m}}\sigma_{\min}(\XXstar) \\
   & \leq \bigg(1-\frac{1}{\sqrt 2}\bigg)\left(1-2C\kappa  \sqrt{\frac{rd}{m}}  \right)\sigma_{\min}(\XXstar)\\
   &\leq 
   \bigg(1-\frac{1}{\sqrt 2}\bigg) (\sigma_r(\ZZ_1)-\sigma_{r+1}(\ZZ_1)).
\end{align}
Thus, the prerequisites of Lemma \ref{lem:DavisKahan} (Davis-Kahan inequality) are satisfied. 
It follows that when $m\geq 51C^2\kappa^2 rd$,
\begin{align}\label{eq:subspace_angle}
    \| \widetilde{\VV}_{r,\ww,\perp}^\top\widetilde{\VV}_r\|_F
    \leq 
    \frac{2\sqrt{2}\| (\ZZ_1-\ZZ_{2})\widetilde{\VV}_r\|_F }{\sigma_{\min}(\XXstar)}.
\end{align}
Hence,
when $m\geq \left(51C^2+\frac{C_1}{\alpha^2}\right)\kappa^2 rd$,
we obtain from \eqref{eq:UUUwUw} and \eqref{eq:diff_Awops_Aops} that
\begin{align}
     \| \left( \UU_0\UU_0^\top-\UU_{0,\ww}\UU_{0,\ww}^\top\right) \widetilde{\VV}_{r}\|_F&
     \leq 
     \left( 1+2\sqrt{2}C\kappa \sqrt{\frac{rd}{m}}\right)
     \fronorm{(\ZZ_1-\ZZ_2)\widetilde{\VV
     }_r}\\
     &\leq 2\fronorm{(\ZZ_1-\ZZ_2)\widetilde{\VV
     }_r}\leq \left( 2\alpha +C\kappa\sqrt{\frac{rd}{m}}\right)\sigma_{\min}(\XXstar).\label{eq:rdm}
\end{align}
   For the second term in \eqref{eq:twotermsZ1Z2}, we have when $m\geq \left(51C^2+\frac{C_1}{\alpha^2}\right)\kappa^2 rd$,
\begin{align}
    &\| \left( \UU_0\UU_0^\top-\UU_{0,\ww}\UU_{0,\ww}^\top\right) \widetilde{\VV}_{r,\perp}\|_F\\
    \leq & \| \widetilde{\VV}_r^\top\left( \UU_0\UU_0^\top-\UU_{0,\ww}\UU_{0,\ww}^\top\right) \widetilde{\VV}_{r,\perp}\|_F+\| \widetilde{\VV}_{r,\perp}^\top\left( \UU_0\UU_0^\top-\UU_{0,\ww}\UU_{0,\ww}^\top\right) \widetilde{\VV}_{r,\perp}\|_F \\
    \leq & \| \widetilde{\VV}_r^\top\left( \UU_0\UU_0^\top-\UU_{0,\ww}\UU_{0,\ww}^\top\right) \|_F+\| \widetilde{\VV}_{r,\perp}^\top \UU_{0,\ww}\UU_{0,\ww}^\top  \widetilde{\VV}_{r,\perp}\|_F\\
    \leq & \left( 2\alpha +C\kappa\sqrt{\frac{rd}{m}}\right)\sigma_{\min}(\XXstar)+\| \widetilde{\VV}_{r,\perp}^\top \UU_{0,\ww}\UU_{0,\ww}^\top  \widetilde{\VV}_{r,\perp}\|_F, \label{eq:28second1}
\end{align}
where the last inequality is due to \eqref{eq:rdm}.

We now consider the second term in \eqref{eq:28second1}. Recall the definition of $\UU_{0,\ww}$ in \eqref{eq:defU0w}. We have for $m\geq \left(51C^2+\frac{C_1}{\alpha^2}\right)\kappa^2 rd$,
\begin{align}
\| \widetilde{\VV}_{r,\perp}^\top \UU_{0,\ww}\UU_{0,\ww}^\top  \widetilde{\VV}_{r,\perp}\|_F&=\fronorm{\widetilde{\VV}_{r,\perp}^\top\widetilde{\VV}_{r,\ww}\LLambda_{r,\ww}\widetilde{\VV}_{r,\ww}^\top\widetilde{\VV}_{r,\perp}}\\
&\leq \specnorm{\widetilde{\VV}_{r,\perp}^\top\widetilde{\VV}_{r,\ww}\LLambda_{r,\ww}}\fronorm{\widetilde{\VV}_{r,\ww}^\top\widetilde{\VV}_{r,\perp}}\\
&=\sqrt{\specnorm{\widetilde{\VV}_{r,\perp}^\top\widetilde{\VV}_{r,\ww}\LLambda^2_{r,\ww}\widetilde{\VV}_{r,\ww}^\top\widetilde{\VV}_{r,\perp}}}\fronorm{\widetilde{\VV}_{r,\ww}^\top\widetilde{\VV}_{r,\perp}}\\
&=\sqrt{\specnorm{\widetilde{\VV}_{r,\perp}^\top(\UU_{0,\ww}\UU_{0,\ww}^\top)^2\widetilde{\VV}_{r,\perp}}}\fronorm{\widetilde{\VV}_{r,\ww}^\top\widetilde{\VV}_{r,\perp}}\\
&=\specnorm{\widetilde{\VV}_{r,\perp}^\top\UU_{0,\ww}\UU_{0,\ww}^\top}\fronorm{\widetilde{\VV}_{r,\ww}^\top\widetilde{\VV}_{r,\perp}}\\
&=\specnorm{\widetilde{\VV}_{r,\perp}^\top(\UU_{0,\ww}\UU_{0,\ww}^\top-\UU_0\UU_0^\top)}\fronorm{\widetilde{\VV}_{r,\ww}^\top\widetilde{\VV}_{r,\perp}}\\
&\leq \specnorm{\UU_{0,\ww}\UU_{0,\ww}^\top-\UU_0\UU_0^\top}\fronorm{\widetilde{\VV}_{r,\ww}^\top\widetilde{\VV}_{r,\perp}}\\
&\overleq{(a)} 3C\kappa\sigma_{\min}(\XXstar)\sqrt{\frac{rd}{m}} \cdot \frac{2\sqrt{2}\| (\ZZ_1-\ZZ_{2})\widetilde{\VV}_r\|_F }{\sigma_{\min}(\XXstar)}\\
&\overleq{(b)} 6\sqrt{2} C\kappa \left(\alpha+\frac {1}{2}C \kappa \sqrt{\frac{rd}{m}} \right)\sqrt{\frac{rd}{m}}\sigma_{\min}(\XXstar), \label{eq:28second2}
\end{align}
where (a) is due to \eqref{eq:diff_U0_U0w} and \eqref{eq:subspace_angle}, and (b) is due to \eqref{eq:diff_Awops_Aops}.  Therefore from \eqref{eq:28second1} and \eqref{eq:28second2}, we obtain for $m\geq \left(51C^2+\frac{C_1}{\alpha^2}\right)\kappa^2 rd$,
\begin{align}\label{eq:rdm2}
    &\| \left( \UU_0\UU_0^\top-\UU_{0,\ww}\UU_{0,\ww}^\top\right) \widetilde{\VV}_{r,\perp}\|_F\\
    \leq & \left( 2\alpha +C\kappa\sqrt{\frac{rd}{m}}\right)\sigma_{\min}(\XXstar)
    +6\sqrt{2} C\kappa \left( \alpha +\frac{1}{2}C \kappa \sqrt{\frac{rd}{m}} \right)\sqrt{\frac{rd}{m}}\sigma_{\min}(\XXstar). 
\end{align}
From \eqref{eq:rdm}, \eqref{eq:rdm2}, and \eqref{eq:twotermsZ1Z2}, we conclude that if $m\geq \left(51C^2+\frac{C_1}{\alpha^2}\right)\kappa^2 rd$,
\begin{align}
     \fronorm{
        \UU_{0} \UU_{0}^\top
        -
        \UU_{0,\ww} \UU_{0,\ww}^\top
        }\leq \left(  2\alpha +C\kappa\sqrt{\frac{rd}{m}}\right)\left(2\sigma_{\min}(\XXstar)+3\sqrt 2 C\kappa \sqrt{\frac{rd}{m}}\sigma_{\min}(\XXstar)\right).
\end{align}
This finishes the proof of \eqref{eq:diff_r_Awops_Aops}.
 \end{proof}

\section{Proofs of lemmas concerning the distance 
between the virtual sequences and the original sequence}\label{sec:auxsequenceproofs}

\subsection{Some auxiliary estimates}
In order to prove Lemma \ref{lemma:auxdistanceweak} and Lemma \ref{lemma:auxsequencecloseness}
we will need several auxiliary estimates.
These are summarized in the following lemma.

\begin{lemma}\label{lemma:auxestimates}
Assume that the measurement operator $\mathcal{A}$
has the Restricted Isometry Property with constant $\delta=\delta_{4r+1} \le 1$.
Moreover, assume that the conclusion of Lemma \ref{lemma:independencebound} holds. 
Then, the following inequalities hold.
\begin{enumerate}
\item
\begin{align}
    &\fronorm{\left[ \left(  \Aops - \Aopws \right) \left(\XXstar- \UUt \UUt^\top \right) \right]\VV_{\UUtw}}\\
    \le
    &\left(\delta +  \frac{8 \sqrt{rd}}{\sqrt{m}}  \right) 
    \specnorm{ \XXstar - \UUt \UUt^\top }
    +\left( \delta  +  \frac{4 \sqrt{ 2 d}}{\sqrt{m}}  \right) \fronorm{\UUt \UUt^\top - \UUtw \UUtw^\top},
    \label{ineq:intern4}
\end{align}
\item 
\begin{align}
    \fronorm{
        \left[\left(   \Aopws - \IdOp \right) \left( \UUtw \UUtw^\top - \UUt \UUt^\top \right) \right] \VV_{\UUtw}
    } 
    \le 2 \delta \fronorm{  \UUtw \UUtw^\top - \UUt \UUt^\top  },
    \label{ineq:intern5}
\end{align}
\item
\begin{equation}
    \begin{split}
    \fronorm{\left[\left(  \Aops - \Aopws \right) \left(\XXstar- \UUtw \UUtw^\top \right) \right] \VV_{\UUtw}}
    \le
    \left( \delta+ 8 \sqrt{\frac{rd}{m}} \right) \specnorm{\XXstar - \UUtw \UUtw^\top},
    \end{split}
    \label{ineq:intern7}
\end{equation}
\item and
\begin{align}
   \specnorm{ \left( \Aopws -\IdOp \right) \left( \XXstar - \UUtw \UUtw^\top \right) }\nonumber
   \le
   &\specnorm{ \left( \Aops - \IdOp \right) \left(\XXstar - \UUt \UUt^\top \right) }
   +\left(  \delta +8 \sqrt{\frac{rd}{m}} \right)\specnorm{\XXstar - \UUt \UUt^\top }\nonumber \\
   &+\left(  2\delta + 4 \sqrt{\frac{ 2d}{m}}  \right)\fronorm{\UUtw \UUtw^\top -\UUt \UUt^\top}
   \label{ineq:intern9}.
\end{align}
\end{enumerate}
\end{lemma}

\begin{proof}[Proof of Lemma \ref{lemma:auxestimates}]
To prove inequality \eqref{ineq:intern4}, we compute that
\begin{align*}
    (\Aopws) \left(\XXstar- \UUt \UUt^\top \right)
    =
    &(\Aopws) \left( \Projw \left( \XXstar- \UUt \UUt^\top \right) \right) 
    + (\Aopws) \left( \Projwperp \left(\XXstar- \UUt \UUt^\top \right) \right)\\
    \overeq{(a)}
    &(\Aops) \left(  \Projwperp \left( \XXstar- \UUt \UUt^\top \right) \right)
    +\Projw \left( \XXstar- \UUt \UUt^\top \right)\\
    &-\innerproduct{\Aop \left(\ww \ww^\top \right), \Aop \left( \Projwperp \left(\XXstar - \UUt \UUt^\top \right) \right)} \ww \ww^\top,
\end{align*}
where in equation $(a)$ we used Lemma \ref{lemma:AopIdentities}.
It follows that
\begin{align}
    \left(  \Aops - \Aopws \right) \left(\XXstar- \UUt \UUt^\top \right)
    =&
    \left( \Aops - \IdOp \right)
    \left( \Projw \left( \XXstar- \UUt \UUt^\top \right) \right) \nonumber \\
    &+\innerproduct{\Aop (\ww \ww^\top), \Aop \left( \Projwperp (\XXstar - \UUt \UUt^\top) \right)} \ww \ww^\top. \label{ineq:intern6}
\end{align}
By using the triangle inequality, we obtain the estimate
\begin{align}
    &\fronorm{\left(  \Aops - \Aopws \right) \left(\XXstar- \UUt \UUt^\top \right) \VV_{\UUtw}} \nonumber \\
    \le&
    \fronorm{ \left( \Aops - \IdOp \right) \left( \Projw \left( \XXstar- \UUt \UUt^\top \right) \right)\VV_{\UUtw}} 
    +
    \fronorm{\innerproduct{\Aop (\ww \ww^\top), \Aop \left( \Projwperp (\XXstar - \UUt \UUt^\top) \right)} \ww \ww^\top} \nonumber \\
    \overleq{(a)} &
    \delta \fronorm{ \Projw \left( \XXstar- \UUt \UUt^\top \right) }
    + \big\vert \innerproduct{\Aop (\ww \ww^\top), \Aop \left( \Projwperp (\XXstar - \UUt \UUt^\top) \right)} \big\vert \nonumber\\
    \overleq{(b)} &
    \delta \specnorm{ \XXstar - \UUt \UUt^\top }
    + \big\vert \innerproduct{ \Aop (\ww \ww^\top), \Aop \left( \Projwperp (\XXstar - \UUtw \UUtw^\top)  \right) } \big\vert \nonumber \\ 
    +& \big\vert \innerproduct{ \Aop (\ww \ww^\top), \Aop \left( \Projwperp (\UUt \UUt^\top - \UUtw \UUtw^\top)  \right) } \big\vert \nonumber \\
    \overleq{(c)} &
    \delta \specnorm{ \XXstar - \UUt \UUt^\top }
    + \frac{4 \sqrt{d}}{\sqrt{m}} \twonorm{\Aop \left( \Projwperp (\XXstar - \UUtw \UUtw^\top ) \right)}
    + \delta \fronorm{\UUt \UUt^\top - \UUtw \UUtw^\top} \nonumber \\
    \overleq{(d)} &
    \delta \specnorm{ \XXstar - \UUt \UUt^\top }
    + \frac{ 4 \sqrt{2d}}{\sqrt{m}} \fronorm{ \Projwperp (\XXstar - \UUtw \UUtw^\top ) }
    + \delta \fronorm{\UUt \UUt^\top - \UUtw \UUtw^\top} \nonumber \\
    \overleq{(e)} &
    \delta \specnorm{ \XXstar - \UUt \UUt^\top }
    + \frac{ 4 \sqrt{2d}}{\sqrt{m}} \fronorm{ \XXstar - \UUt \UUt^\top }
    + \left( \delta + \frac{4 \sqrt{2d} }{\sqrt{m}} \right) \fronorm{\UUt \UUt^\top - \UUtw \UUtw^\top} \nonumber \\
    \overleq{(f)} &
    \left(\delta +  \frac{ 8 \sqrt{rd}}{\sqrt{m}}  \right) \specnorm{ \XXstar - \UUt \UUt^\top }
    +  \left( \delta  +  \frac{ 4 \sqrt{2d}}{\sqrt{m}}  \right) \fronorm{\UUt \UUt^\top - \UUtw \UUtw^\top}.
    \nonumber 
\end{align}
Inequality $(a)$ follows from the RIP-assumption combined with Lemma \ref{lemma: RIP} 
and from the fact that $\twonorm{\ww}=1$.
Inequality $(b)$ is a consequence of the fact that $\Projw$ is a rank-one projection and of the triangle inequality.
In inequality $(c)$, we used that the conclusion of Lemma \ref{lemma:independencebound} holds and Lemma \ref{lemma: RIP}.
In inequality $(d)$, we used the RIP of rank $2r+1$.
Inequality $(e)$ is due to the fact that $\Projwperp$ is an orthogonal projection and due to the triangle inequality.
In inequality $(f)$, we used that $\XXstar - \UUt \UUt^\top$ has rank at most $2r$.
This proves inequality \eqref{ineq:intern4}.

To prove inequality \eqref{ineq:intern5} we compute first that 
\begin{align*}
   &\left(  \Aopws - \IdOp \right) \left( \UUtw \UUtw^\top - \UUt \UUt^\top \right)\\
   =
   &\left(  \Aops - \IdOp \right) \left( \Projwperp \left( \UUtw \UUtw^\top - \UUt \UUt^\top \right) \right)
   -
   \innerproduct{ \Aop (\ww \ww^\top), \Aop \left( \Projwperp (\UUtw \UUtw^\top - \UUt \UUt^\top) \right)  } \ww \ww^\top.
\end{align*}
It follows that 
\begin{align}
    &\fronorm{
        \left[\left(   \Aopws - \IdOp \right) \left( \UUtw \UUtw^\top - \UUt \UUt^\top \right) \right] \VV_{\UUtw}
    } \nonumber \\
    \overleq{(a)}& 
    \delta \fronorm{ \Projwperp \left( \UUtw \UUtw^\top - \UUt \UUt^\top \right) } 
    +
    \big\vert \innerproduct{  \Aop (\ww \ww^\top), \Aop \left( \Projwperp (\UUtw \UUtw^\top - \UUt \UUt^\top) \right) } \big\vert \nonumber \\
    \overleq{(b)}& 
    2\delta \fronorm{ \Projwperp ( \UUtw \UUtw^\top - \UUt \UUt^\top ) }\nonumber\\
    \le& 
    2\delta \fronorm{  \UUtw \UUtw^\top - \UUt \UUt^\top  }. \nonumber
\end{align}
In inequalities $(a)$ and $(b)$ we used Lemma \ref{lemma: RIP}.
This proves inequality \eqref{ineq:intern5}.

Next, we prove the third inequality.
For that, we observe that using Lemma \ref{lemma:AopIdentities} it holds that
\begin{align*}
    \left(  \Aops - \Aopws \right) \left(\XXstar- \UUtw \UUtw^\top \right)
    =
    &\left( \Aops - \IdOp \right)
    \left( \Projw \left( \XXstar- \UUtw \UUtw^\top \right) \right) \nonumber \\
    &+\innerproduct{\Aop (\ww \ww^\top), \Aop \left( \Projwperp (\XXstar - \UUtw \UUtw^\top) \right)}\ww\ww^\top.
\end{align*}
Then it follows that 
\begin{align}
    &\fronorm{\left(  \Aops - \Aopws \right) \left(\XXstar- \UUtw \UUtw^\top \right) \VV_{\UUtw}} \nonumber\\
    \le
    &\fronorm{\left( \Aops - \IdOp \right)
    \left( \Projw \left( \XXstar- \UUtw \UUtw^\top \right) \right) \VV_{\UUtw} }
    +
    \vert \innerproduct{\Aop (\ww \ww^\top), \Aop \left( \Projwperp (\XXstar - \UUtw \UUtw^\top) \right)} \vert \nonumber \\
    \overleq{(a)}
    &\delta \specnorm{\XXstar - \UUtw \UUtw^\top}
    + 4 \sqrt{\frac{d}{m}} \twonorm{    \Aop \left( \Projwperp (\XXstar - \UUtw \UUtw^\top) \right)  }
    \nonumber\\
    \overleq{(b)}
    &\delta \specnorm{\XXstar - \UUtw \UUtw^\top}
    + 4 \sqrt{\frac{ 2 d}{m}} \fronorm{   \XXstar - \UUtw \UUtw^\top  } \nonumber \\
    \le 
    & \left( \delta + 8 \sqrt{\frac{rd}{m}} \right) \specnorm{\XXstar - \UUtw \UUtw^\top},
    \nonumber 
\end{align}
where  inequality $(a)$ holds due to Lemma ~\ref{lemma: RIP},
since $  \Projwperp$ is a rank-one projection,
and since we assumed that the conclusion of Lemma \ref{lemma:independencebound} holds.
Inequality $(b)$ is again due to Lemma~\ref{lemma: RIP}
and since $\Projwperp$ is an orthogonal projection.
This proves inequality \eqref{ineq:intern7}.

It remains to prove inequality \eqref{ineq:intern9}.
We note that it holds that
\begin{align}\label{eq:Awops_minus_I}
&\left(  \Aopws - \IdOp \right) \left(\XXstar- \UUtw \UUtw^\top \right) \\
=
&\left(  \Aops - \IdOp \right) \left( \Projwperp \left(\XXstar- \UUtw \UUtw^\top \right) \right) 
- 
\innerproduct{\Aop(\ww \ww^\top ), \Aop ( \Projwperp \left(\XXstar- \UUtw \UUtw^\top \right) )} \ww \ww^\top,
\end{align}
where in the last line we applied Lemma \ref{lemma:AopIdentities}.
It follows from the triangle inequality that
\begin{align}
&\specnorm{ \left(  \Aopws - \IdOp \right) \left(\XXstar- \UUtw \UUtw^\top \right)  }\nonumber\\
\le 
&\specnorm{\left(  \Aops - \IdOp \right) \left( \XXstar- \UUt \UUt^\top \right)  }
+
\specnorm{\left(  \Aops - \IdOp \right) \left( \Projw \left(\XXstar- \UUtw \UUtw^\top \right) \right)  }\nonumber\\
&+
\specnorm{\left(  \Aops - \IdOp \right) \left(  \UUt \UUt^\top - \UUtw \UUtw^\top  \right)  }
+\big\vert  \innerproduct{\Aop \left(\ww \ww^\top \right), \Aop \left( \Projwperp \left(\XXstar- \UUtw \UUtw^\top \right) \right)} \big\vert
\nonumber\\
\overleq{(a)}
&  \specnorm{ \left( \Aops - \IdOp \right) \left(\XXstar - \UUt \UUt^\top \right) } 
+ \delta \fronorm{\Projw \left(\XXstar- \UUtw \UUtw^\top \right)}
+\delta \fronorm{\UUtw \UUtw^\top -\UUt \UUt^\top} \nonumber \\
&+ 4\sqrt{\frac{2 d}{m}} \fronorm{\XXstar- \UUtw \UUtw^\top}
\nonumber\\
\le
&  \specnorm{ \left( \Aops - \IdOp \right) \left(\XXstar - \UUt \UUt^\top \right) } 
+ \delta \specnorm{\XXstar - \UUt \UUt^\top }\\
&+ 2\delta \fronorm{\UUtw \UUtw^\top -\UUt \UUt^\top} 
+4 \sqrt{\frac{2 d}{m}} \fronorm{\XXstar- \UUtw \UUtw^\top}
\nonumber \\
\le
&  \specnorm{ \left( \Aops - \IdOp \right) \left(\XXstar - \UUt \UUt^\top \right) } 
+\left(  \delta + 8 \sqrt{\frac{rd}{m}} \right)\specnorm{\XXstar - \UUt \UUt^\top } \\
&+\left(  2\delta+ 4 \sqrt{\frac{ 2 d}{m}}  \right)\fronorm{\UUtw \UUtw^\top -\UUt \UUt^\top}\nonumber.
\end{align}
In inequality $(a)$ we applied Lemma \ref{lemma: RIP}
and that the conclusion of Lemma \ref{lemma:independencebound} holds.
This proves inequality \eqref{ineq:intern9}.
Thus, the proof of Lemma \ref{lemma:auxestimates} is complete.
\end{proof}

\subsection{Proof of Lemma \ref{lemma:auxdistanceweak}}\label{subsec:proofauxdistanceweak}

\begin{proof}[Proof of Lemma \ref{lemma:auxdistanceweak}]
    We define the shorthand notation
    \begin{align*}
     \MM_t     &:= \left( \Aops \right) \left( \XXstar - \UUt \UUtT \right), \\
     \MM_{t,\ww} &:= \left( \Aopws \right) \left( \XXstar - \UUtw \UUtw^\top \right).
    \end{align*} 
    It follows that 
    \begin{align*}
     \UUtplus&= \left( \Id + \mu \MM_t \right) \UUt, \\
     \UUtplusw&= \left( \Id + \mu \MM_{t,\ww} \right) \UUtw.
    \end{align*}
    We compute that 
    \begin{align*}
     \UUtplus \UUtplus^\top - \UUtw \UUtplusw
     =
    &\left( \Id +\mu \MM_t \right) \UUt \UUt^\top (\Id + \mu \MM_t)
    -
    \left( \Id +\mu \MM_{t,\ww} \right) \UUtw \UUtw^\top (\Id + \mu \MM_{t,\ww})\\
    =
    &\UUt \UUtT - \UUtw \UUtw^\top
    +\mu \bracing{=:(i)}{ \MM_t (\UUt \UUtT -\UUtw \UUtw^\top)}
    +\mu \bracing{=:(ii)}{ \left( \MM_t - \MM_{t,\ww} \right) \UUtw \UUtw^\top}\\
    &+\mu \bracing{=:(iii)}{(\UUt \UUtT -\UUtw \UUtw^\top) \MM_t}
    +\mu  \bracing{=:(iv)}{ \UUtw \UUtw^\top \left( \MM_t - \MM_{t,\ww} \right)}\\
    &+\mu^2 
    \bracing{=:(v)}{\left( \MM_t \UUt \UUtT \MM_t - \MM_{t,\ww} \UUtw \UUtw^\top \MM_{t,\ww} \right)}.
    \end{align*}
    We want to estimate the spectral norm of these terms individually.
    Before that, we note that
    \begin{align}
     \specnorm{\MM_t}
     \overleq{(a)}
     &\specnorm{\XXstar - \UUt \UUt}
     +
     \specnorm{ \left(\Aops - \IdOp \right) (\XXstar - \UUt \UUtT)}\nonumber \\
     \overleq{(b)}
     &  \specnorm{\XXstar - \UUt \UUtT}
     + \constone \sigma_{\min} \left( \XXstar \right)
     \label{ineq:weakboundintern5}\\
     \overleq{(c)}
     & 2 \sigma_{\min} \left( \XXstar \right).
     \label{ineq:weakboundintern9}
    \end{align}
   Inequality $(a)$ follows from the triangle inequality and
   inequality $(b)$ follows from assumption \eqref{ineq:weakboundintern2}. 
   Inequality $(c)$ is a consequence of assumption \eqref{ineq:weakboundintern3}.
   Moreover, we note that 
   \begin{align*}
     \MM_t - \MM_{t,\ww}
     =
     \left( \Aops - \Aopws \right) \left( \XXstar - \UUt \UUtT \right)
     -
     \left( \Aopws \right) \left(\UUt \UUtT - \UUtw \UUtw^\top \right).
   \end{align*}
   It follows that
   \begin{align}
     &\fronorm{ \left( \MM_t - \MM_{t,\ww} \right) \VV_{\UUtw} }\\
     \le
     &\fronorm{ \left[\left( \Aops - \Aopws \right) \left( \XXstar - \UUt \UUtT \right)\right] \VV_{\UUtw}} 
     +
     \fronorm{\left[ \left( \Aopws -\IdOp \right) \left(\UUt \UUtT - \UUtw \UUtw^\top \right) \right] \VV_{\UUtw}} \nonumber \\
     &+
     \fronorm{\UUt \UUtT - \UUtw \UUtw^\top} \nonumber \\
     \overleq{(a)}
     &\left( \delta + \frac{8 \sqrt{rd} }{\sqrt{m}} \right) 
     \specnorm{\XXstar - \UUt \UUtT}
     +
     \left( 3\delta+ \frac{4 \sqrt{2d}}{\sqrt{m}}+1 \right) \fronorm{ \UUt \UUtT - \UUtw \UUtw^\top } \nonumber\\
     \overleq{(b)}
     &\frac{2\constthree}{\kappa} \specnorm{\XXstar - \UUt \UUtT}
     + \left( \frac{4 \constthree}{\kappa} +1 \right) \fronorm{\UUt \UUtT - \UUtw \UUtw^\top},
     \label{ineq:weakboundintern7}
   \end{align}
   where in inequality $(a)$ we used inequalities \eqref{ineq:intern4} and \eqref{ineq:intern5} from Lemma \ref{lemma:auxestimates}.
   Inequality (b) is due to assumption \eqref{ineq:weakboundintern11}.
   Note that it also follows from these estimates that 
   \begin{align}
     \specnorm{\MM_{t,\ww} \VV_{\UUtw}}
     &\le 
     \specnorm{\MM_t}
     +
     \fronorm{\left(\MM_t -\MM_{t,\ww} \right) \VV_{\UUtw}} \nonumber \\
     &\overleq{(a)} 
     2 \sigma_{\min} (\XXstar)
     +
     \frac{2 \constthree}{\kappa} \specnorm{\XXstar - \UUt \UUtT} 
     +\left( \frac{4 \constthree }{\kappa}+1 \right) \fronorm{\UUt \UUtT - \UUtw \UUtw^\top} \nonumber \\
     &\overleq{(b)} 
     3 \sigma_{\min} (\XXstar), \label{ineq:weakboundintern8}
   \end{align}
   where inequality $(a)$ follows from \eqref{ineq:weakboundintern7}.
   Inequality $(b)$ is a consequence of the 
   assumptions \eqref{ineq:weakboundintern3}
   and \eqref{ineq:weakboundintern4}
   (and by choosing the absolute constant $\constthree >0$ small enough).
 
    Now we are in a position to estimate the Frobenius norms of the terms $(i)$-$(v)$.
    \paragraph*{Estimating term (i):}
    We compute that that 
    \begin{align*}
     \fronorm{\MM_t (\UUt \UUtT -\UUtw \UUtw^\top)}
     \le 
     &\specnorm{\MM_t}
     \fronorm{\UUt \UUtT -\UUtw \UUtw^\top}\\
     \overleq{\eqref{ineq:weakboundintern5}}
     & 
     \left( \specnorm{\XXstar - \UUt \UUtT} + \constone \sigma_{\min} (\XXstar) \right)
     \fronorm{\UUt \UUtT -\UUtw \UUtw^\top}.
    \end{align*}
 
    \paragraph*{Estimating term (ii):}
    We compute that 
    \begin{align*}
     \fronorm{\left( \MM_t - \MM_{t,\ww} \right) \UUtw \UUtw^\top}
     \le 
     &\fronorm{\left(\MM_t - \MM_{t,\ww} \right) \VV_{\UUtw} } \specnorm{\UUtw \UUtw^\top}\\
     \le 
     &\fronorm{\left( \MM_t - \MM_{t,\ww} \right) \VV_{\UUtw} } 
     \left(\specnorm{\UUt \UUt^\top} + \specnorm{\UUt \UUtT - \UUtw \UUtw^\top} \right)\\
     \le 
     &3 \specnorm{\XXstar} \fronorm{ \left(\MM_t - \MM_{t,\ww} \right) \VV_{\UUtw}} ,
    \end{align*}
    where in the last inequality we used assumptions \eqref{ineq:weakboundintern1} 
    and \eqref{ineq:weakboundintern3}.
    \paragraph*{Estimating term (iii):}
    With the same argument as for term $(i)$ we observe that 
    \begin{align*}
     \fronorm{(\UUt \UUtT -\UUtw \UUtw^\top)\MM_t }
     \le 
     & 
     \left( \specnorm{\XXstar - \UUt \UUtT} 
     + \constone \sigma_{\min} \left(\XXstar \right) \right)
     \fronorm{\UUt \UUtT -\UUtw \UUtw^\top}.
    \end{align*}
 
    \paragraph*{Estimating term (iv):}
    With the same argument as for term $(ii)$ we compute that
    \begin{align*}
     \fronorm{ \UUtw \UUtw^\top\left( \MM_t - \MM_{t,\ww} \right)}
     \le 3 \specnorm{\XXstar} \fronorm{\left(\MM_t - \MM_{t,\ww}\right) \VV_{\UUtw}}.
    \end{align*}
   \paragraph*{Estimating term (v):}
   First, we compute that 
   \begin{align*}
     \MM_t \UUt \UUtT \MM_t - \MM_{t,\ww} \UUtw \UUtw^\top \MM_{t,\ww} 
     =
     &\MM_t \left( \UUt \UUtT - \UUtw \UUtw^\top \right) \MM_t  
     +
     \left( \MM_t - \MM_{t,\ww}  \right) \UUtw \UUtw^\top \MM_t\\
     &+
     \MM_{t,\ww} \UUtw \UUtw^\top \left( \MM_t - \MM_{t,\ww}  \right).
   \end{align*}
   It follows that 
   \begin{align*}
     &\fronorm{\MM_t \UUt \UUtT \MM_t - \MM_{t,\ww} \UUtw \UUtw^\top \MM_{t,\ww}}\\
     \le 
     &\specnorm{\MM_t}^2 \fronorm{\UUt \UUtT - \UUtw \UUtw^\top}
     +
     \left(\specnorm{\UUt}^2 + \specnorm{\UUt \UUtT -\UUtw \UUtw^\top} \right) \specnorm{\MM_t} \fronorm{\left(\MM_t- \MM_{t,\ww}\right) \VV_{\UUtw}}\\
     &+
     \specnorm{\MM_{t,\ww} \VV_{\UU_{t,\ww}}} \left( \specnorm{\UUt}^2 + \specnorm{\UUt \UUt^\top -\UUtw \UUtw^\top} \right) \fronorm{\left( \MM_t - \MM_{t,\ww} \right) \VV_{\UUtw} }\\
     \overleq{(a)}
     &\specnorm{\MM_t}^2 \fronorm{\UUt \UUtT - \UUtw \UUtw^\top}
     +
     3 \specnorm{\XXstar} \specnorm{\MM_t} \fronorm{\left(\MM_t- \MM_{t,\ww}\right) \VV_{\UUtw}}\\
     &+
     3 \specnorm{ \XXstar } \specnorm{\MM_{t,\ww} \VV_{ \UUtw }}  \fronorm{ \left(\MM_t - \MM_{t,\ww} \right) \VV_{\UUtw} }\\
     \overleq{(b)}
     &4 \sigma^2_{\min} (\XXstar) \fronorm{\UUt \UUtT - \UUtw \UUtw^\top}
     +
     15 \sigma_{\min} (\XXstar) \specnorm{\XXstar} \fronorm{\left(\MM_t - \MM_{t,\ww} \right) \VV_{\UUtw}}.
   \end{align*}
   For inequality $(a)$ we used the assumptions \eqref{ineq:weakboundintern1} and
   \eqref{ineq:weakboundintern4}.
   Inequality $(b)$ is a consequence of inequalities 
   \eqref{ineq:weakboundintern9} and \eqref{ineq:weakboundintern8}.

   \paragraph*{Conclusion:}
   By summing up all terms we obtain that 
   \begin{align*}
     &\fronorm{\UUtplus \UUtplus^\top - \UUtplusw \UUtplusw^\top}\\
     \le
     &\fronorm{\UUt \UUtT - \UUtw \UUtw^\top}
     +
     2 \mu \left( \specnorm{\XXstar - \UUt \UUtT} + \constone \sigma_{\min} (\XXstar) \right) 
      \fronorm{\UUt \UUtT - \UUtw \UUtw^\top}\\
     &+
     6\mu \specnorm{\XXstar} \fronorm{ \left(\MM_t - \MM_{t,\ww} \right) \VV_{\UUtw} }\\
     &+
     \mu^2 \left( 4  \sigma^2_{\min} (\XXstar)  \fronorm{\UUt \UUtT - \UUtw \UUtw^\top}
     +
     15  \sigma_{\min} (\XXstar) \specnorm{\XXstar} \fronorm{\left(\MM_t - \MM_{t,\ww} \right) \VV_{\UUtw}}  \right)\\
     \overleq{(a)}
     &\left(1+2\mu \specnorm{\XXstar - \UUt \UUtT} + 2c_1 \sigma_{\min} (\XXstar) \right)
     \fronorm{\UUt \UUtT - \UUtw \UUtw^\top}\\
     &+12 \mu \sigma_{\min} (\XXstar ) \constthree \specnorm{\XXstar - \UUt \UUtT}
     + 6 \mu \specnorm{\XXstar} \left( \frac{4 \constthree}{\kappa} +1 \right) \fronorm{\UUt \UUtT - \UUtw\UUtw^\top} \\
     &+
     4\mu^2  \sigma^2_{\min} (\XXstar)  \fronorm{\UUt \UUtT - \UUtw \UUtw^\top}
     +
     30c_3 \mu^2 \sigma_{\min}^2 (\XXstar) \specnorm{\XXstar - \UUt \UUtT}\\
     &+ 60 c_3 \mu^2  \sigma_{\min}^2 (\XXstar) \fronorm{\UUt \UUtT - \UUtw \UUtw^\top } 
     +15 \mu^2 \sigma_{\min} (\XXstar) \specnorm{\XXstar} \fronorm{\UUt \UUtT - \UUtw \UUtw^\top } 
     \\
     =
     & \left( 1 + 2 \mu \specnorm{\XXstar-\UUt \UUtT} 
     +\left( 2 \constone + 24 \constthree  \right) \mu \sigma_{\min} (\XXstar)+6\mu \specnorm{\XXstar}
     + 4\mu^2 \sigma_{\min}^2 (\XXstar)+ 60 \constthree \mu^2 \sigma^2_{\min} (\XXstar) \right)\\
     & \cdot \fronorm{\UUt \UUtT - \UUtw \UUtw^\top}+ 
     \left( 12 \constthree \mu \sigma_{\min} (\XXstar)+ 30 \constthree \mu^2 \sigma_{\min}^2 (\XXstar)  \right)
     \specnorm{\XXstar - \UUt \UUtT}\\
     \overleq{(b)} 
     &\frac{\sqrt{\sqrt{2}-1}}{40} \sigma_{\min} (\XXstar).
   \end{align*}
   Inequality $(a)$ follows from inequality \eqref{ineq:weakboundintern7}.
   Inequality $(b)$ is due to assumptions \eqref{ineq:weakboundintern3}, \eqref{ineq:weakboundintern4},
   and the assumption $\mu \le \frac{ \consttwo}{ \kappa \specnorm{\XXstar}}$
    for a sufficiently small absolute constant $\consttwo>0$.
   This completes the proof of Lemma \ref{lemma:auxdistanceweak}.
 \end{proof}

\subsection{Proof of Lemma \ref{lemma:auxcloseness1}}\label{subsec:proofauxcloseness1}

\begin{proof}[Proof of Lemma \ref{lemma:auxcloseness1}]
    Let $\RR \in \R^{r \times r}$ be an orthogonal matrix.
    We compute that 
    \begin{align*}
        \UUt \UUt^\top - \UUtw \UUtw^\top
        &=
        \UUt \RR \left( \UUt \RR \right)^\top - \UUtw \UUtw 
        =
        \UUt \RR \left(\UUt \RR - \UUtw\right)^\top - (\UUtw - \UUt \RR)\UUtw^\top.
    \end{align*}
    It follows that 
    \begin{align}
        &\fronorm{\VXXPT \left(\UUt \UUt^\top - \UUtw \UUtw^\top \right) \VXXP} \\
        \le 
        &\specnorm{ \VXXPT \UUt \RR  } \fronorm{ \UUt \RR - \UUtw  }
        +
        \fronorm{\UUtw - \UUt \RR} \specnorm{ \UUtw^\top \VXXP }
        \nonumber \\
        \le
        &\left( \specnorm{\VXXPT \UUt \RR} + \specnorm{\VXXPT \UUtw}  \right) \fronorm{\UUt \RR - \UUtw}\nonumber \\
        \le
        &\left( 2\specnorm{\VXXPT \UUt  } + \specnorm{\UUt \RR - \UUtw}  \right) \fronorm{\UUt \RR - \UUtw} \nonumber\\
        =& \left( 2\sqrt{\specnorm{\VXXPT \UUt\UUt^\top \VXXP  }} + \specnorm{\UUt \RR - \UUtw}  \right) \fronorm{\UUt \RR - \UUtw}\\
        =& \left( 2\sqrt{\specnorm{\VXXPT (\UUt\UUt^\top-\XXstar) \VXXP  }} + \specnorm{\UUt \RR - \UUtw}  \right) \fronorm{\UUt \RR - \UUtw}\\
        \overleq{(a)}
        &\left( \frac{1}{20} \sqrt{\sigma_{\min}(\XXstar)} +  \fronorm{\UUt \RR - \UUtw} \right)
        \fronorm{\UUt \RR - \UUtw}.
        \label{ineq:lemmaequ4} 
    \end{align}
    In inequality $(a)$ we used Assumption \eqref{assump:cloneness3}.
    By choosing the orthogonal matrix $\RR$ as the minimizer of Procruste's problem,
    i.e., such that $ \fronorm{\UUt \RR - \UUtw} $ is minimal,
    we obtain by Lemma \ref{lemma:procrustebound} that 
    \begin{align*}
        \fronorm{\UUt \RR - \UUtw}
        &\le 
        \frac{ \fronorm{\UUt \UUt^\top - \UUtw \UUtw^\top}}{\sqrt{2 \left(\sqrt{2}-1\right)  \sigma_{\min}^2 \left(\UUt\right)}}
        \overleq{(a)}
        \frac{\fronorm{\UUt \UUt^\top - \UUtw \UUtw^\top}}{ \sqrt{ \left( \sqrt{2}-1 \right)  \frac{3}{2}\sigma_{\min} \left( \XXstar \right) }  } 
        \overleq{(b)}
        \frac{ \sqrt{\sigma_{\min} (\XXstar)}}{ 20}.
    \end{align*}
    Inequality $(a)$ follows from Assumption \eqref{assump:cloneness3}  and Weyl's inequalities for singular values.
    For inequality $(b)$ we used Assumption \eqref{assump:cloneness4}.
    Inequality \eqref{ineq:lemmaequ4} combined with this inequality chain yields that  
    \begin{align}
        \fronorm{\VXXPT \left(\UUt \UUt^\top - \UUtw \UUtw^\top \right) \VXXP}
        \le  
        &\frac{\sqrt{\sigma_{\min}(\XXstar)}}{10}
        \cdot
        \frac{ \fronorm{\UUt \UUt^\top - \UUtw \UUtw^\top}}{\sqrt{ \left( \sqrt{2}-1\right) \cdot \frac{3}{2}  \sigma_{\min} \left( \XXstar \right)}}\nonumber \\
        \le
        &\frac{\fronorm{\UUt \UUt^\top - \UUtw \UUtw^\top}}{5}
        .\label{ineq:lemmaequ3} 
    \end{align}
    In order to proceed we note that
    \begin{align*}
        \fronorm{\UUt \UUt^\top - \UUtw \UUtw^\top}
        \le
        &\fronorm{\VXXT \left( \UUt \UUt^\top - \UUtw \UUtw^\top \right)}
        +
        \fronorm{ \VXXPT \left( \UUt \UUt^\top - \UUtw \UUtw^\top \right) \VXX}\\
        &+
        \fronorm{\VXXPT \left(\UUt \UUt^\top - \UUtw \UUtw^\top \right) \VXXP}\\
        \le
        &2\fronorm{\VXXT \left( \UUt \UUt^\top - \UUtw \UUtw^\top \right)}
        +\fronorm{ \VXXPT \left( \UUt \UUt^\top - \UUtw \UUtw^\top \right) \VXXP}\\
        \overleq{(a)}
        &2\fronorm{\VXXT \left( \UUt \UUt^\top - \UUtw \UUtw^\top \right)}
        + \frac{1}{5} \fronorm{  \UUt \UUt^\top - \UUtw \UUtw^\top }.
    \end{align*}
    In inequality $(a)$ we have used inequality \eqref{ineq:lemmaequ3}. 
    By rearranging terms we obtain that 
    \begin{align*}
        \fronorm{\UUt \UUt^\top - \UUtw \UUtw^\top}
        &\le 
        \frac{2}{1-\frac{1}{5}}
        \fronorm{ \VXXT \left(  \UUt \UUt^\top - \UUtw \UUtw^\top \right) }\\
        &\le 
        3  \fronorm{  \VXXT \left( \UUt \UUt^\top - \UUtw \UUtw^\top \right) }.
    \end{align*}
    This shows inequality \eqref{ineq:lemmaequ2}.
    Then  \eqref{ineq:lemmaequ1} follows directly from inserting the above inequality into \eqref{ineq:lemmaequ3}.
    \end{proof}

\subsection{Proof of Lemma \ref{lemma:auxsequencecloseness}}\label{subsec:proofauxsequencecloseness}

The key idea in the proof of Lemma \ref{lemma:auxsequencecloseness}
is to decompose 
$\VXXT \left(\UUtplus \UUtplus^\top - \UUtplusw \UUtplusw^\top \right)$
into a sum of the form
\begin{equation}\label{keydecomposition}
\begin{split}
    & \VXXT \left( \UUtplus \UUtplus^\top - \UUtplusw \UUtplusw^\top \right)\\
    =
    &\VXXT \left(
        1 + \mu \left( \XXstar - \UUt \UUtT - \UUtw \UUtw^\top \right)
    \right)
    \left( \UUt \UUt^\top - \UUtw \UUtw^\top \right)
    \left(
        1 + \mu \left( \XXstar - \UUt \UUtT - \UUtw \UUtw^\top \right)
    \right)\\
    &+ \VXXT \DDelta.
\end{split}
\end{equation}
The first summand can be interpreted as a contraction mapping applied to the matrix
$\UUt \UUtT - \UUtw \UUtw^\top$ 
and thus can be expected to have a smaller Frobenius norm than
$\Vert \VXXT \left(\UUt \UUtT - \UUtw \UUtw^\top\right)\Vert_F $.
In contrast, the term $\DDelta$,
which will be determined explicitly in the proof of Lemma \ref{lemma:auxsequencecloseness},
can be interpreted as an additive error term
which, as we will show, has relatively small Frobenius norm.

To deal with the first summand we need the following auxiliary lemma.
\begin{lemma}\label{lemma:aux2}
    Denote by $\lambda_{\max} (\AAf)$ the largest eigenvalue of a symmetric matrix $\AAf$ 
   and by $\lambda_{\min} (\AAf)$ the smallest eigenvalue of $\AAf$.
   Assume that the assumptions of Lemma \ref{lemma:auxsequencecloseness} are satisfied.
   Then it holds that
   \begin{align}
    \lambda_{\min}
    \left(
         \Id + \mu \left(\XXstar -\UUt \UUt^\top - \UUtw \UUtw^\top \right)
    \right)
    &\ge 0, \label{ineq:intern3}\\
    \lambda_{\max} \left(
        \VXXT \left( \XXstar -  \UUt \UUt^\top - \UUtw \UUtw^\top  \right)\VXX
        \right)
    &\le 
    - \frac{\sigma_{\min} \left( \XXstar \right) }{2},\label{ineq:interndistance3}\\
    \specnorm{
         \Id +\mu  (\XXstar -  \UUt \UUt^\top -  \UUtw \UUtw^\top ) 
    }
    &\le
    1+\frac{\mu \sigma_{\min} (\XXstar) }{128}.
    \label{ineq:interndistance1}
    \end{align}
\end{lemma}

\begin{proof}[Proof of Lemma \ref{lemma:aux2}]
    Note that the assumptions $ \mu \le \frac{ \constfour }{ \kappa \specnorm{\XXstar}} $, \eqref{assump:closeness6}, and \eqref{assump:closeness8} 
    together with Weyl's inequalities imply
    \begin{align*}
        &\lambda_{\min} \left(
         \Id + \mu \left( \XXstar -\UUt \UUt^\top - \UUtw \UUtw^\top \right)
    \right)\\
    =&
            \lambda_{\min} \left(
         \Id + \mu \left(
          (\XXstar -\UUt \UUt^\top) -\UUt \UUt^\top  + \UUt\UUt^\top- \UUtw \UUtw^\top  \right)
    \right)\\
    \geq & 
    1- \mu \specnorm{\XXstar - \UUt \UUtT}
    - \mu \specnorm{\UUt \UUtT}
    -\mu \specnorm{\UUt\UUt^\top- \UUtw \UUtw^\top }\\
    \ge & 
    0.
    \end{align*} for sufficiently small $c_2,c_3,c_4>0$.
    This shows inequality \eqref{ineq:intern3}.

    We observe that
    \begin{align}
        &\lambda_{\max} \left(
        \VXXT \left( \XXstar -  \UUt \UUt^\top - \UUtw \UUtw^\top  \right)\VXX
        \right)\nonumber\\
        \overleq{(a)}
        &\lambda_{\max} \left(- \VXXT \UUt \UUt^\top \VXX \right)
        +
        \specnorm{\XXstar -  \UUt \UUt^\top} 
        +
        \specnorm{\UUt \UUt^\top - \UUtw \UUtw^\top}
        \nonumber\\
        \overleq{(b)}
        &\lambda_{\max} \left( - \VXXT \UUt \UUt^\top \VXX \right)
        +
        \left( \consttwo + \constthree \right) \sigma_{\min} \left( \XXstar \right)\nonumber \\
        =
        &-\lambda_{\min} \left( \VXXT \VV_{\UUt}\VV_{\UUt}^\top \UUt \UUt^\top \VV_{\UUt}\VV_{\UUt}^\top\VXX \right)
        +
        \left(\consttwo + \constthree \right) \sigma_{\min} \left( \XXstar \right)\nonumber \\
        \le
        &-\sigma_{\min} \left( \VXXT \VV_{\UUt} \right)^2 \lambda_{\min} \left( \UUt \UUt^\top  \right)
        +
        \left(\consttwo + \constthree\right) \sigma_{\min} \left( \XXstar \right)\nonumber \\
        \overleq{(c)}
        &- \frac{ \sigma_{\min} \left(  \XXstar\right)}{2}.
    \end{align}
    Inequality $(a)$ follows from Weyl's inequalities.
    Inequality $(b)$ follows from assumption \eqref{assump:closeness7} and \eqref{assump:closeness8}.
    For inequality $(c)$ we used assumptions \eqref{assump:closeness5}, \eqref{assump:closeness7} for suffciently small $c_1,c_2,c_3$,
    and Weyl's inequalities.
    This proves inequality \eqref{ineq:interndistance3}.

    To prove inequality \eqref{ineq:interndistance1}, 
    we first establish an upper bound for the largest eigenvalue of 
    $  \XXstar -  \UUt \UUt^\top - \UUtw \UUtw^\top$.
    For that let $\xx \in \mathbb{R}^d$ be arbitrary.
    We use the orthogonal decomposition $\xx=\xx_{\parallel}+\xx_{\perp}$,
    where $\xx_{\parallel}$ is the orthogonal projection of $\xx$ onto the column span of $\XXstar$. 
    We compute that
    \begin{align}
        &\xx^\top \left( \XXstar -  \UUt \UUt^\top - \UUtw \UUtw^\top \right) \xx \nonumber \\
        =
        &\xx_{\parallel}^\top \left( \XXstar -  \UUt \UUt^\top - \UUtw \UUtw^\top \right) \xx_{\parallel}
        -
        \xx_{\bot}^\top \left(  \UUt \UUt^\top + \UUtw \UUtw^\top \right) \xx_{\bot}
        -
        2\xx_{\bot}^\top \left(  \UUt \UUt^\top + \UUtw \UUtw^\top \right) \xx_{\parallel} \nonumber \\
        \overleq{\eqref{ineq:interndistance3}}
        &- \frac{ \sigma_{\min} \left(  \XXstar\right)}{2}  \twonorm{\xx_{\parallel}}^2
        -
        2\xx_{\bot}^\top \left(  \UUt \UUt^\top + \UUtw \UUtw^\top \right) \xx_{\parallel}.
        \label{ineq:interndistance2}
    \end{align}
    Next, we observe that 
    \begin{align*}
        -\xx_{\bot}^\top \left(  \UUt \UUt^\top + \UUtw \UUtw^\top \right) \xx_{\parallel}
        &\le 
        \specnorm{ \VXXPT \left(  \UUt \UUt^\top + \UUtw \UUtw^\top \right) \VXX}
        \twonorm{\xx_{\parallel}} \twonorm{\xx_{\perp}}
        \\
        &\le 
        \left(
        2\specnorm{ \VXXPT  \UUt \UUt^\top  \VXX}
        +
        \specnorm{ \UUt \UUt^\top - \UUtw \UUtw^\top}
        \right)
        \twonorm{\xx_{\parallel}} \twonorm{\xx_{\perp}}
        \\
        &= 
        \left(
        2\specnorm{ \VXXPT   \left( \UUt \UUt^\top - \XXstar \right)  \VXX}
        +
        \specnorm{ \UUt \UUt^\top - \UUtw \UUtw^\top}
        \right)
        \twonorm{\xx_{\parallel}} \twonorm{\xx_{\perp}}
        \\
        &\le 
        \left(
        2\specnorm{ \XXstar - \UUt \UUtT}
        +
        \specnorm{ \UUt \UUt^\top - \UUtw \UUtw^\top }
        \right)
        \twonorm{\xx_{\parallel}} \twonorm{\xx_{\perp}}
        \\
        &\le \frac{\sigma_{\min} (\XXstar)\twonorm{\xx_{\parallel}} \twonorm{\xx_{\perp}}}{16}.
    \end{align*}
    In the last inequality we have used the assumptions \eqref{assump:closeness7} and \eqref{assump:closeness8} for sufficiently small $c_2,c_3 >0$.
    Combining this estimate with \eqref{ineq:interndistance2} we obtain that 
    \begin{align}
        \xx^\top \left( \XXstar -  \UUt \UUt^\top - \UUtw \UUtw^\top \right) \xx 
        &\le 
        \sigma_{\min} \left(\XXstar\right)
        \left(
        \frac{\twonorm{\xx_{\parallel}} \twonorm{\xx_{\perp}}}{8}- \frac{\twonorm{\xx_{\parallel}}^2}{2}
        \right)\nonumber\\
        &\le
        \frac{\sigma_{\min} \left(\XXstar\right) \twonorm{\xx_{\perp}}^2}
        {128}
        \le
        \frac{\sigma_{\min} (\XXstar)  \twonorm{\xx}^2 }{128}.
        \nonumber
    \end{align}
    This implies that
    \begin{equation}\label{ineq:intern235}
        \lambda_{\max} 
        \left(
         \Id + \mu \left(\XXstar -\UUt \UUt^\top - \UUtw \UUtw^\top \right)
        \right)
        \le 1 + \frac{\mu \sigma_{\min} (\XXstar)}{128}.
    \end{equation}
    This inequality, together with inequality  \eqref{ineq:intern3}, yields inequality \eqref{ineq:interndistance1}.
    Thus, the proof of Lemma \ref{lemma:aux2} is complete.
\end{proof}
With Lemma \ref{lemma:aux2} in place,
we can show that the first term in the decomposition \eqref{keydecomposition}
indeed has a smaller Frobenius norm than the term 
$\VXXT \left( \UUt \UUtT-\UUtw \UUtw^\top  \right) $.
\begin{lemma}\label{lemma:aux3}
Assume that the assumptions of Lemma \ref{lemma:auxsequencecloseness} are satisfied.
Then, it holds that
\begin{equation}
\begin{split}
        &\fronorm{
        \VXXT \left( \Id +\mu  (\XXstar -  \UUt \UUt^\top - \UUtw \UUtw^\top ) \right) 
        \left(\UUt \UUt^\top - \UUtw \UUtw^\top \right)
         \left( \Id +\mu  (\XXstar -  \UUt \UUt^\top -  \UUtw \UUtw^\top ) \right)
        } \nonumber \\
        &\le
        \left( 1- \frac{ \mu \sigma_{\min} (\XXstar)}{8} \right)
        \fronorm{
            \VXXT  \left(  \UUt \UUt^\top - \UUtw \UUtw^\top \right)
        }.
\end{split}    
\end{equation}
\end{lemma}
\begin{proof}[Proof of Lemma \ref{lemma:aux3}]
    We first compute that \small 
    \begin{align}
       &\fronorm{
        \VXXT \left( \Id +\mu  (\XXstar -  \UUt \UUt^\top - \UUtw \UUtw^\top ) \right) 
        \left(\UUt \UUt^\top - \UUtw \UUtw^\top \right)
         \left( \Id +\mu  (\XXstar -  \UUt \UUt^\top -  \UUtw \UUtw^\top ) \right)
        } \nonumber \\
        \le &
        \fronorm{
        \VXXT \left( \Id +\mu  (\XXstar -  \UUt \UUt^\top - \UUtw \UUtw^\top ) \right) 
        \left(\UUt \UUt^\top - \UUtw \UUtw^\top \right)}
        \specnorm{\Id +\mu  (\XXstar -  \UUt \UUt^\top -  \UUtw \UUtw^\top ) } \nonumber \\
       \le & 
        \left(
            1+\frac{\mu \sigma_{\min} (\XXstar) }{128}
        \right)
        \fronorm{
        \VXXT \left( \Id +\mu  (\XXstar -  \UUt \UUt^\top - \UUtw \UUtw^\top ) \right) 
        \left(\UUt \UUt^\top - \UUtw \UUtw^\top \right)},
        \label{ineq:intern2}
    \end{align}
    \normalsize
    where in the last line we used inequality \eqref{ineq:interndistance1} from Lemma \ref{lemma:aux2}.
    In order to proceed, we consider the decomposition
    \begin{align*}
        &\VXXT \left( \Id +\mu  (\XXstar -  \UUt \UUt^\top - \UUtw \UUtw^\top ) \right) 
        \left(\UUt \UUt^\top - \UUtw \UUtw^\top \right) \\
        &\bracing{=:\NN_1}{\VXXT \left( \Id +\mu  (\XXstar -  \UUt \UUt^\top - \UUtw \UUtw^\top ) \right) 
        \VXX \VXXT\left(\UUt \UUt^\top - \UUtw \UUtw^\top \right)}\\
        & - \mu \bracing{=:\NN_2}{ \VXXT   ( \UUt \UUt^\top + \UUtw \UUtw^\top )  
        \VXXP \VXXPT\left(\UUt \UUt^\top - \UUtw \UUtw^\top \right) \VXX \VXXT}\\
        & - \mu  \bracing{=:\NN_3}{\VXXT   ( \UUt \UUt^\top + \UUtw \UUtw^\top )  
        \VXXP \VXXPT\left(\UUt \UUt^\top - \UUtw \UUtw^\top \right)\VXXP \VXXPT}.
    \end{align*}
    We estimate the Frobenius norm of the three terms individually.
    For the first term we obtain that 
    \begin{align*}
        \fronorm{\NN_1}
        &\le 
        \specnorm{\VXXT \left( \Id +\mu  (\XXstar -  \UUt \UUt^\top - \UUtw \UUtw^\top ) \right) \VXX}
        \fronorm{\VXXT\left(\UUt \UUt^\top - \UUtw \UUtw^\top \right)}\\
        &= 
        \specnorm{ \Id +\mu  \VXXT (\XXstar -  \UUt \UUt^\top - \UUtw \UUtw^\top ) \VXX}
        \fronorm{\VXXT\left(\UUt \UUt^\top - \UUtw \UUtw^\top \right)}\\
        &\overleq{(a)} 
        \left( 1 +\mu  \lambda_{\max} \left( \VXXT (\XXstar -  \UUt \UUt^\top - \UUtw \UUtw^\top ) \VXX \right) \right)
        \fronorm{\VXXT\left(\UUt \UUt^\top - \UUtw \UUtw^\top \right)}\\
        &\overleq{(b)}
        \left( 1 - \frac{\mu \sigma_{\min} (\XXstar)}{2} \right)
        \fronorm{\VXXT\left(\UUt \UUt^\top - \UUtw \UUtw^\top \right)},
    \end{align*}
    where in inequality $(a)$ we have used \eqref{ineq:intern3} and in $(b)$ we have used inequality \eqref{ineq:interndistance3}
    from Lemma \ref{lemma:aux2}.
    The Frobenius norm of the term $\NN_2$ can be estimated by 
    \begin{align*}
        &\fronorm{\NN_2}\\
        &\le 
        \specnorm{\VXXT   ( \UUt \UUt^\top + \UUtw \UUtw^\top )  \VXXP}
        \fronorm{ \VXXPT\left(\UUt \UUt^\top - \UUtw \UUtw^\top \right)\VXX}\\
        &=
        \left( \specnorm{\VXXPT \left[ 2 \left( \UUt \UUtT -\XXstar  \right) 
        + \left( \UUtw \UUtw^\top - \UUt \UUtT \right) \right] } \right)
        \fronorm{ \VXXT \left(\UUt \UUt^\top - \UUtw \UUtw^\top \right)}\\
        &\le
        \left(2 \specnorm{\VXXPT (\UUt \UUt^\top-\XXstar)}+ \specnorm{  \UUt \UUt^\top - \UUtw \UUtw^\top} \right)
        \fronorm{ \VXXT \left(\UUt \UUt^\top - \UUtw \UUtw^\top \right)}\\
        &\le
        \left( 2 \consttwo \sigma_{\min}(\XXstar)+ \fronorm{  \UUt \UUt^\top - \UUtw \UUtw^\top} \right)
        \fronorm{ \VXXT \left(\UUt \UUt^\top - \UUtw \UUtw^\top \right)}\\
        &\le
        \left(  2 \consttwo + \constthree \right) \sigma_{\min} (\XXstar) \fronorm{ \VXXT \left(\UUt \UUt^\top - \UUtw \UUtw^\top \right)},
    \end{align*}
    where we have used Assumptions \eqref{assump:closeness7}
    and \eqref{assump:closeness8}.
    With similar arguments, we can estimate the Frobenius norm of the term $\NN_3$ by
    \begin{align*}
        \fronorm{\NN_3}
        &\le \left( 2 \consttwo + \constthree \right) \sigma_{\min} (\XXstar)
        \fronorm{ \VXXPT \left(\UUt \UUt^\top - \UUtw \UUtw^\top \right) \VXXP}.
    \end{align*}
    By using Lemma \ref{lemma:auxcloseness1} we obtain that 
    \begin{equation*}
        \fronorm{ \VXXPT \left(\UUt \UUt^\top - \UUtw \UUtw^\top \right) \VXXP}
        \le 
        \frac{3
        \fronorm{
         \VXXT  \left(  \UUt \UUt^\top - \UUtw \UUtw^\top \right)
        }}
        {5}.
    \end{equation*}
    It follows that 
    \begin{equation*}
    \fronorm{\NN_3}
    \le  \frac{3 \left(  2 \consttwo + \constthree \right) \sigma_{\min} (\XXstar)
    \fronorm{ \VXXT \left(\UUt \UUt^\top - \UUtw \UUtw^\top \right) }}{5}.
    \end{equation*}
    By summing up our estimates for $\fronorm{\NN_1}$, $\fronorm{\NN_2}$, and $\fronorm{\NN_3}$
    and choosing the constants $ \constone ,  \consttwo >0$ small enough 
    we obtain that 
    \begin{align*}
        &\fronorm{
        \VXXT \left( \Id +\mu  (\XXstar -  \UUt \UUt^\top - \UUtw \UUtw^\top ) \right) 
        \left(\UUt \UUt^\top - \UUtw \UUtw^\top \right)  
        }\\
        \le
        &\left( 
            1- \frac{ \mu \sigma_{\min} (\XXstar)}{4}
        \right)
        \fronorm{
            \VXXT  \left(  \UUt \UUt^\top - \UUtw \UUtw^\top \right)
        }.
    \end{align*}
    Inserting this estimate into \eqref{ineq:intern2} yields that
    \small 
    \begin{align*}
        &\fronorm{
        \VXXT \left( \Id +\mu  (\XXstar -  \UUt \UUt^\top - \UUtw \UUtw^\top ) \right) 
        \left(\UUt \UUt^\top - \UUtw \UUtw^\top \right)
         \left( \Id +\mu  (\XXstar -  \UUt \UUt^\top -  \UUtw \UUtw^\top ) \right)
        }\\
        \le
        &\left(
            1+\frac{\mu \sigma_{\min} (\XXstar) }{128}
        \right)
        \left( 
            1- \frac{ \mu \sigma_{\min} (\XXstar)}{4}
        \right)
        \fronorm{
            \VXXT  \left(  \UUt \UUt^\top - \UUtw \UUtw^\top \right)
        }\\
        \le 
        &
        \left( 1- \frac{ \mu \sigma_{\min} (\XXstar)}{8} \right)
        \fronorm{
            \VXXT  \left(  \UUt \UUt^\top - \UUtw \UUtw^\top \right)
        },
    \end{align*}
    \normalsize
    where in the last line, we used our assumption on the step size $\mu$.
    This completes the proof of Lemma \ref{lemma:aux3}.
\end{proof}
With the auxiliary estimates in Lemma \ref{lemma:aux3}
we can give a proof of Lemma \ref{lemma:auxsequencecloseness}.
\begin{proof}[Proof of Lemma \ref{lemma:auxsequencecloseness}]
    First, we compute that
    \begin{align*}
        \UUtplus \UUtplus^\top
        =
        &\left( \Id +\mu \left[ \left( \Aops \right) \left(\XXstar- \UUt \UUt^\top \right) \right]  \right)  \UUt \UUt^\top   \left( \Id +\mu \left[ \left( \Aops \right) \left(\XXstar- \UUt \UUt^\top \right) \right]  \right) \\
        =
        &\left( \Id +\mu  (\XXstar -  \UUt \UUt^\top - \UUtw \UUtw^\top ) \right) \UUt \UUt^\top \left( \Id +\mu  (\XXstar -  \UUt \UUt^\top -  \UUtw \UUtw^\top ) \right)\\
        &+\mu \UUtw \UUtw^\top  \UUt \UUt^\top + \mu \UUt \UUt^\top  \UUtw \UUtw^\top\\
        &+\mu^2 \UUtw \UUtw^\top  \UUt \UUt^\top \left( \XXstar - \UUt \UUt^\top \right)
        +\mu^2\left( \XXstar - \UUt \UUt^\top \right) \UUt \UUt^\top  \UUtw \UUtw^\top \\
        &- \mu^2  \UUtw \UUtw^\top  \UUt \UUt^\top  \UUtw \UUtw^\top\\ 
        &+\mu \left[ \left(  \Aops - \IdOp \right) \left(\XXstar- \UUt \UUt^\top \right) \right] \UUt \UUt^\top \left( \Id + \mu \XXstar - \mu \UUt \UUt^\top \right)\\
        &+\mu \left( \Id + \mu \XXstar - \mu \UUt \UUt^\top \right) \UUt \UUt^\top \left[ \left(  \Aops - \IdOp \right) \left(\XXstar- \UUt \UUt^\top \right) \right]\\
        &+ \mu^2 \left[ \left(  \Aops - \IdOp \right) \left(\XXstar- \UUt \UUt^\top \right) \right] \UUt \UUt^\top  \left[ \left(  \Aops - \IdOp \right) \left(\XXstar- \UUt \UUt^\top \right) \right].
    \end{align*}
    Analogously, we can compute that 
    \begin{align*}
        &\UUtplusw \UUtplusw^\top\\
        =
        &\left( \Id +\mu  (\XXstar -  \UUt \UUt^\top - \UUtw \UUtw^\top ) \right) \UUtw \UUtw^\top \left( \Id +\mu  (\XXstar -  \UUt \UUt^\top -  \UUtw \UUtw^\top ) \right)\\
        &+\mu \UUtw \UUtw^\top  \UUt \UUt^\top + \mu \UUt \UUt^\top  \UUtw \UUtw^\top\\
        &+\mu^2 \UUt \UUt^\top  \UUtw \UUtw^\top \left( \XXstar - \UUtw \UUtw^\top \right)
        +\mu^2\left( \XXstar - \UUtw \UUtw^\top \right) \UUtw \UUtw^\top  \UUt \UUt^\top \\
        &- \mu^2  \UUt \UUt^\top  \UUtw \UUtw^\top  \UUt \UUt^\top\\ 
        &+\mu \left[ \left(  \Aopws - \IdOp \right) \left(\XXstar- \UUtw \UUtw^\top \right) \right] \UUtw \UUtw^\top \left( \Id + \mu \XXstar - \mu \UUtw \UUtw^\top \right)\\
        &+\mu \left( \Id + \mu \XXstar - \mu \UUtw \UUtw^\top \right) \UUtw \UUtw^\top \left[ \left(  \Aopws - \IdOp \right) \left(\XXstar- \UUtw \UUtw^\top \right) \right]\\
        &+ \mu^2 \left[ \left(  \Aopws - \IdOp \right) \left(\XXstar- \UUtw \UUtw^\top \right) \right] \UUtw \UUtw^\top  \left[ \left(  \Aopws - \IdOp \right) \left(\XXstar- \UUtw \UUtw^\top \right) \right].
    \end{align*}
    Thus, we obtain that
    \begin{align}
        &\UUtplus \UUtplus^\top-\UUtplusw \UUtplusw^\top
        \\
        = &
        \MM_1+\mu^2 \MM_2 + \mu^2 \MM_3 +\mu^2 \MM_4 + \mu^2 \MM_4 + \mu \MM_5 + \mu \MM_6 + \mu^2 \MM_7, \label{ineq:intern42}
    \end{align}
    where
    \begin{align*}
        \MM_1:=& \left( \Id +\mu  (\XXstar -  \UUt \UUt^\top - \UUtw \UUtw^\top ) \right) \left(\UUt \UUt^\top - \UUtw \UUtw^\top \right) \left( \Id +\mu  (\XXstar -  \UUt \UUt^\top -  \UUtw \UUtw^\top ) \right)\\
        \MM_2
        :=& 
         \UUtw \UUtw^\top  \UUt \UUt^\top \left( \XXstar - \UUt \UUt^\top \right)
        - \UUt \UUt^\top  \UUtw \UUtw^\top \left( \XXstar - \UUtw \UUtw^\top \right),\\
        \MM_3:=& 
        \left( \XXstar - \UUt \UUt^\top \right) \UUt \UUt^\top  \UUtw \UUtw^\top 
        -\left( \XXstar - \UUtw \UUtw^\top \right) \UUtw \UUtw^\top  \UUt \UUt^\top,\\
        \MM_4:=& 
         \UUt \UUt^\top  \UUtw \UUtw^\top  \UUt \UUt^\top
        - \UUtw \UUtw^\top  \UUt \UUt^\top  \UUtw \UUtw^\top,\\
        \MM_5:=& 
        \left[ \left(  \Aops - \IdOp \right) \left(\XXstar- \UUt \UUt^\top \right) \right] \UUt \UUt^\top \left( \Id + \mu \XXstar - \mu \UUt \UUt^\top \right)\\
        &-\left[ \left(  \Aopws - \IdOp \right) \left(\XXstar- \UUtw \UUtw^\top \right) \right] \UUtw \UUtw^\top \left( \Id + \mu \XXstar - \mu \UUtw \UUtw^\top \right),\\
        \MM_6:=& 
         \left( \Id + \mu \XXstar - \mu \UUt \UUt^\top \right) \UUt \UUt^\top \left[ \left(  \Aops - \IdOp \right) \left(\XXstar- \UUt \UUt^\top \right) \right]\\
        &- \left( \Id + \mu \XXstar - \mu \UUtw \UUtw^\top \right) \UUtw \UUtw^\top \left[ \left(  \Aopws - \IdOp \right) \left(\XXstar- \UUtw \UUtw^\top \right) \right],\\
        \MM_7:=& 
        \left[ \left(  \Aops - \IdOp \right) \left(\XXstar- \UUt \UUt^\top \right) \right] \UUt \UUt^\top  \left[ \left(  \Aops - \IdOp \right) \left(\XXstar- \UUt \UUt^\top \right) \right]\\
        &-  \left[ \left(  \Aopws - \IdOp \right) \left(\XXstar- \UUtw \UUtw^\top \right) \right] \UUtw \UUtw^\top  \left[ \left(  \Aopws - \IdOp \right) \left(\XXstar- \UUtw \UUtw^\top \right) \right].
    \end{align*}
    Recall that Lemma \ref{lemma:aux3} shows that
    \begin{equation*}
        \fronorm{\VXXT \MM_1}
        \le
        \left( 1- \frac{ \mu \sigma_{\min} (\XXstar)}{8} \right)
        \fronorm{
            \VXXT  \left(  \UUt \UUt^\top - \UUtw \UUtw^\top \right)
        }.
    \end{equation*}
    To complete the proof, we need to derive upper bounds for $ \fronorm{ \MM_i} $, where $i=2,3,\ldots, 7$.\\

    \noindent\textbf{Estimating $ \fronorm{\MM_2} $:}
    We compute that
    \begin{align*}
        \MM_2
        =& 
         \UUtw \UUtw^\top  \UUt \UUt^\top \left( \XXstar - \UUt \UUt^\top \right)
        - \UUt \UUt^\top  \UUtw \UUtw^\top \left( \XXstar - \UUtw \UUtw^\top \right)\\
         =&
         \left( \UUtw \UUtw^\top - \UUt \UUt^\top \right)  \UUt \UUt^\top \left( \XXstar - \UUt \UUt^\top \right)
         +\UUt \UUt^\top  \left( \UUt \UUt^\top - \UUtw \UUtw^\top \right) \left( \XXstar - \UUt \UUt^\top \right)\\
         &+\UUt \UUt^\top  \UUtw \UUtw^\top \left( \UUtw \UUtw^\top - \UUt \UUt^\top \right).
    \end{align*}
    Thus, we obtain that 
    \begin{align*}
       &\fronorm{\MM_2} \\
       \le
        &2\fronorm{ \UUtw \UUtw^\top - \UUt \UUt^\top  } \specnorm{ \UUt \UUt^\top} \specnorm{ \XXstar - \UUt \UUt^\top }
         + \specnorm{ \UUt \UUt^\top} \specnorm{  \UUtw \UUtw^\top } \fronorm{ \UUtw \UUtw^\top - \UUt \UUt^\top }\\
        \le
        &2\fronorm{ \UUtw \UUtw^\top - \UUt \UUt^\top  } \specnorm{ \UUt \UUt^\top} \specnorm{ \XXstar - \UUt \UUt^\top }\\
        &+ \specnorm{ \UUt \UUt^\top} \left( \specnorm{  \UUt \UUt^\top } + \specnorm{\UUt \UUt^\top - \UUtw \UUtw^\top } \right) \fronorm{ \UUtw \UUtw^\top - \UUt \UUt^\top }\\
        \le 
        &5  \specnorm{\XXstar}^2 \fronorm{ \UUtw \UUtw^\top - \UUt \UUt^\top  }.
    \end{align*}
    In the last inequality we used assumptions \eqref{assump:closeness6}, \eqref{assump:closeness7},
    and \eqref{assump:closeness8} for sufficiently small $\consttwo, \constthree >0$.\\
    
    \noindent\textbf{Estimating $ \fronorm{\MM_3} $:} 
    Since $\MM_3=\MM_2^\top$ it follows that
    \begin{align*}
       \fronorm{\MM_3} 
       \le  
       5 \norm{\XXstar}^2 \fronorm{\UUtw \UUtw^\top - \UUt \UUt^\top}.
    \end{align*}
    \noindent\textbf{Estimating $ \fronorm{\MM_4} $:}
    We compute that
    \begin{align*}
        \MM_4
        =
        &\left( \UUt \UUt^\top -\UUtw \UUtw^\top \right)  \UUtw \UUtw^\top  \UUt \UUt^\top
        + \UUtw \UUtw^\top   \left( \UUtw \UUtw^\top - \UUt \UUt^\top \right) \UUt \UUt^\top\\
        + &\UUtw \UUtw^\top  \UUt \UUt^\top \left(  \UUt \UUt^\top- \ \UUtw \UUtw^\top \right).
    \end{align*}
    Again, using the assumptions \eqref{assump:closeness6} and \eqref{assump:closeness8},
    and the triangle inequality we obtain that 
    \begin{align*}
       \fronorm{\MM_4} 
       \le 
       20 \specnorm{\XXstar}^2 \fronorm{\UUt \UUt^\top- \ \UUtw \UUtw^\top}.
    \end{align*}
    
    \noindent\textbf{Estimating $ \fronorm{\MM_5} $:}
    We compute
    \begin{align}
       \MM_5
        =
        &\bracing{=: \OO_1}{\left[ \left(  \Aops - \IdOp \right) \left(\XXstar- \UUt \UUt^\top \right) \right] \left( \UUt \UUt^\top - \UUtw \UUtw^\top \right) \left( \Id + \mu \XXstar - \mu \UUt \UUt^\top \right)}\nonumber\\
        &+\mu \bracing{=: \OO_2}{ \left[ \left(  \Aops - \IdOp \right) \left(\XXstar- \UUt \UUt^\top \right) \right] \UUtw \UUtw^\top \left( \UUtw \UUtw^\top -\UUt \UUt^\top \right)}\nonumber\\
        &+ \bracing{=: \OO_3}{\left[ \left(  \Aops - \Aopws \right) \left(\XXstar- \UUt \UUt^\top \right) \right] \UUtw \UUtw^\top \left( \Id + \mu \XXstar - \mu \UUtw \UUtw^\top \right)} \nonumber\\
        &+ \bracing{=: \OO_4}{\left[ \left(  \Aopws - \IdOp \right) \left( \UUtw \UUtw^\top - \UUt \UUt^\top \right) \right] \UUtw \UUtw^\top \left( \Id + \mu \XXstar - \mu \UUtw \UUtw^\top \right)}. 
        \label{equ:intern1}
    \end{align}
    We estimate the Frobenius norm of these summands individually.
    For the first term we observe that
    \begin{align*}
        \fronorm{\OO_1}
        \le
        & \specnorm{  \left(  \Aops - \IdOp \right) \left(\XXstar- \UUt \UUt^\top \right)  } 
        \fronorm{  \UUt \UUt^\top - \UUtw \UUtw^\top} \left( 1 + \mu \specnorm{\XXstar} + \mu \specnorm{\UUtw \UUtw^\top} \right)\\
        \overleq{(a)}
        & 2\specnorm{  \left(  \Aops - \IdOp \right) \left(\XXstar- \UUt \UUt^\top \right)  } 
        \fronorm{  \UUt \UUt^\top - \UUtw \UUtw^\top} \\
        \overleq{(b)}
        &2 \constfive
        \sigma_{\min} (\XXstar)
        \fronorm{  \UUt \UUt^\top - \UUtw \UUtw^\top} ,
    \end{align*} 
    where in inequality $(a)$ we have used assumptions \eqref{assump:closeness6}, \eqref{assump:closeness8}, 
    and the assumption on the step size $\mu$. 
    In inequality $(b)$ we have used assumption \eqref{assump:closenessRIP}.
    
    Using again assumptions \eqref{assump:closeness6}, \eqref{assump:closeness8}, and \eqref{assump:closenessRIP} we obtain that 
    \begin{align*}
        \fronorm{\OO_2}
         \le & 3 \constfive \sigma_{\min} (\XXstar) \specnorm{\XXstar} \fronorm{\UUtw \UUtw^\top -\UUt \UUt^\top}.
    \end{align*}
    For the term $\fronorm{\OO_3}$ we obtain that
    \begin{align}
        \fronorm{\OO_3}
        \le 
        &\fronorm{\left[ \left(  \Aops - \Aopws \right) \left(\XXstar- \UUt \UUt^\top \right) \right]\VV_{\UUtw}}
        \specnorm{ \UUtw \UUtw^\top }
        \left( 1+\mu \specnorm{\XXstar} + \mu \specnorm{\UUtw \UUtw^\top} \right)\nonumber\\
        \le 
        &\fronorm{\left[ \left(  \Aops - \Aopws \right) \left(\XXstar- \UUt \UUt^\top \right) \right] \VV_{\UUtw}}
        \left( \specnorm{ \UUt \UUt^\top - \UUtw \UUtw^\top } + \specnorm{ \UUt \UUt^\top } \right)\nonumber\\
        &\left( 1+\mu \specnorm{\XXstar} + \mu \specnorm{\UUtw \UUtw^\top - \UUt \UUt^\top} + \mu \specnorm{\UUt \UUt^\top} \right)\nonumber\\
        \overleq{(a)}
        &
        4 \fronorm{\left[ \left(  \Aops - \Aopws \right) \left(\XXstar- \UUt \UUt^\top \right) \right] \VV_{\UUtw}}
        \specnorm{\XXstar} \nonumber \\
        \overleq{(b)} 
        & 
        4\left( \delta + \frac{8 \sqrt{rd} }{\sqrt{m}} \right) \specnorm{\XXstar -\UUt \UUt^\top}
        \specnorm{\XXstar}  +4 \left( \delta + \frac{8 \sqrt{2d} }{ \sqrt{m} } \right) \fronorm{\UUt \UUt^\top - \UUtw \UUtw^\top}
        \specnorm{\XXstar}
       . \nonumber
    \end{align}
    Inequality $(a)$ follows from the assumptions \eqref{assump:closeness6} and \eqref{assump:closeness8},
    and the assumption on the step size $\mu$.
    In inequality $(b)$ we used the estimate \eqref{ineq:intern4} from Lemma \ref{lemma:auxestimates}.
    
    For the term $\fronorm{\OO_4}$ we obtain that 
    \begin{align*}
      \fronorm{\OO_4}
      \le 
      &\fronorm{\left(  \Aopws - \IdOp \right) \left( \UUtw \UUtw^\top - \UUt \UUt^\top \right) \VV_{\UUtw}}
      \left(\specnorm{\UUtw \UUtw^\top - \UUt \UUt^\top} + \specnorm{\UUt \UUt^\top} \right)\\
      &\cdot\left(1 + \mu \specnorm{\XXstar - \UUt \UUt^\top} 
       + \mu \specnorm{\UUt \UUt^\top - \UUtw \UUtw^\top} \right)\\
      \overleq{(a)}
      &3 \specnorm{\XXstar} \fronorm{\left[ \left(  \Aopws - \IdOp \right) \left( \UUtw \UUtw^\top - \UUt \UUt^\top \right) \right] \VV_{\UUtw} } \nonumber \\
      \overleq{(b)}
      &6 \delta\specnorm{\XXstar} \fronorm{  \UUtw \UUtw^\top - \UUt \UUt^\top  }.
    \end{align*}
    Inequality $(a)$ follows from assumptions \eqref{assump:closeness7} and \eqref{assump:closeness8},
    and the assumption on the step size $\mu$.
    Inequality $(b)$ is due to inequality \eqref{ineq:intern5} in Lemma \ref{lemma:auxestimates}.
    By summing up all terms we obtain that 
    \begin{align*}
        &\fronorm{\MM_5} \nonumber 
        \le
        \fronorm{\OO_1}+
        \mu \fronorm{\OO_2}+
        \fronorm{\OO_3}+
        \fronorm{\OO_4}\\
        \le
        &2 \constfive \sigma_{\min} (\XXstar) \fronorm{  \UUt \UUt^\top - \UUtw \UUtw^\top}
        + 3\mu \constfive \sigma_{\min} (\XXstar)
        \specnorm{\XXstar} \fronorm{\UUtw \UUtw^\top -\UUt \UUt^\top}\\
        &+4 \left( \delta + \frac{8 \sqrt{rd} }{\sqrt{m}} \right) 
        \specnorm{\XXstar} \specnorm{\XXstar -\UUt\UUtT}
        +4 \left( \delta + \frac{4 \sqrt{2d} }{\sqrt{m}} \right) 
        \specnorm{\XXstar} \fronorm{\UUt \UUt^\top -\UUtw \UUtw^\top}\\
        &+ 6 \delta \specnorm{\XXstar} \fronorm{  \UUtw \UUtw^\top - \UUt \UUt^\top  }\\
        =
        & \left[ \left( \left( 2 + 3 \mu \right) \constfive + 6 \kappa \delta \right) \sigma_{\min} (\XXstar) 
        +4 \left( \delta + \frac{4 \sqrt{2d}}{\sqrt{m}} \right) \specnorm{\XXstar}   \right] \fronorm{  \UUt \UUt^\top - \UUtw \UUtw^\top}\\
        &+4\left(\delta+ \frac{8 \sqrt{rd}}{\sqrt{m}}  \right)
        \specnorm{\XXstar} \specnorm{\XXstar -\UUt\UUtT}\\
        \overleq{(a)}
        & \left( \left( \left( 2 + 3 \mu \right) \constfive + 6\constsix \right) \sigma_{\min} (\XXstar) 
        + 8 \constsix \sigma_{\min} (\XXstar)    \right) \fronorm{  \UUt \UUt^\top - \UUtw \UUtw^\top}
        +8 \constsix \sigma_{\min} (\XXstar) \specnorm{\XXstar -\UUt\UUtT}\\
        \overleq{(b)}
        & \frac{  \sigma_{\min} (\XXstar)}{100} \cdot \fronorm{  \UUt \UUt^\top - \UUtw \UUtw^\top}
        +8 \constsix \sigma_{\min} (\XXstar) \specnorm{\XXstar -\UUt\UUtT}\\
        \overleq{(c)}
        & \frac{ 3 \sigma_{\min} (\XXstar)}{100} \cdot \fronorm{ \VXXT \left( \UUt \UUt^\top - \UUtw \UUtw^\top \right)}
        +8 \constsix \sigma_{\min} (\XXstar) \specnorm{\XXstar -\UUt\UUtT},
    \end{align*}
    where in inequality $(a)$ we used the assumption \eqref{assump:closenessDelta}.
    Inequality $(b)$ follows from choosing the constants $\constfive$ and $\constsix$ small enough. 
    To obtain inequality $(c)$ we applied Lemma \ref{lemma:auxcloseness1}.
    \\

    \noindent\textbf{Estimating $\fronorm{\MM_6 }$:}\\
    Since $\MM_6 = \MM_5^\top $ we obtain that 
    \begin{align*}
        \fronorm{\MM_6}
        \le 
        \frac{ 3 \sigma_{\min} (\XXstar)}{100} \cdot \fronorm{ \VXXT \left( \UUt \UUt^\top - \UUtw \UUtw^\top \right)}
        +8 \constsix \sigma_{\min} (\XXstar) \specnorm{\XXstar -\UUt\UUtT}.
    \end{align*}
    
    \noindent\textbf{Estimating $ \fronorm{ \MM_7} $:}
    To deal with the term $\MM_7$ we first compute that 
    \begin{align*}
        \MM_7=
        &\bracing{=:\LL_1}{\left[ \left(  \Aops - \IdOp \right) \left(\XXstar- \UUt \UUt^\top \right) \right] \left( \UUt \UUt^\top - \UUtw \UUtw^\top \right)  \left[ \left(  \Aops - \IdOp \right) \left(\XXstar- \UUt \UUt^\top \right) \right]} \\
        &+\bracing{=:\LL_2}{\left[ \left(  \Aops - \IdOp \right) \left(\UUtw \UUtw^\top- \UUt \UUt^\top \right) \right] \UUtw \UUtw^\top  \left[ \left(  \Aops - \IdOp \right) \left(\XXstar- \UUt \UUt^\top \right) \right]}\\
        &+\bracing{=:\LL_3}{\left[ \left(  \Aops - \IdOp \right) \left(\XXstar- \UUtw \UUtw^\top \right) \right] \UUtw \UUtw^\top  \left[ \left(  \Aops - \IdOp \right) \left(\UUtw \UUtw^\top- \UUt \UUt^\top \right) \right]}\\
        &+\bracing{=:\LL_4}{\left[ \left(  \Aops - \Aopws \right) \left(\XXstar- \UUtw \UUtw^\top \right) \right] \UUtw \UUtw^\top  \left[ \left(  \Aops - \IdOp \right) \left(\XXstar- \UUtw \UUtw^\top \right) \right]}\\
        &+\bracing{=:\LL_5}{\left[ \left(  \Aopws - \IdOp \right) \left(\XXstar- \UUtw \UUtw^\top \right) \right] \UUtw \UUtw^\top  \left[ \left(  \Aops - \Aopws \right) \left(\XXstar- \UUtw \UUtw^\top \right) \right]}.
    \end{align*}
    We estimate the Frobenius norm of the summands individually.
    For $\fronorm{\LL_1}$ we obtain that
    \begin{align*}
    \fronorm{\LL_1}
    &\le 
    \specnorm{\left(  \Aops - \IdOp \right) \left(\XXstar- \UUt \UUt^\top \right)}
    \fronorm{\UUtw \UUtw^\top- \UUt \UUt^\top}
    \specnorm{\left(  \Aops - \IdOp \right) \left(\XXstar- \UUt \UUt^\top \right)}\\
    &\le 
    \constfive^2 \sigma_{\min} (\XXstar)^2 
    \fronorm{\UUtw \UUtw^\top- \UUt \UUt^\top},
    \end{align*}
    where we have used assumption \eqref{assump:closenessRIP}.
    Next, we note that
    \begin{align*}
        \fronorm{\LL_2}
        \le
        &\fronorm{\left(  \Aops - \IdOp \right) \left(\UUtw \UUtw^\top- \UUt \UUt^\top  \right)\VV_{\UUtw}}
        \left( \specnorm{\UUt \UUt^\top} + \specnorm{ \UUt \UUt^\top -\UUtw \UUtw^\top} \right)\\
        &\cdot 
        \specnorm{\left(  \Aops - \IdOp \right) \left(\XXstar- \UUt \UUt^\top \right)}\\
        \overleq{(a)}
        &3\constfive \sigma_{\min} \left( \XXstar \right) \specnorm{\XXstar}
        \fronorm{\left(  \Aops - \IdOp \right) \left(\UUtw \UUtw^\top- \UUt \UUt^\top  \right)\VV_{\UUtw}} \nonumber
        \\
        \overleq{(b)}
        &3 \constfive \delta \sigma_{\min} (\XXstar) \specnorm{\XXstar}
        \fronorm{\UUtw \UUtw^\top- \UUt \UUt^\top}\\
        \overleq{(c)}
        & 3 \constfive \constsix \sigma_{\min}^2 (\XXstar) 
        \fronorm{\UUtw \UUtw^\top- \UUt \UUt^\top}.
    \end{align*}
    Inequality $(a)$ follows from 
    assumptions \eqref{assump:closeness6}, \eqref{assump:closeness8}, and \eqref{assump:closenessRIP}.
    Inequality $(b)$ is due to Lemma \ref{lemma: RIP}
    and inequality $(c)$ is due to assumption \eqref{assump:closenessDelta}.
    In order to estimate $\fronorm{\LL_3}$ we note that
    \begin{align*}
        \fronorm{\LL_3}
        &\left(\specnorm{\left(  \Aops - \IdOp \right) \left(\XXstar- \UUt \UUtT \right)} 
        +\fronorm{\left[\left(  \Aops - \IdOp \right) \left(\UUt \UUt^\top- \UUtw \UUtw^\top \right) \right] \VV_{\UUtw} } \right)\\
        &\cdot \left( \specnorm{\UUt \UUt^\top} + \fronorm{\UUt \UUt^\top - \UUtw \UUtw^\top} \right)
        \fronorm{  \left[ \left(  \Aops - \IdOp \right) \left(\UUtw \UUtw^\top- \UUt \UUt^\top \right)\right] \VV_{\UUtw} }\\
        \overleq{(a)}
        &\left( \constfive \sigma_{\min} (\XXstar) + \delta \fronorm{\UUt \UUt^\top - \UUtw \UUtw^\top} \right)
        \left(  2\specnorm{\XXstar} + \constthree \sigma_{\min} (\XXstar) \right)
        \delta \fronorm{\UUtw \UUtw^\top- \UUt \UUt^\top }\\
        \overleq{(b)}
        & 3 \left( \constfive + \delta \constthree  \right) \delta \sigma_{\min} \left( \XXstar \right) \specnorm{\XXstar} \fronorm{\UUtw \UUtw^\top- \UUt \UUt^\top }\\
        \overleq{(c)}
        & 3 \constsix  \left( \constfive + \delta \constthree  \right)  \sigma_{\min}^2 \left( \XXstar \right) \fronorm{\UUtw \UUtw^\top- \UUt \UUt^\top }.
    \end{align*}
    In inequality $(a)$ we used the assumptions \eqref{assump:closeness6}, \eqref{assump:closeness8}, \eqref{assump:closenessRIP},
    and Lemma \ref{lemma: RIP}.
    Inequality $(b)$ follows from assumption \eqref{assump:closeness8} and since the constant $\constthree>0$ is chosen small enough.
    Inequality $(c)$ is due to assumption \eqref{assump:closenessDelta}.
    
    Next, we can estimate $\fronorm{\LL_4}$ by
    \begin{align*}
        \fronorm{\LL_4}
        \le 
        &\fronorm{\left[ \left(  \Aops - \Aopws \right) \left(\XXstar- \UUtw \UUtw^\top \right)  \right]\VV_{\UUtw}} 
        \left(\specnorm{\UUt \UUt^\top}+\specnorm{\UUt \UUt^\top - \UUtw \UUtw^\top} \right)\\
        & \cdot \left( \specnorm{  \left(  \Aops - \IdOp \right) \left(\XXstar- \UUt \UUt^\top \right) }
        +  \specnorm{  \left(  \Aops - \IdOp \right) \left(\UUtw \UUtw^\top - \UUt \UUt^\top \right) }\right)\\ 
        \overleq{(a)}
        &\fronorm{\left[ \left(  \Aops - \Aopws \right) \left(\XXstar- \UUtw \UUtw^\top \right)  \right]\VV_{\UUtw}} 
        \left(2\specnorm{\XXstar}+ \constthree  \sigma_{\min} (\XXstar) \right)\\
        & \cdot \left( \constfive \sigma_{\min} (\XXstar) +  \delta \fronorm{\UUtw \UUtw^\top - \UUt \UUtT}  \right)\\ 
        \overleq{(b)} 
        &3 \left( \constfive + \constthree \delta \right)
        \sigma_{\min} (\XXstar) \specnorm{\XXstar} 
        \fronorm{\left[ \left(  \Aops - \Aopws \right) \left(\XXstar- \UUtw \UUtw^\top \right)  \right]\VV_{\UUtw}}  
        \nonumber \\
        \overleq{(c)} 
        & 3 \left( \constfive + \constthree \delta \right) 
        \left( \delta  + 8 \sqrt{\frac{rd}{m}} \right)
        \sigma_{\min} (\XXstar) \specnorm{\XXstar}
        \specnorm{\XXstar - \UUtw \UUtw^\top}
        \nonumber\\
        \le
        & 3 \left( \constfive + \constthree \delta \right) 
        \left( \delta  + 8 \sqrt{\frac{rd}{m}} \right)
        \sigma_{\min} (\XXstar) \specnorm{\XXstar}
         \left(\specnorm{\XXstar - \UUt \UUtT} 
        + \specnorm{\UUtw \UUtw^\top - \UUt \UUtT}\right)
        \nonumber \\
        \overleq{(d)}
        & 6 \constsix \left( \constfive + c_3 \delta \right)
        \sigma_{\min}^2 \left(\XXstar \right) 
        \left(\specnorm{\XXstar - \UUt \UUtT} + \specnorm{\UUtw \UUtw^\top - \UUt \UUtT}\right).
    \end{align*}
    In inequality $(a)$ we used assumptions \eqref{assump:closeness6}, \eqref{assump:closeness8}, and
    \eqref{assump:closenessRIP}
    as well as Lemma \ref{lemma: RIP}.
    Inequality $(b)$ uses assumption \eqref{assump:closeness8}.
    Inequality $(c)$ follows from inequality \eqref{ineq:intern7} in Lemma \ref{lemma:auxestimates}.
    Inequality $(d)$ is due to assumption \eqref{assump:closenessDelta}.
    
    The norm $\fronorm{\LL_5}$ can be estimated by
    \begin{align}
        \fronorm{\LL_5} 
        \le 
        &\specnorm{\left(  \Aopws - \IdOp \right) \left(\XXstar- \UUtw \UUtw^\top \right)}
        \specnorm{\UUtw \UUtw^\top}
        \fronorm{ \VV_{\UUtw}^\top \left[ \left(  \Aops -  \Aopws \right) \left(  \XXstar- \UUtw \UUtw^\top  \right) \right] }
        \nonumber \\
        \overleq{(a)}
        &3 \specnorm{\XXstar}
        \specnorm{\left(  \Aopws - \IdOp \right) \left(\XXstar- \UUtw \UUtw^\top \right) }
        \fronorm{  \left[ \left(  \Aops -  \Aopws \right) \left(  \XXstar- \UUtw \UUtw^\top  \right) \right] \VV_{\UUtw}}.
        \label{ineq:intern41}
    \end{align}
    In inequality $(a)$ we used the triangle inequality 
    and the assumptions \eqref{assump:closeness6}, \eqref{assump:closeness8}.
    In order to proceed, we note first that
    \begin{align*}
        &\specnorm{\left( \Aopws - \IdOp \right) \left( \XXstar - \UUtw \UUtw^\top \right) }\\
        \overleq{(a)}
        &\specnorm{ \left( \Aops - \IdOp \right) \left(\XXstar - \UUt \UUt^\top \right) }
        +\left(  \delta+ 8 \sqrt{\frac{rd}{m}} \right)\specnorm{\XXstar - \UUt \UUt^\top }\\
        &+\left(  2\delta + 4 \sqrt{\frac{2d}{m}}  \right)\fronorm{\UUtw \UUtw^\top -\UUt \UUt^\top}\\
        \overleq{(b)}
        &\specnorm{ \left( \Aops - \IdOp \right) \left(\XXstar - \UUt \UUt^\top \right) }
        +
        \frac{2\constsix}{\kappa} \specnorm{\XXstar - \UUt \UUtT}
        +
        \frac{3 \constsix}{\kappa} \fronorm{ \UUtw \UUtw^\top - \UUt \UUtT}\\
        \overleq{(c)} 
        & \left( \constfive + \frac{2 \consttwo \constsix }{\kappa} + \frac{3 \constthree \constsix}{\kappa} \right) \sigma_{\min} (\XXstar),
    \end{align*}
    where in inequality $(a)$ we used Lemma \ref{lemma:auxestimates}.
    Inequality $(b)$ follows from the assumptions \eqref{assump:closenessDelta}.
    Inequality $(c)$ is due to assumption \eqref{assump:closeness7}, \eqref{assump:closeness8}, 
    and \eqref{assump:closenessRIP}.
    Moreover, it holds that
    \begin{align*}
        \fronorm{  \left[ \left(  \Aops -  \Aopws \right) \left(  \XXstar- \UUtw \UUtw^\top  \right) \right] \VV_{\UUtw}}
        \overleq{(a)}
        & \left(\delta + 8 \sqrt{\frac{rd}{m}} \right)
        \specnorm{\XXstar - \UUtw \UUtw^\top}\\
        \overleq{(b)}
        & \frac{2\constsix}{\kappa}
        \left( \specnorm{\XXstar - \UUt \UUtT} 
        + \specnorm{ \UUtw \UUtw^\top - \UUt \UUtT }  \right).
    \end{align*}
    Inequality $(a)$ follows from inequality \eqref{ineq:intern7} in Lemma \ref{lemma:auxestimates}.
    Inequality $(b)$ is due to assumption \eqref{assump:closenessDelta}. 
    Inserting the last two inequality chains into inequality \eqref{ineq:intern41} we obtain that
    \begin{align*}
    \fronorm{\LL_5} 
    \le 
    6 \constsix
    \left( \constfive + \frac{2 \consttwo \constsix }{\kappa} + \frac{3 \constthree \constsix}{\kappa} \right) 
    \sigma_{\min}^2 (\XXstar) 
    \left( \specnorm{\XXstar - \UUt \UUtT} 
    + \specnorm{ \UUtw \UUtw^\top - \UUt \UUtT }  \right)
    \end{align*}

    By summing up all terms $\fronorm{\LL_i}$ for $i=1,\ldots,5$ it follows that 
    \begin{align*}
        \fronorm{\MM_7}
        \le
        &\constfive^2 \sigma^2_{\min} (\XXstar) \fronorm{\UUtw \UUtw^\top- \UUt \UUt^\top}\\
        &+
        3 \constfive \constsix \sigma_{\min}^2 (\XXstar) 
        \fronorm{\UUtw \UUtw^\top- \UUt \UUt^\top}\\
        &+
        3 \constsix  \left( \constfive + \constthree \delta   \right)  \sigma_{\min}^2 \left( \XXstar \right) \fronorm{\UUtw \UUtw^\top- \UUt \UUt^\top }\\
        &+
        6 \constsix \left( \constfive + \constthree \delta \right)
        \sigma_{\min}^2 \left(\XXstar \right) 
        \left(\specnorm{\XXstar - \UUt \UUtT} + \specnorm{\UUtw \UUtw^\top - \UUt \UUtT}\right)\\
        &+
        6 \constsix
        \left( \constfive + \frac{2 \consttwo \constsix }{\kappa} + \frac{3 \constthree \constsix}{\kappa} \right) 
        \sigma_{\min}^2 (\XXstar) 
        \left( \specnorm{\XXstar - \UUt \UUtT} 
        + \specnorm{ \UUtw \UUtw^\top - \UUt \UUtT }  \right)\\
        \le
        &\sigma_{\min}^2 (\XXstar) 
        \left( 
            \specnorm{\XXstar - \UUt \UUtT}
            +
            \fronorm{\UUtw \UUtw^\top - \UUt \UUtT}
         \right),
    \end{align*}
    where the last inequality holds since the absolute constants $\constthree, \constfive, \constsix > 0$
    are chosen small enough.

    Using the decomposition \eqref{ineq:intern42}, the triangle inequality,
    combined with our estimates for $\fronorm{\VXXT \MM_1}$ and for $\fronorm{\MM_i}$, 
    where $ 2 \le i \le 7 $,
    we obtain that
    \begin{align*}
        &\fronorm{\VXXT \left( \UUtplus \UUtplus^\top - \UUtplusw \UUtplusw^\top \right)}\\
        \le 
        &\left( 1- \frac{\mu \sigma_{\min} (\XXstar)}{8} \right)
        \fronorm{ \VXXT \left( \UUt \UUt^\top - \UUtw \UUtw^\top \right) }
        +30 \mu^2 \specnorm{\XXstar}^2 \fronorm{\UUt \UUt^\top - \UUtw \UUtw^\top}\\
        &+\frac{ 3 \mu \sigma_{\min} (\XXstar)}{50} \cdot \fronorm{ \VXXT \left( \UUt \UUt^\top - \UUtw \UUtw^\top \right)}
        +16 \mu \constsix \sigma_{\min} (\XXstar) \specnorm{\XXstar -\UUt\UUtT}
        \\
        &+\mu^2\sigma_{\min}^2 (\XXstar) \left( \specnorm{\XXstar - \UUt \UUtT}+\fronorm{\UUtw \UUtw^\top - \UUt \UUtT}\right)\\
        \overleq{(a)} 
        &\left( 1- \frac{\mu \sigma_{\min} (\XXstar)}{8} \right)
        \fronorm{ \VXXT \left( \UUt \UUt^\top - \UUtw \UUtw^\top \right) }
        +90  \mu \constfour \sigma_{\min} (\XXstar) \fronorm{\VXXT \left(\UUt \UUt^\top - \UUtw \UUtw^\top \right)}\\
        &+\frac{ 3 \mu  \sigma_{\min} (\XXstar)}{50} \cdot \fronorm{ \VXXT \left( \UUt \UUt^\top - \UUtw \UUtw^\top \right)}
        + 16 \mu \constsix \sigma_{\min} (\XXstar) \specnorm{\XXstar -\UUt\UUtT}\\
        &+\mu^2\sigma_{\min}^2 (\XXstar) \specnorm{\XXstar - \UUt \UUtT}
        + \frac{3\mu \constfour \sigma_{\min} \left( \XXstar \right)}{\kappa}  \fronorm{\VXXT \left(\UUtw \UUtw^\top - \UUt \UUtT \right)}
        \\
        \overleq{(b)}
        &\left( 1- \frac{\mu \sigma_{\min} (\XXstar)}{16} \right)
        \fronorm{ \VXXT \left( \UUt \UUt^\top - \UUtw \UUtw^\top \right) }
        +
        \mu \left( 16 \constsix   + \mu \sigma_{\min} (\XXstar) \right) \sigma_{\min}(\XXstar)\specnorm{\XXstar - \UUt \UUtT}\\
        \le
        &\left( 1- \frac{\mu \sigma_{\min} (\XXstar)}{16} \right)
        \fronorm{ \VXXT \left( \UUt \UUt^\top - \UUtw \UUtw^\top \right) }
        +
        \mu  \sigma_{\min}(\XXstar)\specnorm{\XXstar - \UUt \UUtT},
    \end{align*}
    where inequality $(a)$ is due to Lemma \ref{lemma:auxcloseness1} and the assumption on the step size $\mu$.
    Inequality $(b)$ is obtained by choosing  $\constfour < 1/2$, 
    and the last inequality is obtained by choosing $c_6<\frac{1}{32}$.
    \end{proof}
\section{Proof of the lemmas controlling the distance between $\XXstar$ and $\UUt \UUtT$
(Lemma \ref{lemma:localconvaux}, Lemma \ref{lemma:localconv}, and Lemma \ref{lemma:convaprioribound})}

\subsection{Proof of Lemma \ref{lemma:localconvaux}}\label{sec:prooflocalconv_a}

\begin{proof}[Proof of Lemma \ref{lemma:localconvaux}]
    We first note that 
    \begin{align*}
        \VXXPT \UUt \UUtT \VXXP
        =
        &\VXXPT \VUUt \VUUtP \UUt \UUtT \VXXP\\
        =
        &\VXXPT \VUUt\left(  \VXXT \VUUt \right)^{-1} \VXXT \VUUt  \VUUtP \UUt \UUtT \VXXP\\
        =
        &\VXXPT \VUUt\left(  \VXXT \VUUt \right)^{-1} \VXXT  \UUt \UUtT \VXXP\\
        =
        &\VXXPT \VUUt\left(  \VXXT \VUUt \right)^{-1} \VXXT  \left( \UUt \UUtT - \XXstar \right) \VXXP.
    \end{align*} 
    Using the submultiplicativity property of the $\triplenorm{\cdot}$-norm it follows that
    \begin{align*}
       \triplenorm{\VXXPT \UUt \UUtT \VXXP}
       \le 
       &\specnorm{\VXXPT \VUUt}  \specnorm{ \left(  \VXXT \VUUt \right)^{-1}}
       \triplenorm{ \VXXT  \left( \UUt \UUtT - \XXstar \right) \VXXP}\\
       =
       &\frac{\specnorm{\VXXPT \VUUt}}{\sigma_{\min} (\VXXT \VUUt) }
        \triplenorm{ \VXXT  \left( \UUt \UUtT - \XXstar \right) \VXXP}.
    \end{align*}
    Recall that
    \begin{align*}
       \sigma^2_{\min} \left( \VXXT \VUUt \right)
       &=
       1- \specnorm{ \VUUtP \VXXP \VXXPT \VUUtP }
       =
       1- \specnorm{\VUUtP \VXXP }^2 \ge \frac{1}{4},
    \end{align*} 
    where in the last inequality, we used assumption \eqref{ineq:localconv1}.
    It follows that 
    \begin{equation*}
       \triplenorm{\VXXPT \UUt \UUtT \VXXP}
       \le 
       2 \specnorm{\VXXPT \VUUt} 
       \triplenorm{ \VXXT  \left( \UUt \UUtT - \XXstar \right) \VXXP}.
    \end{equation*} 
    This proves inequality \eqref{ineq:localconv2}.
    To prove inequality \eqref{ineq:localconv3} we note that
    \begin{align*}
        \triplenorm{\UUt \UUtT - \XXstar }
        \le 
        &\triplenorm{ \VXXT \left( \UUt \UUtT - \XXstar \right) }
        +
        \triplenorm{ \VXXT \left( \UUt \UUtT - \XXstar \right) \VXXP }\\
        &+
        \triplenorm{ \VXXPT \left( \UUt \UUtT - \XXstar \right) \VXXP }\\
        \le
        & 2\triplenorm{ \VXXT \left( \UUt \UUtT - \XXstar \right) }
        +\triplenorm{ \VXXPT \UUt \UUtT  \VXXP }\\
        \le
        & 2 \left( 1+ \specnorm{ \VXXPT \VUUt } \right)
        \triplenorm{ \VXXT \left( \UUt \UUtT - \XXstar \right) },
    \end{align*}
    where in the last inequality we used \eqref{ineq:localconv2}. 
    This completes the proof of Lemma \ref{lemma:localconvaux}.
    \end{proof}

\subsection{Proof of Lemma \ref{lemma:localconv}}\label{sec:prooflocalconv_b}

    \begin{proof}[Proof of Lemma \ref{lemma:localconv}]
        We define the shorthand notation
        \begin{equation*}
            \MM_t 
            :=
            \left(\Aops\right) \left( \XXstar - \UUt \UUtT \right)
            =
            \XXstar - \UUt \UUtT 
            + \bracing{=: \EEb_t}{\left( \Aops - \IdOp \right) \left(\XXstar - \UUt \UUtT\right)}.
        \end{equation*}
        Thus, we have that 
        \begin{equation*}
            \UUtplus 
            = \left(\Id + \mu \MM_t \right) \UUt.
        \end{equation*}
        We compute that 
        \begin{align*}
            &\XXstar - \UUtplus \UUtplus^\top\\
            =
            &\XXstar - \UUt \UUtT 
            - \mu \MM_t \UUt \UUtT 
            - \mu \UUt \UUtT \MM_t 
            - \mu^2 \MM_t \UUt \UUtT \MM_t\\
            =
            &\XXstar - \UUt \UUtT 
            - \mu \left( \XXstar - \UUt \UUtT \right) \UUt \UUtT
            - \mu  \UUt \UUtT \left( \XXstar - \UUt \UUtT \right)
            - \mu \EEb_t \UUt \UUtT- \UUt \UUtT \EEb_t\\
            &-\mu^2\MM_t \UUt \UUtT \MM_t\\
            =
            & \left(\Id - \mu \UUt \UUtT \right) \left(\XXstar - \UUt \UUtT\right) \left( \Id -\mu \UUt \UUtT \right)
            -\mu^2 \UUt \UUtT \left( \XXstar - \UUt \UUtT \right) \UUt \UUtT\\
            &- \mu \EEb_t \UUt \UUtT - \mu \UUt \UUtT \EEb_t -\mu^2 \MM_t \UUt \UUtT \MM_t.
        \end{align*}
        It follows that
        \begin{align*}
            &\VXXT \left( \XXstar - \UUtplus \UUtplus^\top \right)\\
            =
            &\VXXT \left(\Id - \mu \UUt \UUtT \right) \VXX \VXXT \left(\XXstar - \UUt \UUtT\right) \left( \Id -\mu \UUt \UUtT \right)\\
            & + \VXXT\left(\Id - \mu \UUt \UUtT \right) \VXXP \VXXPT \left(\XXstar - \UUt \UUtT\right) \left( \Id -\mu \UUt \UUtT \right)\\
            &-\mu^2 \VXXT \UUt \UUtT \left( \XXstar - \UUt \UUtT \right) \UUt \UUtT
            - \mu \VXXT\EEb_t \UUt \UUtT - \mu \VXXT\UUt \UUtT \EEb_t -\mu^2 \VXXT\MM_t \UUt \UUtT \MM_t\\
            =
            & \bracing{=:(I)}{ \left(\Id - \mu \VXXT\UUt \UUtT \VXX\right)  \VXXT \left(\XXstar - \UUt \UUtT\right) \left( \Id -\mu \UUt \UUtT \right)}\\
            & + \mu \bracing{=:(II)}{ \VXXT \UUt \UUtT \VXXP \VXXPT  \UUt \UUtT \left( \Id -\mu \UUt \UUtT \right)}\\
            &-\bracing{=:(III)}{\left(\mu^2 \VXXT \UUt \UUtT \left( \XXstar - \UUt \UUtT \right) \UUt \UUtT
            + \mu \VXXT\EEb_t \UUt \UUtT + \mu \VXXT\UUt \UUtT \EEb_t + \mu^2 \VXXT\MM_t \UUt \UUtT \MM_t \right)}.
        \end{align*}
        We estimate the spectral norm of these terms individually.
        
        \paragraph*{Estimating term $(I)$:}
        We obtain that 
        \begin{align*}
            &\triplenorm{
            \left(\Id - \mu \VXXT\UUt \UUtT \VXX\right)  \VXXT \left(\XXstar - \UUt \UUtT\right) \left( \Id -\mu \UUt \UUtT \right)
            }\\
            \overleq{(a)}
            &\specnorm{\Id - \mu \VXXT\UUt \UUtT \VXX} 
            \triplenorm{\VXXT \left(\XXstar - \UUt \UUtT\right)}
            \specnorm{\Id -\mu \UUt \UUtT}\\
            \overleq{(b)}
            &\specnorm{\Id - \mu \VXXT\UUt \UUtT \VXX} 
            \triplenorm{\VXXT \left(\XXstar - \UUt \UUtT\right)}\\
            \overeq{(c)}
            &\left( 1 -\mu  \sigma^2_{\min} (\VXXT \UUt) \right)
            \triplenorm{\VXXT \left(\XXstar - \UUt \UUtT\right)}\\
            \le
            &\left( 1 -\mu \left( \sigma_{\min} (\VXXT \VUUt) \sigma_{\min} \left(\UUt \right) \right)^2 \right)
            \triplenorm{\VXXT \left(\XXstar - \UUt \UUtT\right)}\\
            \overleq{(d)}
            &\left( 1 - \frac{\mu}{2}  \sigma_{\min}^2 \left( \UUt \right)  \right)
            \triplenorm{\VXXT \left(\XXstar - \UUt \UUtT\right)}\\
            \overleq{(e)}
            &\left( 1 - \frac{\mu}{4}  \sigma_{\min} \left( \XXstar  \right) \right)
            \triplenorm{\VXXT \left(\XXstar - \UUt \UUtT\right)}.
        \end{align*}
        Inequality $(a)$ is due to the submultiplicativity of the $\triplenorm{\cdot}$-norm.
        In inequality $(b)$ and equality $(c)$ we used the assumptions $\specnorm{\UUt} \le \sqrt{2 \specnorm{\XXstar}}$
        and $\mu \le \frac{1}{1024 \kappa \specnorm{\XXstar}}$.
        In inequality $(d)$ we used assumption \eqref{ineq:localconv8}.
        Inequality $(e)$ follows from assumption \eqref{ineq:localconv5},
        which, due to Weyl's inequality, implies 
        $\sigma^2_{\min} \left(\UUt\right)  \ge \frac{1}{2} \sigma_{\min} (\XXstar)$.

        \paragraph*{Estimating term $(II)$:}
        We note that 
        \begin{align*}
            &\triplenorm{\VXXT \UUt \UUtT \VXXP \VXXPT  \UUt \UUtT \left( \Id -\mu \UUt \UUtT \right)}\\
            =
            &\triplenorm{\VXXT \left(\UUt \UUtT  -\XXstar \right) \VXXP \VXXPT \left(  \UUt \UUtT - \XXstar \right) \left( \Id -\mu \UUt \UUtT \right)}\\
            \overleq{(a)}
            &\specnorm{\VXXPT \left( \UUt \UUtT - \XXstar  \right) } 
            \triplenorm{\VXXPT \left(\UUt \UUtT -\XXstar \right)}
            \specnorm{\Id -\mu \UUt \UUtT}\\
            \overleq{(b)}
            & \specnorm{\VXXPT \left( \UUt \UUtT - \XXstar \right)} 
            \triplenorm{\VXXPT \left( \UUt \UUtT -\XXstar \right)}\\
            \le
            &\specnorm{\VXXPT \left( \UUt \UUtT - \XXstar  \right) } 
            \left( \specnorm{\VXXPT \left( \UUt \UUtT -\XXstar \right) \VXX} 
             + \triplenorm{\VXXPT \left( \UUt \UUtT -\XXstar \right) \VXXP}\right)\\
            \le
            &\specnorm{\VXXPT \left( \UUt \UUtT - \XXstar  \right) } 
            \left( \triplenorm{\VXXT \left( \UUt \UUtT -\XXstar \right)} 
             + \triplenorm{\VXXPT \left( \UUt \UUtT -\XXstar \right) \VXXP}\right)\\
            \overleq{(c)}
            &2\specnorm{\VXXPT \left( \UUt \UUtT - \XXstar  \right) } 
            \left( 1+  \specnorm{ \VXXPT \VUUt }\right)
             \triplenorm{\VXXT \left( \UUt \UUtT -\XXstar \right)}\\
            \overleq{(d)}
            &3\specnorm{ \UUt \UUtT - \XXstar   } 
             \triplenorm{\VXXT \left( \UUt \UUtT -\XXstar \right)}.
        \end{align*}
        In inequality $(a)$ we used the submultiplicativity
        of the $\triplenorm{\cdot}$-norm.
        Inequality $(b)$ follows from the assumption  $\specnorm{\UUt} \le \sqrt{2 \specnorm{\XXstar}}$
        and $\mu \le \frac{1}{1024 \kappa \specnorm{\XXstar}}$.
        In inequality $(c)$, we used Lemma \ref{lemma:localconvaux}.
        In inequality $(d)$ we used the assumption $ \specnorm{\VXXPT \VUUt} \le \frac{1}{2} $.
        Thus, by using the assumption $\specnorm{ \XXstar - \UUt \UUtT} \le \frac{\sigma_{\min} (\XXstar) }{48}$ it follows that 
        \begin{equation*}
            \triplenorm{(II)}
            \le 
            \frac{\sigma_{\min} \left(\XXstar \right) }{16}
            \triplenorm{\VXXT \left( \UUt \UUtT - \XXstar \right)}.
        \end{equation*}
        
        \paragraph*{Estimating term $(III)$:}
        We first note that 
        \begin{align}
        \triplenorm{\MM_t \VUUt}
         \overleq{(a)}
        &\triplenorm{\XXstar -\UUt \UUtT} 
        + \triplenorm{ \left[ \left( \Aops - \IdOp \right) \left( \XXstar - \UUt \UUtT \right) \right] \VUUt } 
        \nonumber \\
        \overleq{(b)} 
        &4 \triplenorm{ \VXXT \left( \XXstar - \UUt \UUtT \right)} 
        + \triplenorm{\EEb_t \VUUt},
        \label{ineq:localconv7}
        \end{align}
        where $(a)$ follows from the triangle inequality and
        $(b)$ follows from Lemma \ref{lemma:localconvaux}.
        Moreover, we have that
        \begin{align}
            \specnorm{\MM_t}
            \le
            \specnorm{\XXstar - \UUt \UUtT}
            +
            \specnorm{ \left(\Aops - \IdOp \right) 
            \left( \XXstar - \UUt \UUtT  \right)} 
            \overleq{(a)}
            \sigma_{\min} \left( \XXstar \right).
            \label{ineq:localconv9}
        \end{align}
        Inequality $(a)$ follows from assumptions \eqref{ineq:localconv5} and \eqref{ineq:localconv4}.
        Thus, we obtain for term $(III)$ that
        \begin{align*}
            \triplenorm{(III)}
            \le 
            &\mu^2 \specnorm{\UUt}^4 \triplenorm{\XXstar - \UUt \UUtT}
            + 2 \mu \specnorm{\UUt}^2 \triplenorm{\EEb_t \VUUt }
            +\mu^2 \specnorm{\UUt}^2 \triplenorm{\MM_t \VUUt} \specnorm{\MM_t}\\
            \overleq{(a)}
            & 16 \mu^2 \specnorm{\XXstar}^2 \triplenorm{\VXXT \left(\XXstar - \UUt \UUtT \right)}
            +4\mu \specnorm{\XXstar} \triplenorm{\EEb_t \VUUt}
            +2 \mu^2 \sigma_{\min} \left( \XXstar \right) \specnorm{\XXstar} \triplenorm{\MM_t \VUUt}  \\
            \overleq{(b)}
            & \left( 16 \mu^2 \specnorm{\XXstar}^2 + 8\mu^2 \sigma_{\min} (\XXstar) \specnorm{\XXstar} \right)
            \triplenorm{\VXXT \left(\XXstar - \UUt \UUtT \right)}\\
            &+\left( 4 \mu \specnorm{\XXstar } + 2\mu^2 \sigma_{\min} (\XXstar) \specnorm{\XXstar}  \right) \triplenorm{\EEb_t \VUUt}\\
            \overleq{(c)} 
            & \frac{ \mu  \sigma_{\min} \left( \XXstar \right)}{16} 
            \triplenorm{\VXXT \left( \XXstar -\UUt \UUtT \right)}
            +
            5\mu \specnorm{\XXstar } \triplenorm{\EEb_t \VUUt}.
        \end{align*}
        In inequality $(a)$ we used the assumption $\specnorm{\UUt} \le \sqrt{2 \specnorm{\XXstar}}$,
        Lemma \ref{lemma:localconvaux}, and inequality \eqref{ineq:localconv9}.
        Inequality $(b)$ is due to inequalities \eqref{ineq:localconv7}.
        In inequality $(c)$ we used the assumption that 
        $\mu \le \frac{1}{1024 \kappa \specnorm{\XXstar} }$.
        
        \paragraph*{Conclusion:}
        By adding up all terms, we obtain that
        \begin{align*}
            \triplenorm{
                \VXXT \left( \XXstar - \UUtplus \UUtplus^\top \right)
            }
            \le
            &\triplenorm{(I)}+\mu \triplenorm{(II)}+\triplenorm{(III)}\\
            \le
            &\left( 1 - \frac{\mu \sigma_{\min} (\XXstar) }{8} \right)
            \triplenorm{\VXXT \left( \XXstar - \UUt \UUtT \right)}
            +
            5 \mu \specnorm{\XXstar} \triplenorm{\EEb_t \VUUt}
            .
        \end{align*}
        This completes the proof.
        \end{proof}

\subsection{Proof of Lemma \ref{lemma:convaprioribound}}
\label{sec:convaprioribound}

\begin{proof}[Proof of Lemma \ref{lemma:convaprioribound}]
  Analogously, as in the proof of Lemma \ref{lemma:localconv}
  we define the shorthand notation
          \begin{equation*}
            \MM_t 
            :=
            \left(\Aops\right) \left( \XXstar - \UUt \UUtT \right)
            =
            \XXstar - \UUt \UUtT 
            + \bracing{=: \EEb_t}{\left( \Aops - \IdOp \right) \left(\XXstar - \UUt \UUtT\right)}.
        \end{equation*}
We note that 
\begin{align*}
    \specnorm{\MM_t}
    \le 
    \specnorm{ \XXstar - \UUt \UUtT  }
    +
    \specnorm{\left( \Aops - \IdOp \right) \left(\XXstar - \UUt \UUtT\right)}
    \le 
    (c_2+c_3) \sigma_{\min} (\XXstar). 
\end{align*}
With an analogous computation as in the proof of Lemma \ref{lemma:localconv}, it follows that
        \begin{align*}
            \XXstar - \UUtplus \UUtplus^\top
            =
            & \left(\Id - \mu \UUt \UUtT \right) \left(\XXstar - \UUt \UUtT\right) \left( \Id -\mu \UUt \UUtT \right)
            -\mu^2 \UUt \UUtT \left( \XXstar - \UUt \UUtT \right) \UUt \UUtT\\
            &- \mu \EEb_t \UUt \UUtT - \mu \UUt \UUtT \EEb_t -\mu^2 \MM_t \UUt \UUtT \MM_t.
        \end{align*}
When $c_1\leq 1/2$,
we have $\specnorm{\Id - \mu \UUt \UUtT}\leq 1$ 
by assumption \eqref{eq:UUt1}. It follows from the assumptions $ \mu \le \frac{c_1}{\specnorm{\XXstar}} $, \eqref{eq:UUt3}, and \eqref{eq:UUt2} that
for sufficiently small $c_1,c_2,c_3>0$
\begin{align*}
    \specnorm{\XXstar - \UUtplus \UUtplus^\top}
    \le 
    &\specnorm{\Id - \mu \UUt \UUtT} 
    \specnorm{\XXstar - \UUt \UUtT } 
    \specnorm{\Id -\mu \UUt \UUtT }
    +\mu^2 \specnorm{\UUt}^4 \specnorm{\XXstar - \UUt \UUtT}\\
    &+ 2 \mu \specnorm{\EEb_t} \specnorm{\UUt}^2
    + \mu^2 \specnorm{\MM_t}^2 \specnorm{\UUt}^2\\
    \le 
    & \specnorm{\XXstar - \UUt\UUtT}
    +  4 \mu^2 c_2 \specnorm{\XXstar}^2 \sigma_{\min} \left(\XXstar\right)  
    + 4 \mu c_3 \specnorm{\XXstar} \sigma_{\min} (\XXstar)\\
    &+ 2 (c_2+c_3)^2 \mu^2 \specnorm{\XXstar} \sigma^2_{\min} (\XXstar)  \\
    \le &  \left(c_2+ 4c_1^2c_2+4c_1c_3+2(c_2+c_3)^2c_1^2\right) \sigma_{\min}(\XXstar)\\
    \leq & \left(1-\frac{1}{\sqrt 2}\right) \sigma_{\min}(\XXstar).
\end{align*}
This completes the proof.
\end{proof}

\section{Proofs regarding the Restricted Isometry Property and its consequences}

\subsection{Proof of Lemma~\ref{lem:rank_RIP}}\label{sec:RIPproof}

\newcommand{\diamspec}{d_{\Vert \cdot \Vert}}
\newcommand{\diamfro}{d_{ \Vert \cdot \Vert_F }}
\newcommand{\diamtriple}{d_{\triplenorm{\cdot}}}

As already mentioned in Section \ref{sec:RIP},
there exist similar versions of Lemma \ref{lem:rank_RIP} in the literature
(see, e.g., \cite{candes2011tight}),
which, however, do not specify the dependence of the number of samples $m$ 
on the constant $\delta>0$.
It would be possible to trace the steps of the $\varepsilon$-net argument in \cite{candes2011tight}
and work out the $\delta$-dependence explicitly.
However, this would lead to an extra $\log(1/\delta)$-factor, which is unnecessary.
The reason is that as $\delta$ is decreased, 
a covering with smaller balls is required, leading to a larger $\varepsilon$-net.
This observation suggests a proof strategy based on generic chaining.
Indeed, we will use the following general theorem from \cite{krahmer2014suprema},
which is proven via the generic chaining technique.
To state it, we define the diameter of a set of matrices $\mathcal B$
with respect to some norm $\triplenorm{\cdot}$ as 
\begin{align*}
    \diamtriple (\mathcal B)
    &:=
    \sup_{\BB \in \mathcal{B}}
    \triplenorm{\BB}.
\end{align*}
Moreover,
we will also need Talagrand's functional $\gamma_2 \left(\mathcal B, \triplenorm{\cdot} \right)$ \cite{GenericChaining},
where for a precise definition, we refer to \cite{krahmer2014suprema}.
\begin{theorem}[Theorem 3.1 in \cite{krahmer2014suprema}]\label{lem:generic_chaining}
    Let $\mathcal B$ be a set of matrices, 
    and $\xxii$ be a random Gaussian vector,
    i.e., $\xxii$ has i.i.d. entries with distribution $\mathcal{N} (0,1)$. 
    Set 
    \begin{align} 
    E &:=\gamma_2(\mathcal B, \specnorm{\cdot} )
    \left( \gamma_2\left(\mathcal B, \specnorm{\cdot} \right)+ \diamfro \left(\mathcal B \right)\right)
    + \diamfro (\mathcal B) \diamspec(\mathcal B),\\
    V&:= \diamspec (\mathcal B)\left( \gamma_2\left(\mathcal B, \specnorm{\cdot} \right)
    + \diamfro (\mathcal B)\right),
    \quad U:= \diamspec^2(\mathcal B).
    \end{align}
    Then, for any $t>0$,
    \begin{align}
        \mathbb P\left( \sup_{\BB\in \mathcal B} \left| \twonorm{\BB\xxii}^2-\EE \twonorm{\BB\xxii}^2\right|>c_1E+t \right)
        \leq 2\exp\left(-c_2\min \left\{ \frac{t^2}{V^2}, \frac{t}{U}\right\}\right),
    \end{align}
    where $c_1,c_2 >0$ denote absolute constants.
\end{theorem}
With this result in place, we can give a proof of Lemma \ref{lem:rank_RIP}. This proof strategy has been used in \cite[Section A.3]{krahmer2014suprema}.
\newcommand{\vecXX}{\text{vec} (\XX)}
\begin{proof}[Proof of Lemma~\ref{lem:rank_RIP}]
Since $\Aop$ is a linear operator
we can write $\Aop(\XX)=\VV_{\XX} \xxii$, 
where $\xxii$ is a Gaussian random vector with independent entries of length $m\binom{d+1}{2}$
and 
\[ \VV_{\XX}
:=\frac{1}{\sqrt m}\begin{bmatrix}
    \vecXX^\top &  & & \\
     &\vecXX^\top & & \\
     & & \ddots &\\
     & & &\vecXX^\top
\end{bmatrix}\]
is an $m\times (m\binom{d+1}{2})$ block-diagonal matrix.
Here, $\vecXX \in \mathbb R^{\binom{d+1}{2}}$ is a vector indexed by 
$\{(i,j) \in [d] \times [d]: i\leq j\}$ such that 
\begin{align}
    \vecXX (i,j)=\begin{cases}
        \sqrt{2} \XX_{ij}   & i\not=j\\
        \XX_{ii}  & i=j.
    \end{cases}
\end{align}
Let 
\begin{equation*}
    D_r :=
    \{\XX\in \mathcal{S}^d: 
    \fronorm{\XX}=1, \ \mathrm{rank}(\XX)\leq r \}.
\end{equation*}
Then it follows from the identity $ \Aop (\XX) = \VV_{\XX} \xxii $ that 
\begin{align}
    \delta_{r}
    :=
    \sup_{\XX \in D_r}
    \left| \twonorm{\Aop \left( \XX \right) }^2 - \fronorm{\XX}^2    \right|
    =\sup_{\XX\in D_r}\left| 
        \twonorm{\VV_{\XX} \xxii }^2-\EE \twonorm{\VV_{\XX} \xxii }^2\right| .
\end{align}
Denote $\mathcal B:=\{ \VV_{\XX}: \XX\in D_r\}$.
We now estimate the parameters in Theorem~\ref{lem:generic_chaining}.
Note that it follows directly from the definition of $\vecXX$ that $\twonorm{\vecXX}=\fronorm{\XX}=1$
and hence $\fronorm{\VV_\XX}=\fronorm{\XX}$ for all $X \in \mathcal{S}^d$.
Thus, we have $d_{F}(\mathcal B)=1$ since $\|\VV_{\XX}\|_{F} =\fronorm{\XX}$  for all $\XX\in D_r$.
On the other hand,
for $\XX \in D_r$,
\begin{align}
    m \VV_{\XX}\VV_{\XX}^T= \Id_m,
\end{align}
which implies that  
\begin{align} \label{eq:metric}
\specnorm{\VV_{\XX}}
=\frac{1}{\sqrt m} \twonorm{\vecXX}
=\frac{1}{\sqrt m} \fronorm{\XX}
\end{align}
and $ \diamspec (\mathcal B)=\frac{1}{\sqrt m}$. 
From \cite [Lemma 3.1]{candes2011tight}, 
it follows that the covering number for $d\times d$ symmetric matrices with Frobenius norm $1$ and rank at most $r$ satisfies
\begin{align} \label{eq:covering_number}
    \mathcal N (D_r, \fronorm{\cdot}, \varepsilon)
    \leq 
    \left(1+6/\varepsilon \right)^{(2d+1)r}.
\end{align}
Using Dudley's integral estimate (see, e.g., \cite{GenericChaining}), combined with \eqref{eq:metric} and \eqref{eq:covering_number}, 
we obtain that 
\begin{align}
    \gamma_2\left(\mathcal B, \specnorm{\cdot} \right)
    =
    \gamma_2\left(D_r , \fronorm{\cdot} \right)
    \leq C \frac{1}{\sqrt m}\int_{0}^1 \sqrt{\log (\mathcal N(D_r, \|\cdot \|_F, u))} du\leq C'\sqrt{\frac{dr}{m}}.
\end{align}
With the notations in Theorem~\ref{lem:generic_chaining}, we have
\begin{align}
    E=C'\sqrt{\frac{dr}{m}}\left( C'\sqrt{\frac{dr}{m}}+1\right)+\frac{1}{\sqrt m}  , \quad   V=\frac{1}{\sqrt m} \left( C'\sqrt{\frac{dr}{m}}+1\right),\quad  U=\frac{1}{m}.
\end{align}
Therefore, applying Theorem~\ref{lem:generic_chaining}, 
we have $\delta_r\leq \delta$ with probability at least $1-\varepsilon$ 
when $$m\geq C\delta^{-2}(rd+\log(2\varepsilon^{-1})).$$
Here, $C>0$ denotes some universal constant.
This completes the proof of Lemma \ref{lem:rank_RIP}.
\end{proof}
\subsection{Proof of Lemma \ref{lemma: RIP}}\label{sec:proofRIPLemma}

\begin{proof}[Proof of Lemma \ref{lemma: RIP}]
    We will establish first that for all symmetric matrices $\ZZ_1, \ZZ_2 \in \R^{d \times d }$ with rank $\text{rank} (\ZZ_1)=r$ and  $\text{rank} (\ZZ_1)=r'$ it holds that 
     \begin{equation}\label{ineq:RIPintern1}
        \vert \innerproduct{ \left(\IdOp -\Aops\right) (\ZZ_1), \ZZ_2 } \vert
        \le 
        \delta_{r+r'} \fronorm{\ZZ_1} \fronorm{\ZZ_2}.
     \end{equation}
    Let us remark that in the case of $\innerproduct{\ZZ_1,\ZZ_2}=0$,
    this inequality has been proven in \cite[Lemma 3.3]{candes2011tight}.
    The following proof of this slightly more general statement is analogous.
    
    To prove inequality \eqref{ineq:RIPintern1} we assume without loss of generality 
    that $\fronorm{\ZZ_1}= \fronorm{\ZZ_2}=1$.
    We note first that from the parallelogram identity, it follows that 
    \begin{align*}
        \innerproduct{\Aop \left( \ZZ_1 \right), \Aop \left(\ZZ_2 \right) }
        =
        &\frac{1}{4} \twonorm{ \Aop \left( \ZZ_1 + \ZZ_2 \right) }^2
        -
        \frac{1}{4} \twonorm{ \Aop \left( \ZZ_1 - \ZZ_2 \right) }^2\\
        \le 
        &\frac{1+\delta_{r+r'}}{4} \fronorm{\ZZ_1+\ZZ_2}^2
        -
        \frac{1-\delta_{r+r'}}{4} \fronorm{\ZZ_1-\ZZ_2}^2\\
        =
        &\frac{\delta_{r+r'}}{2} \left( \fronorm{\ZZ_1}^2 + \fronorm{\ZZ_2}^2 \right)
        + \innerproduct{\ZZ_1, \ZZ_2}.
    \end{align*}
    By rearranging terms and using the assumption $\fronorm{\ZZ_1}=\fronorm{\ZZ_2}=1$
    we obtain that 
    \begin{equation*}
        \innerproduct{(\Aops - \IdOp) (\ZZ_1), \ZZ_2}
        =
        \innerproduct{\Aop \left( \ZZ_1 \right), \Aop \left(\ZZ_2 \right) }
        -
        \innerproduct{\ZZ_1, \ZZ_2}
        \le \delta_{r+r'}.
    \end{equation*}
    Since the reverse bound 
    \begin{align*}
        \innerproduct{(\Aops - \IdOp) (\ZZ_1), \ZZ_2}
        \ge -\delta_{r+r'}
    \end{align*}
    can be shown analogously, inequality \eqref{ineq:RIPintern1} follows.
    
    Next, we prove inequality \eqref{ineq:RIPlemma1}.
    For that, we note that
    there exists a matrix $\MM \in \R^{d \times r' }$ with $\fronorm{\MM}=1$
    such that 
    \begin{align*}
        \fronorm{\left( \IdOp - \Aops \right) (\ZZ) \VV }
        &=
        \innerproduct{ \left[ \left( \IdOp - \Aops \right) (\ZZ) \right] \VV , \MM }
        =
        \innerproduct{ \left[ \left( \IdOp - \Aops \right) (\ZZ) \right]  , \VV \MM^\top }\\
        &=
        \innerproduct{ \left( \IdOp - \Aops \right) (\ZZ)  , \frac{1}{2} \VV \MM^\top 
        + \frac{1}{2} \MM \VV^\top  }.
    \end{align*} 
    holds. 
    Using inequality \eqref{ineq:RIPintern1} we obtain that 
    \begin{align*}
        \fronorm{\left( \IdOp - \Aops \right) (\ZZ) \VV }
        \le 
        \delta_{r+2r'} \fronorm{\ZZ} \fronorm{ \frac{1}{2} \VV \MM^\top + \frac{1}{2} \MM \VV^\top}
        \le 
        \delta_{r+2r'} \fronorm{\ZZ} \specnorm{\VV} \fronorm{\MM} 
        =
        \delta_{r+2r'} \fronorm{\ZZ}.
    \end{align*}
    This proves inequality \eqref{ineq:RIPlemma1}.

    Inequality \eqref{ineq:RIPlemma2} is a direct consequence of \eqref{ineq:RIPlemma1}.
    Indeed, let $\vv \in \R^{d} $ with $\twonorm{\vv}=1$ be an eigenvector 
    of 
    $ \left( \IdOp - \Aops \right) \left(\ZZ\right) $
    corresponding to the largest eigenvalue in absolute value.
    It then follows from inequality  \eqref{ineq:RIPlemma1} that
    \begin{equation*}
        \specnorm{\left( \IdOp - \Aops \right) (\ZZ)  }
        =
        \twonorm{\left[\left( \IdOp - \Aops \right) (\ZZ) \right] \vv }
        \le 
        \delta_{r+2} \fronorm{\ZZ}.
    \end{equation*}
    It remains to prove inequality \eqref{ineq:RIPlemma3}.
    Note that using the fact 
    $\innerproduct{\ww \ww^\top, \Projwperp (\ZZ)=0}$, we have 
    \begin{align*}
        \vert \innerproduct{ \Aop (\ww \ww^\top), \Aop \left( \Projwperp (\ZZ) \right) } \vert 
        &=\vert \innerproduct{ \left(\Aops \right) (\ww \ww^\top),    \Projwperp (\ZZ)  ) } \vert\\
        &=\vert \innerproduct{ \left(\IdOp -\Aops \right) \left(\ww \ww^\top\right),    \Projwperp (\ZZ)  } \vert\\
         &\overleq{(a)} \delta_{(r+1)+1} \fronorm{\ww \ww^\top}\fronorm{ \Projwperp (\ZZ) }\\
       & \leq \delta_{r+2}  \fronorm{\ZZ},
    \end{align*}
    where in inequality $(a)$ we used \eqref{ineq:RIPintern1}.
    This completes the proof of Lemma \ref{lemma: RIP}.
    \end{proof}

\section{Proofs for the noisy case}\label{appendix:noise}

\subsection{Preliminaries}

In this section, we extend our main result in the noiseless setting 
to the case 
where the measurements are corrupted by noise,
i.e.,
\begin{equation*}
    \yy
    =
    \Aop 
    \left(
        \XXstar
    \right)+\xxi,
\end{equation*}
where $\xxi\in \mathbb R^{m}$ is a random vector with i.i.d. entries with distribution $\mathcal N(0,\sigma^2)$. 
Then, following the same derivation as in the noiseless case
and using the notation $ \DDeltat = \XXstar - \UUt \UUtT $, 
the gradient descent step becomes
\begin{align}
    \UUtplus 
    &=\UUt 
    + \frac{\mu}{m} \sum_{i=1}^m \innerproduct{\AAi, \XXstar - \UUt \UUtT} \AAi \UUt+\frac{\mu}{\sqrt m}\sum_{i=1}^m \xi_i \AAi \UUt\\
    &=\UUt 
    + \frac{\mu}{m} \sum_{i=1}^m \innerproduct{\AAi, \XXstar - \UUt \UUtT} \AAi \UUt+\mu\Aop^*(\xxi)\UUt.
\end{align}
Define $\hat{\xxi}\in \mathbb R^{m+1}$ such that $\hat{\xxi}=\begin{bmatrix}
\xxi\\
0
\end{bmatrix}$. 
Then we observe that
\begin{align} 
 \Aopw^*(\hat\xxi)
=\frac{1}{\sqrt m}\sum_{i=1}^m \xxi_i \AAiw,
\end{align}
and 
\begin{align}
    \yy_{\ww}=\Aopw 
    \left(
        \XXstar
    \right)+\hat\xxi.
\end{align}
The corresponding virtual sequence $\UUtw$ is now defined by
\begin{align}
    \UUtplusw
:=
\UUtw + \mu \left[ \left( \Aopws \right) \left( \XXstar - \UUtw \UUtw^\top \right) \right] \UUtw+\mu\Aopw^*(\xxi) \UUtw
\end{align}
with $ \UU_{0,\ww} := \UU_0 $.

We now gather several technical lemmas regarding 
the perturbation induced by the noise vector $\xxi$.
The first lemma compares $ \Aopw^* (\hat{\xxi}) $
to $ \Aop^* (\xxi) $.
\begin{lemma} \label{lem:noise_spec}
Assume $m\geq Cd$. 
With probability at least $1-3\exp(-d)$
it holds that
\begin{equation*}
   \sup_{w\in \mathcal N_{\varepsilon}}\specnorm{\Aopw^*(\hat\xxi)-\Aop^*(\xxi)} 
   = 
   \sup_{w\in \mathcal N_{\varepsilon}}\fronorm{\Aopw^*(\hat\xxi)-\Aop^*(\xxi)} 
   \leq 2\sigma \sqrt{2d}.
\end{equation*}
\end{lemma}
\begin{proof}
We have for each $\ww\in \mathcal N_{\varepsilon}$,
\begin{align}
   \specnorm{\Aopw^*(\hat\xxi)-\Aop^*(\xxi)} 
   &=\specnorm{\frac{1}{\sqrt m} \sum_{i=1}^m\xxi_i \langle \AAi,\ww\ww^\top\rangle \ww\ww^\top} \\
   &=\left|\frac{1}{\sqrt m} \sum_{i=1}^m\xxi_i \langle \AAi,\ww\ww^\top \rangle \right|. 
   \label{eq:Awxi}
\end{align}
Note that $\Aopw^*(\hat\xxi)-\Aop^*(\xxi)$ is a rank-1 matrix 
which implies that the first identity in the lemma holds.

Conditioned on $(\AAi)_{i=1}^m$, the random variable
\eqref{eq:Awxi} is a Gaussian random variable with mean zero and variance 
\[\frac{\sigma^2}{m} \sum_{i=1}^m \langle \AAi,\ww\ww^\top \rangle^2=\sigma^2\|\Aop (\ww\ww^\top)\|^2.\]
Then,
conditioned on $(\AAi)_{i=1}^m$, we have
with probability at least $1-2\exp(-d)$ that
\begin{align}
    \specnorm{\Aopw^*(\hat\xxi)-\Aop^*(\xxi)}\leq \sqrt{2d}\sigma \|\Aop (\ww \ww^\top)\|.
\end{align}
With Lemma~\ref{lem:rank_RIP}, when $m\geq Cd$, with probability at least $1-\exp(-d)$, 
\[ \sup_{w\in\mathcal N_{\varepsilon}} \|\Aop (\ww \ww^\top)\| \leq 2.\]
This finishes the proof.
\end{proof}

We also need the following estimate 
on $\Aop^*(\xxi)$, see  \cite[Lemma 1.1]{candes2011tight}.
\begin{lemma}\label{lemma:candesyaniv}
With probability at least $1-2e^{-4d}$,
it holds that 
\begin{equation*}
    \specnorm{\Aop^* ( \xxi ) }
    \le
   c \sqrt{d} \sigma,
\end{equation*}
where $c>0$ is an absolute constant.
\end{lemma}

\subsection{Spectral Initialization}
In the noisy case, we consider the following eigendecompositions:
\begin{align}
    \left(\Aops \right)(\XXstar)+ (\Aop^*)(\xxi)
    &= \widetilde\VV\widetilde\LLambda \widetilde\VV^\top, \quad 
    \UU_{0}=\widetilde \VV_{r}{\widetilde\LLambda_{r}}^{1/2}, \label{eq:defU0_new}\\
    \left(\Aopws \right)(\XXstar)+(\Aopw^*)(\xxi)&= \widetilde\VV_{\ww}\widetilde\LLambda_{\ww} \widetilde\VV_{\ww}^\top, \quad 
    \UU_{0,\ww}=\widetilde \VV_{r,\ww}{\widetilde\LLambda_{r,\ww}}^{1/2}.
\end{align}

The following lemma is an adaptation of Lemma~\ref{lemma:spectralinitialization} to the noisy case. In particular, we recover the  bounds in Lemma~\ref{lemma:spectralinitialization} up to a constant factor  when $\sigma=0$.

\begin{lemma}\label{lemma:spectralinitializationnoise}
    Let $c>0$ be the absolute constant from Lemma~\ref{lemma:candesyaniv}. 
    Then there exists an absolute constant $C>0$  such that the following holds: 
\begin{enumerate}
    \item  With probability at least $1-3\exp(-4d)$, 
 if $m > C^2\kappa^2rd$ is satisfied, it holds that
 \begin{equation}  
        \specnorm{
        \XXstar - \UU_0 \UU_0^\top
        }
    \le C\kappa\sigma_{\min}(\XXstar) \sqrt{\frac{rd}{m}} +c\sqrt{d} \sigma .
    \end{equation}
\item  With probability at least $1-6\exp(-d)$, if $m > 4C^2\kappa^2rd$ is satisfied, 
it holds for every $\ww\in \mathcal N_{\varepsilon}$ that
 \begin{equation} 
        \specnorm{
        \XXstar - \UU_{0,\ww} \UU_{0,\ww}^\top
        }
    \le 2C\kappa\sigma_{\min}(\XXstar)\sqrt{\frac{rd}{m}}+(c+2)\sqrt{d} \sigma.
    \end{equation}
    Consequently, if $m > 4C^2\kappa^2rd$, with probability at least $1-9\exp(-d)$,
    it holds for every $\ww\in \mathcal N_{\varepsilon}$ that
    \begin{align}\label{eq:diff_U0_U0w2}
        \specnorm{\UU_0 \UU_0^\top-\UU_{0,\ww} \UU_{0,\ww}^\top}\leq 3C\kappa\sigma_{\min}(\XXstar) \sqrt{\frac{rd}{m}}+(c+2)\sqrt{d} \sigma.
    \end{align}    
    \item  Assume $\sigma \le \frac{\sigma_{\min}(\XXstar)}{32(c+2)\sqrt{d}}$ 
    and for any $\alpha\in (0,1)$,  
    $m\geq \left(2400C^2+\frac{C_1}{\alpha^2}\right)\kappa^2 rd$ with an absolute constant $C_1>0$.
    Then, 
    with probability at least $1-18\exp(-d)$, 
    it holds for every $\ww\in \mathcal N_{\varepsilon}$ that
    \begin{align}\label{eq:diff_r_Awops_Aops2}
        &\fronorm{
        \UU_{0} \UU_{0}^\top
        -
        \UU_{0,\ww} \UU_{0,\ww}^\top
        }\\
        \le &  \left(  4\alpha +2C\kappa\sqrt{\frac{rd}{m}}\right) \left( \sigma_{\min}(\XXstar) +3\sqrt{2}  C\kappa \sqrt{\frac{rd}{m}}\sigma_{\min}(\XXstar)  +(c+2) \sqrt{2d}\sigma \right) \\ &+8\sigma\sqrt{2d} \left(2+3\sqrt{2} C\kappa \sqrt{\frac{rd}{m}} \right).
    \end{align}
\end{enumerate}
\end{lemma}

\begin{proof}
In our proof we can proceed mostly analogously to the proof of Lemma~\ref{lemma:spectralinitialization} in Appendix~\ref{appendix:spectral}.
Thus we will only highlight the differences.

To prove (1) we 
proceed analogously as in the proof of (1) for Lemma~\ref{lemma:spectralinitialization} in Appendix~\ref{appendix:spectral}
with the difference that in the last inequality chain
we obtain that due to \eqref{eq:defU0_new}, and Lemma~\ref{lemma:candesyaniv},
\begin{align}
    \specnorm{
        \XXstar - \UU_0 \UU_0^\top
        } &\leq \specnorm{
        \XXstar - (\Aops)(\XXstar)
        }  +\specnorm{(\Aops)(\XXstar)- \UU_0 \UU_0^\top}
        +\specnorm{\Aop^* (\xxi)}
        \\
        &\leq \specnorm{
        \XXstar - (\Aops)(\XXstar)
        }  +\specnorm{(\Aops) (\XXstar)- \XXstar}
        + \specnorm{\Aop^* (\xxi) } \\
        &
        \leq  C\kappa\sigma_{\min}(\XXstar)\sqrt{\frac{rd}{m}}
        + c\sqrt{d} \sigma.
\end{align}
with probability $1-3\exp(-4d)$.

In order to prove (2) we
proceed again analogously as in the proof of (2) for Lemma~\ref{lemma:spectralinitialization}
and we note that with probability at least $1-6\exp(-d)$,
\begin{align}
    \specnorm{
        \XXstar - \UU_{0,\ww} \UU_{0,\ww}^\top
        } 
        &\leq \specnorm{
        \XXstar - (\Aopws)(\XXstar)
        }  +\specnorm{(\Aopws)(\XXstar)- \UU_{0,\ww} \UU_{0,\ww}^\top}
        +\specnorm{ (\Aopw^*) (\hat\xxi) }\\
        &\leq 2\specnorm{
        \XXstar - (\Aopws)(\XXstar)
        } 
        +\specnorm{(\Aopw^*) (\hat\xxi) }\\
        &\leq 2C\kappa\sigma_{\min}(\XXstar) \sqrt{\frac{rd}{m}}
        + (c+2)\sqrt{d} \sigma,
\end{align}
where in the last line we have used Lemma \ref{lem:noise_spec}
and Lemma \ref{lemma:candesyaniv}.

In order to prove (3), we condition on all the good events in parts (1) and (2). Denote  
\[\ZZ_{1}:= (\Aops)(\XXstar)+\Aop^* (\xxi), 
\quad \ZZ_2:= (\Aopws)(\XXstar) + (\Aopw^*) (\hat\xxi) ,\]
and 
\[ \ZZ_{1,r}:=\UU_0\UU_0^\top, \quad \ZZ_{2,r}:=\UU_{0,\ww}\UU_{0,\ww}^\top.
\]

We have 
\begin{align}
    \sigma_{r+1} \left(\ZZ_2 \right)
    = 
    \|\sigma_{r+1}(\ZZ_{2})-\sigma_{r+1}(\XXstar)\| 
    \le 
    \| \ZZ_2-\XXstar\| \le 2C\kappa\sigma_{\min}(\XXstar) \sqrt{\frac{rd}{m}}
        + (c+2)\sqrt{d} \sigma.
\end{align}
Thus, instead of inequality \eqref{eq:UUUwUw}
we obtain that
\begin{align}
     &\| \left( \UU_0\UU_0^\top-\UU_{0,\ww}\UU_{0,\ww}^\top\right) \widetilde{\VV}_{r}\|_F
     \nonumber \\
     \leq  &\| (\ZZ_1-\ZZ_{2})\widetilde{\VV}_r\|_F
     +
     \left(C\kappa\sigma_{\min}(\XXstar)  \sqrt{\frac{rd}{m}}+ (c+2) \sqrt{d} \sigma \right)
     \| \widetilde{\VV}_{r,\ww,\perp}^\top\widetilde{\VV}_r\|_F. \label{eq:UUUwUw2}
\end{align}
From \eqref{eq:Aopsw_I}  and \eqref{eq:bernsteinAA}
it follows that, when $m\geq C^2\kappa^2 rd$,
\begin{align}
    \|\ZZ_1-\ZZ_2\| 
    &\leq \frac{3C\kappa}{2}\sigma_{\min}(\XXstar) \sqrt{\frac{rd}{m}}
    +\specnorm{ (\Aop^*)(\xxi) - (\Aopw^*)(\xxi) }\\
    &\le 
    \frac{3C}{2}\kappa\sigma_{\min}(\XXstar) \sqrt{\frac{rd}{m}}
    +2\sigma\sqrt{2d}\leq \left( \frac{ 3C \kappa }{2} \sqrt{\frac{rd}{m}}+\frac{\sqrt 2}{16}\right)\sigma_{\min}(\XXstar), \label{eq:Z_1-Z_2} 
\end{align}
where we used Lemma~\ref{lem:noise_spec} 
and the assumption that $\sigma \le \frac{\sigma_{\min}(\XXstar)}{32(c+2)\sqrt{d}}$ 
in the last line.

Similar to the derivation of \eqref{eq:sigma_r_Z2}, 
by using \eqref{eq:Aopsw_I} and Weyl's inequalities
we obtain that
\begin{align}
    |\sigma_r(\ZZ_1)-\sigma_{\min}(\XXstar)| &\leq |  \sigma_r((\Aops)(\XXstar))-\sigma_{\min}(\XXstar)|+ |\sigma_r((\Aops)(\XXstar))-\sigma_r(\ZZ_1) |\\
    &\leq |  \sigma_r((\Aops)(\XXstar))-\sigma_{\min}(\XXstar)|+ \specnorm{\Aop^*(\xxi)}\\
    &\leq  C\kappa\sigma_{\min}(\XXstar) \sqrt{\frac{rd}{m}}
    +c\sigma \sqrt{d}
\end{align}
and
\begin{align}
   \sigma_{r+1}(\ZZ_1) &\leq C\kappa\sigma_{\min}(\XXstar) \sqrt{\frac{rd}{m}}+c \sigma \sqrt{d}.
\end{align}
Therefore, if $m>64C^2\kappa^2rd$, 
the spectral gap between $\sigma_r (\ZZ_1)$ and $\sigma_{r+1} (\ZZ_1)$ 
can be bounded from below by
\begin{align}
    \sigma_r(\ZZ_1)-\sigma_{r+1}(\ZZ_1) \geq \left(1-2C\kappa \sqrt{\frac{rd}{m}}  \right)\sigma_{\min}(\XXstar)-2 c\sqrt{d} \sigma\geq \frac{1}{2}\sigma_{\min}(\XXstar),
\end{align}
where we have used the assumption that $\sigma \le \frac{\sigma_{\min}(\XXstar)}{32(c+2)\sqrt{d}}$. 
From \eqref{eq:Z_1-Z_2}, when $m\geq 2400C^2\kappa^2 rd$,
\begin{align}
     \|\ZZ_1-\ZZ_2\| \leq \left(1-\frac{1}{\sqrt{2}} \right)  \left(\sigma_r(\ZZ_1)-\sigma_{r+1}(\ZZ_1)\right).
\end{align}
Then we can apply Davis-Kahan inequality (Lemma~\ref{lem:DavisKahan})
as in the proof of Lemma~\ref{lemma:spectralinitialization}
and follow the proof steps to obtain 
\begin{align}
      \| \widetilde{\VV}_{r,\ww,\perp}^\top\widetilde{\VV}_r\|_F
    \leq 
    \frac{2\sqrt{2}\| (\ZZ_1-\ZZ_{2})\widetilde{\VV}_r\|_F }{\sigma_{\min}(\XXstar)}.
\end{align}
Hence,
when $m\geq \left(2400C^2+\frac{C_1}{\alpha^2}\right)\kappa^2 rd$,
we obtain from \eqref{eq:UUUwUw2} and \eqref{eq:diff_Awops_Aops} that when $\sigma \le \frac{\sigma_{\min}(\XXstar)}{32(c+2)\sqrt{d}}$,
\begin{align}
    & \| \left( \UU_0\UU_0^\top-\UU_{0,\ww}\UU_{0,\ww}^\top\right) \widetilde{\VV}_{r}\|_F\\
     \leq 
     &\left( 1+2\sqrt{2}C\kappa \sqrt{\frac{rd}{m}}
     + \frac{2 (c+2) \sigma \sqrt{2d}}{\sigma_{\min} (\XXstar)} \right)
     \fronorm{(\ZZ_1-\ZZ_2)\widetilde{\VV
     }_r}
    \\
     \leq &2\fronorm{(\ZZ_1-\ZZ_2)\widetilde{\VV
     }_r}\\
     \leq &  \left( 2\alpha +C\kappa\sqrt{\frac{rd}{m}}\right)\sigma_{\min}(\XXstar)
     +
     2\fronorm{\left(\Aop^* (\xxi) - \Aopw^* ( \hat\xxi ) \right) 
     \widetilde{\VV}_r}\\
     \leq &  \left( 2\alpha +C\kappa\sqrt{\frac{rd}{m}}\right)\sigma_{\min}(\XXstar)
     +
     4\sigma \sqrt{2d},
     \label{eq:rdm3}
\end{align}
where in the last inequality we use Lemma~\ref{lem:noise_spec}.
Following the remaining proof steps of Lemma~\ref{lemma:spectralinitialization}, we find the analog of \eqref{eq:28second1} to be
\begin{align}
    &\| \left( \UU_0\UU_0^\top-\UU_{0,\ww}\UU_{0,\ww}^\top\right) \widetilde{\VV}_{r,\perp}\|_F\\
    \leq & \left( 2\alpha +C\kappa\sqrt{\frac{rd}{m}}\right)\sigma_{\min}(\XXstar)+4\sigma \sqrt{2d}+ \| \widetilde{\VV}_{r,\perp}^\top \UU_{0,\ww}\UU_{0,\ww}^\top  \widetilde{\VV}_{r,\perp}\|_F
    \label{eq:new86}
\end{align}
and the analog of \eqref{eq:28second2} to be
\begin{align}
    &\| \widetilde{\VV}_{r,\perp}^\top \UU_{0,\ww}\UU_{0,\ww}^\top  \widetilde{\VV}_{r,\perp}\|_F \\
    \leq & \specnorm{\UU_{0,\ww}\UU_{0,\ww}^\top-\UU_0\UU_0^\top}\fronorm{\widetilde{\VV}_{r,\ww}^\top\widetilde{\VV}_{r,\perp}}\\
    \leq & \left(3C\kappa\sigma_{\min}(\XXstar) \sqrt{\frac{rd}{m}}+(c+2)\sqrt{d} \sigma\right)\frac{2\sqrt{2}\| (\ZZ_1-\ZZ_{2})\widetilde{\VV}_r\|_F }{\sigma_{\min}(\XXstar)}\\
    \leq & 2\sqrt{2} \left(3C\kappa\sigma_{\min}(\XXstar) \sqrt{\frac{rd}{m}}+(c+2)\sqrt{d} \sigma\right)\left(2\alpha +C\kappa\sqrt{\frac{rd}{m}}
     +\frac{4\sigma \sqrt{2d}}{\sigma_{\min}(\XXstar)}\right) .\label{eq:new87}
\end{align}
Therefore, combining  \eqref{eq:new86} and \eqref{eq:new87}, 
we obtain 
\begin{align}
    & \fronorm{
        \UU_{0} \UU_{0}^\top
        -
        \UU_{0,\ww} \UU_{0,\ww}^\top
        } \\
       \leq & \left(  4\alpha +2C\kappa\sqrt{\frac{rd}{m}}\right)\sigma_{\min}(\XXstar)+8\sigma \sqrt{2d} \\
       & +2\sqrt{2} \left(3C\kappa\sigma_{\min}(\XXstar) \sqrt{\frac{rd}{m}}+(c+2)\sqrt{d} \sigma\right)\left(2\alpha +C\kappa\sqrt{\frac{rd}{m}}
     +\frac{4\sigma \sqrt{2d}}{\sigma_{\min}(\XXstar)}\right)\\
     =&\left(  4\alpha +2C\kappa\sqrt{\frac{rd}{m}}\right) \left( \sigma_{\min}(\XXstar) +3\sqrt{2}  C\kappa \sigma_{\min}(\XXstar) \sqrt{\frac{rd}{m}} +(c+2) \sqrt{2d}\sigma \right)\\
     &+ 8\sigma\sqrt{2d} \left(  1+3\sqrt{2} C\kappa \sqrt{\frac{rd}{m}} +(c+2) \frac{\sqrt{2d}\sigma}{\sigma_{\min}(\XXstar)}\right)\\
     \leq &\left(  4\alpha +2C\kappa\sqrt{\frac{rd}{m}}\right) \left( \sigma_{\min}(\XXstar) +3\sqrt{2}  C\kappa \sigma_{\min}(\XXstar) \sqrt{\frac{rd}{m}} +(c+2) \sqrt{2d}\sigma \right) \\ &+8\sigma\sqrt{2d} \left(2+3\sqrt{2} C\kappa \sqrt{\frac{rd}{m}} \right),
\end{align}
where in the last inequality, we use the assumption $\sigma \le \frac{\sigma_{\min}(\XXstar)}{32(c+2)\sqrt{d}}$. This completes
the proof of part (iii).
\end{proof}

\subsection{Convergence proof}

For the convergence proof,
we proceed analogously to the proof of Theorem~\ref{thm:main} in the noiseless case
and incorporate the noise as perturbation terms.

\subsubsection{Controlling the distance between $\UU_t$ and $\UU_{t,\ww}$}

We start with the following lemma, 
which is an adaptation of Lemma~\ref{lemma:auxdistanceweak} to the noisy case.
\begin{lemma}\label{lemma:auxdistanceweak_noise}

Assume that the assumptions of Lemma \ref{lemma:auxdistanceweak} hold.
Assume in addition that 
the noise level satisfies
$\sigma \le \frac{\tilde{c} \sigma_{\min} (\XXstar)}{\sqrt{d}}$
for some sufficiently small $\tilde{c} >0$.  
Moreover, assume that the statements of Lemma \ref{lem:noise_spec}
and Lemma \ref{lemma:candesyaniv} hold.
     Then it holds that
     \begin{equation*}
      \fronorm{\UUtplus \UUtplus^\top - \UUtplusw \UUtplusw^\top}
      \le 
     \frac{\sqrt{\sqrt{2} - 1}}{40} \sigma_{\min} (\XXstar).
     \end{equation*} 
  \end{lemma}

\begin{proof}[Proof of Lemma \ref{lemma:auxdistanceweak_noise}]
The proof is an adaptation of the proof of Lemma~\ref{lemma:auxdistanceweak} to the noisy case.
We define
\begin{align*}
     \MM_t     &:= \left( \Aops \right) \left( \XXstar - \UUt \UUtT \right)
     + \Aop^* (\xxi) , \\
     \MM_{t,\ww} &:= \left( \Aopws \right) \left( \XXstar - \UUtw \UUtw^\top \right) +\Aopw^* (\hat{\xxi}) .
\end{align*}
We obtain that
\begin{align}
       \specnorm{\MM_t}
     \overleq{(a)}
     &\specnorm{\XXstar - \UUt \UUt}
     +
     \specnorm{ \left(\Aops - \IdOp \right) (\XXstar - \UUt \UUtT)}
     +
     \specnorm{\Aop^* (\xxi) }
     \nonumber \\
     \overleq{(b)}
     &  \specnorm{\XXstar - \UUt \UUtT}
     + \constone \sigma_{\min} \left( \XXstar \right)
     +
     c \sqrt{d} \sigma
     \label{ineq:weakboundintern5noise}\\
     \overleq{(c)}
     & 2 \sigma_{\min} \left( \XXstar \right).
     \label{ineq:weakboundintern9noise}
\end{align}
which holds under the assumption $\sigma\leq \frac{\sigma_{\min}(\XXstar)}{32(c+2)\sqrt{d}}$
and for sufficiently small $c_1$.
Moreover, we note that 
   \begin{align*}
     \MM_t - \MM_{t,\ww}
     =
     \left( \Aops - \Aopws \right) \left( \XXstar - \UUt \UUtT \right)
     -
     \left( \Aopws \right) \left(\UUt \UUtT - \UUtw \UUtw^\top \right)
     +\Aop^*(\xxi)-\Aopw^*(\hat\xxi).
   \end{align*}
   It follows that
   \begin{align}
     &\fronorm{ \left( \MM_t - \MM_{t,\ww} \right) \VV_{\UUtw} }\\
     \le
     &\fronorm{ \left[\left( \Aops - \Aopws \right) \left( \XXstar - \UUt \UUtT \right)\right] \VV_{\UUtw}} 
     +
     \fronorm{\left[ \left( \Aopws -\IdOp \right) \left(\UUt \UUtT - \UUtw \UUtw^\top \right) \right] \VV_{\UUtw}} \nonumber \\
     &+\fronorm{\UUt\UUtT-\UUtw \UUtw^\top}
     + \fronorm{(\Aop^*(\xxi)-\Aopw^*(\hat\xxi))\VV_{\UUtw} }
      \nonumber \\
     \overleq{(a)}
     &\left( \delta + \frac{8 \sqrt{rd} }{\sqrt{m}} \right) 
     \specnorm{\XXstar - \UUt \UUtT}
     \left( 3\delta+ \frac{4 \sqrt{2d}}{\sqrt{m}}+1 \right) \fronorm{ \UUt \UUtT - \UUtw \UUtw^\top }
     +
     2 \sigma \sqrt{2d} 
     \nonumber\\
     \leq
     &\frac{2\constthree}{\kappa} \specnorm{\XXstar - \UUt \UUtT}
     + \left( \frac{4 \constthree}{\kappa} +1 \right) \fronorm{\UUt \UUtT - \UUtw \UUtw^\top} +2\sigma\sqrt{2d},
     \label{ineq:weakboundintern7_new}
   \end{align}
   where inequality (a) 
   follows from Lemma \ref{lem:noise_spec}.  
    Thus, we obtain that
      \begin{align}
     \specnorm{\MM_{t,\ww} \VV_{\UUtw}}
     &\le 
     \specnorm{\MM_t}
     +
     \fronorm{\left(\MM_t -\MM_{t,\ww} \right) \VV_{\UUtw}} \nonumber \\
     &\leq 
     2 \sigma_{\min} (\XXstar)
     +
     \frac{2 \constthree}{\kappa} \specnorm{\XXstar - \UUt \UUtT} 
     +\left( \frac{4 \constthree }{\kappa}+1 \right) \fronorm{\UUt \UUtT - \UUtw \UUtw^\top} 
     + 2 \sigma  \sqrt{2d} \nonumber \\
     &\leq
     3 \sigma_{\min} (\XXstar).
   \end{align}
    Next, we use the decomposition
    \begin{align*}
        \UUtplus \UUtplus^\top - \UUtplusw \UUtplusw^\top
        = \UUt \UUtT - \UUtw \UUtw^\top
        + \mu \cdot (i)
        + \mu \cdot (ii)
        + \mu \cdot (iii)
        + \mu \cdot (iv)
        + \mu^2 \cdot (v),
    \end{align*}
    where $(i), (ii), (iii), (iv), (v)$ are defined as
    in the proof of Lemma~\ref{lemma:auxdistanceweak}.
   Then by arguing exactly as in the proof of Lemma~\ref{lemma:auxdistanceweak}, 
   we can obtain that
   \begin{align*}
     \fronorm{ (ii) } = \fronorm{ (iv) } 
    &\leq
    3 \specnorm{\XXstar} \fronorm{ \left(\MM_t - \MM_{t,\ww} \right) \VV_{\UUtw} },\\
    \fronorm{ (v) }
    &\leq
    4 \sigma_{\min}^2 (\XXstar) 
    \fronorm{ \UUt \UUtT - \UUtw \UUtw^\top }
    + 15 \sigma_{\min} (\XXstar) \specnorm{\XXstar} \fronorm{ \left(\MM_t - \MM_{t,\ww} \right) \VV_{\UUtw} },
   \end{align*}
   In order to bound $\fronorm{(i)}$ and $\fronorm{(iii)}$, we observe that
   \begin{align*}
    \fronorm{(i)}
    =
    \fronorm{(iii)}
    &\leq 
    \specnorm{\MM_t} \fronorm{\UUt \UUtT - \UUtw \UUtw^\top}\\
    &\leq
    \left( \specnorm{\XXstar - \UUt \UUtT} 
    + c_1 \sigma_{\min} (\XXstar) 
    + c \sqrt{d} \sigma \right)
    \fronorm{\UUt \UUtT - \UUtw \UUtw^\top},
   \end{align*}
   where in the last inequality we used \eqref{ineq:weakboundintern5noise}.
   Then,
   as in the proof of Lemma~\ref{lemma:auxdistanceweak}, we can combine all the inequalities
   and obtain that
   \begin{align*}
     &\fronorm{\UUtplus \UUtplus^\top - \UUtplusw \UUtplusw^\top}\\
     \le
     &\fronorm{\UUt \UUtT - \UUtw \UUtw^\top}
     +
     2 \mu \left( \specnorm{\XXstar - \UUt \UUtT} + \constone \sigma_{\min} (\XXstar)
     + c \sigma \sqrt{d}  \right) 
      \fronorm{\UUt \UUtT - \UUtw \UUtw^\top}\\
     &+
     6\mu \specnorm{\XXstar} \fronorm{ \left(\MM_t - \MM_{t,\ww} \right) \VV_{\UUtw} }\\
     &+
     \mu^2 \left( 4  \sigma^2_{\min} (\XXstar)  \fronorm{\UUt \UUtT - \UUtw \UUtw^\top}
     +
     15  \sigma_{\min} (\XXstar) \specnorm{\XXstar} \fronorm{\left(\MM_t - \MM_{t,\ww} \right) \VV_{\UUtw}}  \right)\\
     \overleq{(a)}
     &\left(1+2\mu \specnorm{\XXstar - \UUt \UUtT} + 2c_1 \sigma_{\min} (\XXstar)
     + 2 \mu c \sigma \sqrt{d}  \right)
     \fronorm{\UUt \UUtT - \UUtw \UUtw^\top}\\
     &+12 \mu \sigma_{\min} (\XXstar ) \constthree \specnorm{\XXstar - \UUt \UUtT}
     + 6 \mu \specnorm{\XXstar} \left( \frac{4 \constthree}{\kappa} +1 \right) \fronorm{\UUt \UUtT - \UUtw\UUtw^\top} 
     + 12 \mu \specnorm{\XXstar}  \sigma \sqrt{2d} \\
     &+
     4\mu^2  \sigma^2_{\min} (\XXstar)  \fronorm{\UUt \UUtT - \UUtw \UUtw^\top}
     +
     30c_3 \mu^2 \sigma_{\min}^2 (\XXstar) \specnorm{\XXstar - \UUt \UUtT}\\
     &+ 60 c_3 \mu^2  \sigma_{\min}^2 (\XXstar) \fronorm{\UUt \UUtT - \UUtw \UUtw^\top } 
     +15 \mu^2 \sigma_{\min} (\XXstar) \specnorm{\XXstar} \fronorm{\UUt \UUtT - \UUtw \UUtw^\top } 
     \\
     &+ 30 \sqrt{2d} \mu^2 \sigma \cdot  
     \sigma_{\min} \left( \XXstar \right) \specnorm{\XXstar}\\
     =
     & \left( 1 + 2 \mu \specnorm{\XXstar-\UUt \UUtT} 
     +\left( 2 \constone + 24 \constthree  \right) \mu \sigma_{\min} (\XXstar)+6\mu \specnorm{\XXstar}
     + 4\mu^2 \sigma_{\min}^2 (\XXstar)+ 60 \constthree \mu^2 \sigma^2_{\min} (\XXstar) \right)\\
     & \cdot \fronorm{\UUt \UUtT - \UUtw \UUtw^\top}+ 
     \left( 12 \constthree \mu \sigma_{\min} (\XXstar)+ 30 \constthree \mu^2 \sigma_{\min}^2 (\XXstar)  \right)
     \specnorm{\XXstar - \UUt \UUtT}\\
     &+ \left( 12 \mu  + 30\mu^2 \sigma_{\min} (\XXstar) \right)
     \specnorm{\XXstar} \sigma \sqrt{2d}\\
     \overleq{(b)} 
     &\frac{\sqrt{\sqrt{2}-1}}{40} \sigma_{\min} (\XXstar).
   \end{align*}
   Inequality (a) follows form \eqref{ineq:weakboundintern7_new}.
   In inequality (b),
   compared to the proof of Lemma~\ref{lemma:auxdistanceweak},
   we used additionally the assumption that
   $\sigma \le \frac{\tilde{c} \sigma_{\min} (\XXstar)}{\sqrt{d}}$ for some sufficiently small $\tilde{c}>0$.
  This completes the proof of the lemma. 
\end{proof}
Next, we prove the following lemma,
which is an adaptation of Lemma~\ref{lemma:auxsequencecloseness}
to the noisy case.
\begin{lemma}\label{lemma:auxsequencecloseness_noise}
   Assume that the assumptions of Lemma \ref{lemma:auxsequencecloseness} hold.
Assume in addition that 
the noise level satisfies
$\sigma \le \frac{\tilde{c} \sigma_{\min} (\XXstar)}{\sqrt{d}}$
for some sufficiently small $\tilde{c} >0$.  
Moreover, assume that the statements of Lemma \ref{lem:noise_spec}
and Lemma \ref{lemma:candesyaniv} hold.
    Then, it holds that
    \begin{align*}
        &\fronorm{\VXXT \left( \UUtplus \UUtplus^\top - \UUtplusw \UUtplusw^\top \right)}\\
        \le 
        &\left( 1- \frac{\mu \sigma_{\min} (\XXstar)}{16} \right)
        \fronorm{ \VXXT \left( \UUt \UUt^\top - \UUtw \UUtw^\top \right) }
        +
        \mu   \sigma_{\min}(\XXstar)\specnorm{\XXstar - \UUt \UUtT}
        +
        20 \mu \sqrt{d} \sigma \specnorm{\XXstar}.
    \end{align*}
    \end{lemma}

\begin{proof}[Proof of Lemma \ref{lemma:auxsequencecloseness_noise}]
    Analogously to the proof of Lemma~\ref{lemma:auxsequencecloseness}, we can compute that
    \begin{align*}
        &\UUtplus \UUtplus^\top\\
        =
        &\left( \Id +\mu \left[ \left( \Aops \right) \left(\XXstar- \UUt \UUt^\top \right) 
        + \mu \Aop^* (\xxi) \right]  \right)  \UUt \UUt^\top   \left( \Id +\mu \left[ \left( \Aops \right) \left(\XXstar- \UUt \UUt^\top \right)
        + \mu \Aop^* (\xxi) \right]  \right) \\
        =
        &\left( \Id +\mu  (\XXstar -  \UUt \UUt^\top - \UUtw \UUtw^\top ) \right) \UUt \UUt^\top \left( \Id +\mu  (\XXstar -  \UUt \UUt^\top -  \UUtw \UUtw^\top ) \right)\\
        &+\mu \UUtw \UUtw^\top  \UUt \UUt^\top + \mu \UUt \UUt^\top  \UUtw \UUtw^\top\\
        &+\mu^2 \UUtw \UUtw^\top  \UUt \UUt^\top \left( \XXstar - \UUt \UUt^\top \right)
        +\mu^2\left( \XXstar - \UUt \UUt^\top \right) \UUt \UUt^\top  \UUtw \UUtw^\top \\
        &- \mu^2  \UUtw \UUtw^\top  \UUt \UUt^\top  \UUtw \UUtw^\top\\ 
        &+\mu \left[ \left(  \Aops - \IdOp \right) \left(\XXstar- \UUt \UUt^\top \right) \right] \UUt \UUt^\top \left( \Id + \mu \XXstar - \mu \UUt \UUt^\top \right)\\
        &+\mu \left( \Id + \mu \XXstar - \mu \UUt \UUt^\top \right) \UUt \UUt^\top \left[ \left(  \Aops - \IdOp \right) \left(\XXstar- \UUt \UUt^\top \right) \right]\\
        &+ \mu^2 \left[ \left(  \Aops - \IdOp \right) \left(\XXstar- \UUt \UUt^\top \right) \right] \UUt \UUt^\top  \left[ \left(  \Aops - \IdOp \right) \left(\XXstar- \UUt \UUt^\top \right) \right]\\
        &+ \mu \Aop^* (\xxi) \UUt \UUtT
        + \mu \UUt \UUtT \Aop^* (\xxi)
        +\mu^2  \Aop^* (\xxi)  \UUt \UUt  \Aop^* (\xxi) 
    \end{align*}
    and
    \begin{align*}
        &\UUtplusw \UUtplusw^\top\\
        =
        &\left( \Id +\mu  (\XXstar -  \UUt \UUt^\top - \UUtw \UUtw^\top ) \right) \UUtw \UUtw^\top \left( \Id +\mu  (\XXstar -  \UUt \UUt^\top -  \UUtw \UUtw^\top ) \right)\\
        &+\mu \UUtw \UUtw^\top  \UUt \UUt^\top + \mu \UUt \UUt^\top  \UUtw \UUtw^\top\\
        &+\mu^2 \UUt \UUt^\top  \UUtw \UUtw^\top \left( \XXstar - \UUtw \UUtw^\top \right)
        +\mu^2\left( \XXstar - \UUtw \UUtw^\top \right) \UUtw \UUtw^\top  \UUt \UUt^\top \\
        &- \mu^2  \UUt \UUt^\top  \UUtw \UUtw^\top  \UUt \UUt^\top\\ 
        &+\mu \left[ \left(  \Aopws - \IdOp \right) \left(\XXstar- \UUtw \UUtw^\top \right) \right] \UUtw \UUtw^\top \left( \Id + \mu \XXstar - \mu \UUtw \UUtw^\top \right)\\
        &+\mu \left( \Id + \mu \XXstar - \mu \UUtw \UUtw^\top \right) \UUtw \UUtw^\top \left[ \left(  \Aopws - \IdOp \right) \left(\XXstar- \UUtw \UUtw^\top \right) \right]\\
        &+ \mu^2 \left[ \left(  \Aopws - \IdOp \right) \left(\XXstar- \UUtw \UUtw^\top \right) \right] \UUtw \UUtw^\top  \left[ \left(  \Aopws - \IdOp \right) \left(\XXstar- \UUtw \UUtw^\top \right) \right]\\
        &+ \mu \Aopw^* (\xxi) \UUtw \UUtw^\top
        + \mu \UUtw \UUtw^\top \Aopw^* (\xxi)
        +\mu^2  \Aopw^* (\xxi)  \UUtw \UUtw  \Aopw^* (\xxi).
    \end{align*}
    Let $\MM_i$ for $i=1,\ldots,7$ be defined as in the proof of Lemma~\ref{lemma:auxsequencecloseness}.
    Moreover, we define
    \begin{align*}
        \MM_8
        :=
        &\mu \left( \Aops^* (\xxi) - \Aopw^* (\xxi) \right) \UUt \UUt^\top
        + \mu \UUt \UUt^\top \left( \Aops^* (\xxi) - \Aopw^* (\xxi) \right)\\
        &+ \mu^2 
        \left(  \Aop^* (\xxi)  \UUt \UUt  \Aop^* (\xxi) 
        -\Aopw^* (\xxi)  \UUtw \UUtw  \Aopw^* (\xxi)  \right)\\
        =&\mu \left( \Aops^* (\xxi) - \Aopw^* (\xxi) \right) \UUt \UUt^\top
        + \mu \UUt \UUt^\top \left( \Aops^* (\xxi) - \Aopw^* (\xxi) \right)\\
        &+ \mu^2
        \left(  \Aop^* (\xxi) - \Aopw^* (\xxi) \right)  \UUt \UUt  \Aop^* (\xxi)
        + \mu^2
        \Aopw^* (\xxi)  \UUtw \UUtw  \left( \Aop^* (\xxi) - \Aopw^* (\xxi) \right).
    \end{align*}
    Thus, we obtain that
    \begin{align}
        &\UUtplus \UUtplus^\top-\UUtplusw \UUtplusw^\top
        \\
        = &
        \MM_1+\mu^2 \MM_2 + \mu^2 \MM_3 +\mu^2 \MM_4 + \mu^2 \MM_4 + \mu \MM_5 + \mu \MM_6 + \mu^2 \MM_7
        +\MM_8. \label{ineq:intern42_noise}
    \end{align}
    We need to estimate the Frobenius norm of each of the terms $\MM_i$ for $i=1,\ldots,7$.
    For $\MM_1, \ldots, \MM_7$ we can proceed as in the proof of Lemma~\ref{lemma:auxsequencecloseness}.
    It remains to estimate the Frobenius norm of $\MM_8$.
    We obtain that
    \begin{align*}
        \fronorm{\MM_8}
        \le
        &2 \mu \fronorm{\Aop^* (\xxi) - \Aopw^* (\xxi) }
        \specnorm{\UUt}^2
        +
        \mu^2
        \fronorm{\Aop^* (\xxi) - \Aopw^* (\xxi) }
        \specnorm{\UUt}^2
        \specnorm{\Aop^* (\xxi)}\\
        &+
        \mu^2
        \fronorm{\Aop^* (\xxi) - \Aopw^* (\xxi) }
        \left(
        \specnorm{\UUt }^2
        -
        \specnorm{\UUt \UUt^\top - \UUtw \UUtw^\top }
        \right)
        \left(
            \specnorm{\Aop^* (\xxi)}
            +
            \specnorm{\Aop^* (\xxi) - \Aopw^* (\xxi)}
        \right)\\
        \le
        &8  \sqrt{2d} \mu \sigma  \specnorm{\XXstar}
        +
        8\sqrt{2d} \mu^2  \sigma  \specnorm{\XXstar} c d \sqrt{\sigma}
        +
        \mu^2  2 \sigma \sqrt{2d}
        \left( 4 \specnorm{\XXstar} + c_3 \sigma_{\min} (\XXstar)  \right)
        \left(
            c \sqrt{d} \sigma
            +
            2 \sqrt{2} \sigma
        \right)\\
        \le
        &20 \mu \sqrt{d} \sigma \specnorm{\XXstar}.
    \end{align*}
    Then, by arguing similarly as in the proof of Lemma \ref{lemma:auxsequencecloseness}, we obtain that
    \begin{align*}
        &\fronorm{
            \VXXT
            \left(
                \UUtplus \UUtplus^\top - \UUtplusw \UUtplusw^\top
            \right)
        }\\
        \le
        &\left(
            1 - \frac{\mu \sigma_{\min} (\XXstar )}{16}
        \right)
        \fronorm{ \VXXT \left(\UUt \UUt^\top - \UUtw \UUtw^\top \right)}
        +
        \mu \sigma_{\min} (\XXstar) \specnorm{\XXstar - \UUt \UUtT}
        +
        20 \mu \sqrt{d} \sigma \specnorm{\XXstar}.
    \end{align*}
This completes the proof of the lemma.
\end{proof}

\subsubsection{Lemmas controlling the distance between $\XXstar$ 
and $\UUt \UUtT$}

The following lemma is an adaptation of Lemma~\ref{lemma:localconv} to the noisy case.
\begin{lemma}\label{lemma:localconv_noise}
    Let $\triplenorm{\cdot}$ be a norm
    which is submultiplicative in the sense of inequality \eqref{ineq:triplenorm}.
    Assume that
    the assumptions of Lemma \ref{lemma:localconv} hold.
    Then it holds that
    \begin{align*}
        &\triplenorm{\VXXT \left( \UUtplus \UUtplus^\top - \XXstar \right)}\\
        \le 
        &\left(1 - \frac{\mu}{8} \sigma_{\min} \left( \XXstar \right) \right) 
        \triplenorm{\VXXT \left( \XXstar - \UUt \UUtT \right) }
        +
        5 \mu \specnorm{\XXstar}
        \triplenorm{
            \left[ \left(\Aops - \IdOp \right)
            \left(\XXstar - \UUt \UUtT\right)
            \right]
            \VUUt
        }\\
        &+5 \mu \specnorm{\XXstar}
        \triplenorm{
         \left(   \Aop^* (\xxi) \right) \VUUt
        }
        .
    \end{align*}
\end{lemma}

\begin{proof}[Proof of Lemma \ref{lemma:localconv_noise}]
We can proceed as in the proof of Lemma \ref{lemma:localconv}
with the only change that the matrix $\EEb_t$ is now defined as
\begin{equation*}
    \EEb_t
    :=
    \left( \Aops \right) \left( \XXstar - \UUt \UUtT \right)
    + \Aop^* (\xxi).
\end{equation*}
Then by proceeding as in the proof of Lemma \ref{lemma:localconv}
we obtain that
\begin{align*}
    &\triplenorm{\VXXT \left( \UUtplus \UUtplus^\top - \XXstar \right)}\\
    \le 
    &\left(1 - \frac{\mu}{8} \sigma_{\min} \left( \XXstar \right) \right) 
    \triplenorm{\VXXT \left( \XXstar - \UUt \UUtT \right) }
    +
    5 \mu \specnorm{\XXstar}
    \triplenorm{
        \EEb_t
        \VUUt
    }
\end{align*}
Then the claim follows from the triangle inequality.
\end{proof}
The next lemma is an adaptation of Lemma~\ref{lemma:convaprioribound}
to the noisy case.
\begin{lemma}\label{lemma:convaprioribound_noise}
    Assume that the assumptions of Lemma \ref{lemma:convaprioribound} hold.
    Moreover, assume in addition that 
    the noise level satisfies
$\sigma \le \frac{\tilde{c} \sigma_{\min} (\XXstar)}{\sqrt{d}}$
for some sufficiently small $\tilde{c} >0$.  
Moreover, assume that the statements of Lemma \ref{lem:noise_spec}
and Lemma \ref{lemma:candesyaniv} hold.    
    Then it holds that
    \begin{align*}
        \specnorm{ \XXstar - \UUtplus \UUtplus^\top  }
        &\le 
        \left(
        1-\frac{1}{\sqrt{2}}
        \right)
        \sigma_{\min} \left( \XXstar \right).
    \end{align*}   
\end{lemma}

\begin{proof}[Proof of Lemma \ref{lemma:convaprioribound_noise}]
    Define the matrix $\MM_t$ as
    \begin{align*}
        \MM_t
        &:= \left( \Aops \right) \left( \XXstar - \UUt \UUtT \right)
        + \Aop^* (\xxi)\\
        &= 
        \XXstar - \UUt \UUtT
        +
        \bracing{=:\EEb_t}{\left( \Aops - \IdOp \right) \left( \XXstar - \UUt \UUtT \right)
        + \Aop^* (\xxi)}.
    \end{align*}
Then we obtain that 
\begin{align*}
    \specnorm{\EEb_t}
    \le 
    \specnorm{ \left( \Aops \right) \left( \XXstar - \UUt \UUtT \right)}
    + \specnorm{\Aop^* (\xxi)}
    \le
    c_3 \sigma_{\min} \left( \XXstar \right)
    + c \sqrt{d} \sigma
    \le
    \left( c_3 + \sqrt{2} c \tilde{c} \right) \sigma_{\min} \left( \XXstar \right),
\end{align*}
where we have used Lemma \ref{lemma:candesyaniv}
and the assumption on the noise level $\sigma$.
Consequently, we also obtain that
\begin{align*}
    \specnorm{\MM_t}
    \le 
    &\specnorm{ \XXstar - \UUt \UUtT}
    + \specnorm{\EEb_t}
    \le 
    \left( c_2 +c_3 + \sqrt{2} c \tilde{c} \right) \sigma_{\min} \left( \XXstar \right).
\end{align*}
With these estimates in place we can proceed similarly as in the proof of Lemma \ref{lemma:convaprioribound}
to complete the proof of this lemma.
\end{proof}

\subsection{
    Statement and proof of the main convergence lemma in the noisy case
}

Having gathered all technical lemmas, we can now state the main convergence lemma in the noisy case.
This lemma is an adaptation of Lemma~\ref{lemma:phase1} to the noisy case.
The proof of this lemma is similar to the proof of Lemma~\ref{lemma:phase1}.
\begin{lemma}\label{lemma:phase1noise}  
    There are absolute constants $c_1,c_2,c_3, c_4 >0$ 
    chosen sufficiently small
    such that the following statement holds.
    Assume that the spectral initialization $\UU_0$ satisfies 
    \begin{align}
        \specnorm{
        \XXstar - \UU_0 \UU_0^\top
        }
    \le 
     c_1 \sigma_{\min} \left( \XXstar \right)
    \label{assump:localconv1noise}
    \end{align}
    and that for every $\ww \in \epscover$ we have that
    \begin{align}
        \fronorm{
        \UU_{0} \UU_{0}^\top
        -
        \UU_{0,\ww} \UU_{0,\ww}^\top
        }
        \le c_2 \sigma_{\min} \left( \XXstar \right)
        .
    \label{assump:localconv2noise}
    \end{align}
    Furthermore, we assume that 
   \begin{align}
    \max 
    \left\{
        \delta;  8 \sqrt{\frac{2rd}{m}}
    \right\}
    &\le \frac{\constthree}{\kappa}, 
    \label{assump:localconv4noise}
   \end{align}
   where $\delta= \delta_{4r+2}$ denotes the Restricted Isometry Property of order $4r+2$.
   In addition, assume that $\mu \le \frac{c_4}{\kappa \specnorm{\XXstar}}$, and $\sigma\leq \frac{c_5 \sigma_{\min}(\XXstar)}{\sqrt{d}}$ for a sufficiently small constant $c_5>0$.
Moreover, assume that the statements of Lemma \ref{lem:noise_spec}
and Lemma \ref{lemma:candesyaniv} hold.
 Then take $T=\frac{1}{2} \cdot  6^d$. 
 With probability at least $1-C_1\exp(-d)$, for every iteration $t$ with $0 \le t \le T$ it holds that
\begin{align}\label{ineq:phase1noise_convergence}
    \fronorm{\XXstar - \UUt \UUtT}
    \leq  
    &3 \sqrt{r} \left(1 - \frac{\mu  \sigma_{\min} (\XXstar) }{16}\right)^t
    \specnorm{
       \XXstar - \UU_0 \UU_0^\top
    }+C_2\sigma\kappa \sqrt{rd},
\end{align}
where $\UUstar \in \R^{n \times r}$ denotes a matrix
which satisfies $\UUstar \UUstar^\top =\XXstar$, and $C_1, C_2 > 0$ are absolute constants.
\end{lemma}
\begin{proof}[Proof of Lemma \ref{lemma:phase1noise}]
As in the noiseless case,
we prove by induction that for all iterations $t$ with $ 0 \le t \le T$ the following inequalities hold:    
\begin{align}
    \fronorm{ \VXXT \left( \XXstar - \UUt \UUtT \right)} 
    \le
    &\left(1 - \frac{\mu}{16} \sigma_{\min} (\XXstar) \right)^{t} 
    \fronorm{\VXXT \left( \XXstar - \UU_0 \UU_0^\top  \right)}
    \\
    &+5 c\mu \sigma\sqrt{rd} \specnorm{\XXstar} 
    \sum_{\ell=0}^{t-1} 
    \left(1 - \frac{\mu}{16} \sigma_{\min} \left( \XXstar \right) \right)^\ell 
    ,\label{ineq:induction3noise}\\
    \specnorm{\VXXT \left(\XXstar - \UUt \UUtT \right)} 
    \le
    & c_1 \sigma_{\min} \left( \XXstar \right)  ,\label{ineq:induction4noise}\\
    \specnorm{\VXXPT \VUUt}
    \le 
    &   \sqrt{2} c_1 ,
    \label{ineq:induction6noise}\\
    \specnorm{\XXstar - \UUt \UUtT} 
    \le
    & 3 c_1 \sigma_{\min} \left( \XXstar \right) ,\label{ineq:induction5noise}
\end{align}
and, for every $\ww \in \epscover$,
\begin{align}
    \fronorm{
        \VXXT
        \left(\UUt \UUtT
        -
        \UUtw \UUtw^\top\right)
    }
    \le
    &  c_2 \sigma_{\min} \left( \XXstar \right) , \label{ineq:induction1noise}\\
    \fronorm{ \UUt \UUtT - \UUtw \UUtw^\top }
    \le 
    &  3 c_2 \sigma_{\min} (\XXstar) .\label{ineq:induction2noise}
\end{align}
The constants $c_1, c_2 >0$ are the same as in assumptions \eqref{assump:localconv1noise}
and \eqref{assump:localconv2noise}
and are thus, in particular, independent of the iteration number $t$.
With the exact same arguments as in the proof of Lemma~\ref{lemma:phase1} we obtain that
these inequalities hold for $t=0$.

For the induction step, assume now that these inequalities hold for some $t$.
With the exact same arguments as in the proof of Lemma~\ref{lemma:phase1} we obtain that
\begin{align}
    \specnorm{\left(\Aops - \IdOp \right) \left(\XXstar - \UUt \UUtT \right) }
    \le
    \frac{10 c_3}{\kappa} \sigma_{\min} \left( \XXstar \right), \label{ineq:intern87noise}
\end{align}
Next, we note that from Lemma \ref{lemma:localconv_noise} applied with 
$\triplenorm{\cdot}= \fronorm{\cdot} $ it follows that 
\begin{align*}
    &\fronorm{\VXXT \left(
        \UUtplus \UUtplus^\top - \XXstar
    \right)}\\
    \le
    &\left(1 - \frac{\mu}{8} \sigma_{\min} \left( \XXstar \right) \right) 
        \fronorm{\VXXT \left( \XXstar - \UUt \UUtT \right) }
        +
        5 \mu \specnorm{\XXstar}
        \fronorm{
            \left[\left(\Aops - \IdOp \right)
            \left(\XXstar - \UUt \UUtT\right)\right]
            \VUUt
        } \\
    &+5 \mu \specnorm{\XXstar}
        \fronorm{
         \left(   \Aop^* (\xxi) \right) \VUUt}\\
    \le
    &\left(1 - \frac{\mu}{16} \sigma_{\min} \left( \XXstar \right) \right) 
        \fronorm{\VXXT \left( \XXstar - \UUt \UUtT \right) }
        +5 \mu \specnorm{\XXstar}
        \fronorm{
         \left(   \Aop^* (\xxi) \right) \VUUt}\\
    \le
    &\left(1 - \frac{\mu}{16} \sigma_{\min} \left( \XXstar \right) \right) 
        \fronorm{\VXXT \left( \XXstar - \UUt \UUtT \right) }
        +5 \mu \sqrt{r} \specnorm{\XXstar}\specnorm{\Aop^* (\xxi)}\\
    \overleq{(a)}
    &\left(1 - \frac{\mu}{16} \sigma_{\min} \left( \XXstar \right) \right) 
        \fronorm{\VXXT \left( \XXstar - \UUt \UUtT \right) }
        +5 c\mu \sigma \sqrt{rd} \specnorm{\XXstar} .
\end{align*}
Here we argued as in the proof of Lemma~\ref{lemma:phase1}
and in inequality $(a)$ we have used additionally Lemma~\ref{lemma:candesyaniv}.
Thus, using the induction assumption
obtain that
\begin{align*}
 &\fronorm{\VXXT \left(
        \UUtplus \UUtplus^\top - \XXstar
    \right)}\\   
\le 
&\left(1 - \frac{\mu}{16} \sigma_{\min} \left( \XXstar \right) \right)^{t+1}
        \fronorm{\VXXT \left( \XXstar - \UU_0 \UU_0^\top \right) }
        +5  c\mu \sigma \sqrt{rd} \specnorm{\XXstar} 
        \sum_{\ell=0}^t 
        \left(1 - \frac{\mu}{16} \sigma_{\min} \left( \XXstar \right) \right)^\ell.
\end{align*}
Thus, we see that inequality \eqref{ineq:induction3noise} holds for $t+1$.

To prove inequality \eqref{ineq:induction4noise} for $t+1$
we argue as in the proof of Lemma~\ref{lemma:phase1}
and use additionally
Lemma \ref{lemma:localconv_noise}
with $ \triplenorm{\cdot} = \specnorm{\cdot} $ that
\begin{align}
    &\specnorm{\VXXT \left(
        \UUtplus \UUtplus^\top - \XXstar
    \right)}\\
    \le
    &\left(1 - \frac{\mu}{8} \sigma_{\min} \left( \XXstar \right) \right) 
    \specnorm{\VXXT \left( \XXstar - \UUt \UUtT \right) }
    +
     5 \mu \specnorm{\XXstar}
    \specnorm{
         \left(\Aops - \IdOp \right)
            \left(\XXstar - \UUt \UUtT\right)
        }\\
    &+ 5 \mu \specnorm{\XXstar}
        \specnorm{
         \left(   \Aop^* (\xxi) \right) \VUUt}\\
    \le
    &\left(1 - \frac{\mu}{8} \sigma_{\min} \left( \XXstar \right) \right) 
    c_1 \sigma_{\min} (\XXstar) 
    +
    50 c_3 \mu  \sigma_{\min}^2 \left( \XXstar \right)+5c\mu\sigma\sqrt{d}\|\XXstar\|
    \le
    c_1 \sigma_{\min} \left( \XXstar \right).
    \label{ineq:inter789_noise}
\end{align}
We observe that Lemma \ref{lemma:convaprioribound_noise}  yields the 
a-priori bound
\begin{equation*}
    \specnorm{\XXstar- \UUtplus \UUtplus^\top}
    \le
    \left( 1 - \frac{1}{\sqrt{2}} \right)
    \sigma_{\min} \left( \XXstar \right).
\end{equation*}

Then inequality \eqref{ineq:induction6noise} for $t+1$ follows
with the exact same argument as in the noiseless case.
Also inequality \eqref{ineq:induction5noise} for $t+1$ 
can be shown in the same way as in the proof of Lemma~\ref{lemma:phase1}.

Next, we apply Lemma \ref{lemma:auxsequencecloseness_noise}
and it follows that
\begin{align}
    &\fronorm{
        \VXXT
        \left(\UUtplus \UUtplus^\top
        -
        \UUtplusw \UUtplusw^\top\right)
    }\\
    \le
    &\left( 1- \frac{\mu \sigma_{\min} (\XXstar)}{16} \right)
    \fronorm{ \VXXT \left( \UUt \UUt^\top - \UUtw \UUtw^\top \right) }
    +
    \mu  
    \sigma_{\min}(\XXstar)\specnorm{\XXstar - \UUt \UUtT} +20 \mu \sqrt{d} \sigma \specnorm{\XXstar} \\
    \le 
    &\left( 1- \frac{\mu \sigma_{\min} (\XXstar)}{16} \right)
   c_2 \sigma_{\min} \left( \XXstar \right)
    +
    3 c_1 \mu   \sigma^2_{\min}(\XXstar) +20 \mu \sqrt{d} \sigma \specnorm{\XXstar} \\
    \overleq{(a)} 
    & c_2 \sigma_{\min} \left( \XXstar \right).
\end{align}
Inequality $(a)$ holds since we can choose that $ c_1 \le  \frac{c_2}{48} $
and   from the assumption on $\mu$ and $\sigma$ we have 
$\mu \sigma \sqrt{d}\specnorm{\XXstar} \le c_4c_5\sigma_{\min}(\XXstar)$.
This proves inequality \eqref{ineq:induction1noise}.
In order to prove inequality \eqref{ineq:induction2noise} for $t+1$,
we can apply Lemma \ref{lemma:auxdistanceweak_noise}
and Lemma \ref{lemma:auxcloseness1} and argue as in the noiseless case.
This completes the induction step.

\medskip

It remains to prove inequality \eqref{ineq:phase1noise_convergence}.
We note that 
\begin{align*}
    \fronorm{\XXstar - \UUt \UUtT}
    \overleq{(a)} 
    & 3 \fronorm{\VXXT \left( \XXstar - \UUt \UUtT \right)}\\
    \overleq{(b)}
    &3\left(1 - \frac{\mu  \sigma_{\min} (\XXstar) }{16}\right)^t
    \fronorm{
       \VXXT \left( \XXstar - \UU_0 \UU_0^\top \right)
    }\\
    &+15c \mu \sigma \sqrt{rd} \specnorm{\XXstar} 
    \sum_{\ell=0}^{t-1} 
    \left(1 - \frac{\mu}{16} \sigma_{\min} \left( \XXstar \right) \right)^\ell\\
    \overleq{(c)} 
    &3 \sqrt{r} \left(1 - \frac{\mu  \sigma_{\min} (\XXstar) }{16}\right)^t
    \specnorm{
       \XXstar - \UU_0 \UU_0^\top
    }+240c\sigma\kappa \sqrt{rd}.
\end{align*}
Inequality $(a)$ follows from 
Lemma \ref{lemma:localconvaux} 
with $ \triplenorm{\cdot} = \fronorm{\cdot}  $
which is applicable since we have shown by induction that \eqref{ineq:induction6noise} holds for $0 \le t \le T$.
Inequality $(b)$ follows from \eqref{ineq:induction1noise}
and for inequality $(c)$
we have used the geometric sum over $\ell$.
Thus, the proof of Lemma \ref{lemma:phase1noise} is complete.
\end{proof}

\subsection{Proof of Theorem~\ref{thm:mainnoise}}\label{sec:proof_mainnoise}

We are now ready to prove Theorem~\ref{thm:mainnoise}.
\begin{proof}[Proof of Theorem~\ref{thm:mainnoise}]
In the following 
$c>0$ denotes a sufficiently small absolute constant.
First, by Lemma \ref{lem:rank_RIP} due to our assumption $m \gtrsim rd \kappa^2$, with probability $1-\exp(-d)$
the measurement operator $\mathcal{A}$ satisfies the Restricted Isometry Property
of order $6r$ with a constant $\delta=\delta_{6r} \le \frac{c}{\kappa}$.

As in the noiseless case,
we choose an $\varepsilon$-net $ \epscover$ of the unit sphere in $\R^d$
with $\varepsilon=1/2$
such that $\vert \mathcal{N}_{\varepsilon} \vert \le 6^d$.
Thus, it follows from Lemma \ref{lemma:independencebound}
that
with probability at least $1-2\exp (-10d)$
it holds that
\begin{equation*}
    \vert   \innerproduct{\ww \ww^\top, \left(\Aops \right) \left( \Projwperp \left( \XXstar - \UUtw \UUtw^\top  \right)  \right)  }\vert \\
    \le 
    4 \sqrt{\frac{d}{m}} \twonorm{ \Aop \left( \Projwperp \left( \XXstar - \UUtw \UUtw^\top \right) \right) }
    \end{equation*}
    for all $\ww \in \epscover$
    and for all $ 0 \le t \le 6^d/2$.

   From  Lemma~\ref{lemma:spectralinitializationnoise} and Lemma~\ref{lem:rank_RIP}, when $m\geq C\kappa^2 rd$ and $\sigma<\frac{c_2\sigma_{\min}(\XXstar)}{\sqrt d}$ for sufficiently large $C$ and sufficiently small $c_2>0$, all conditions in Lemma~\ref{lemma:phase1noise} hold with probability $1-C'\exp(-d)$.
In particular, we have that 
\[ \specnorm{
       \XXstar - \UU_0 \UU_0^\top
   }
   \le C\kappa\sigma_{\min}(\XXstar) \sqrt{\frac{rd}{m}} 
   +c\sqrt{d} \sigma
   \ll
   \sigma_{\min} (\XXstar).
\]
   Then by Lemma~\ref{lemma:phase1noise} we obtain that
    \begin{align}\label{eq:noisy_convergence_bound}
    \fronorm{\XXstar - \UUt \UUtT}
    \leq  
    &3 \sqrt{r} \left(1 - \frac{\mu  \sigma_{\min} (\XXstar) }{16}\right)^t
    \specnorm{
       \XXstar - \UU_0 \UU_0^\top
    }+C_2\sigma\kappa \sqrt{rd}.
\end{align}
By taking 
\[  T 
   := \Big\lceil \frac{\log (\sigma \kappa \sqrt{d}/\sigma_{\min}(\XXstar))}{\log(1-\mu \sigma_{\min}(\XXstar)/16)} \Big\rceil
   \asymp
   \frac{ \log \left( \frac{\sigma_{\min} (\XXstar)}{ \kappa \sqrt{d} \sigma } \right)  }{ \mu \sigma_{\min} (\XXstar)  },\]
we can simplify \eqref{eq:noisy_convergence_bound} to obtain \eqref{eq:noisy_T}.
Note that the condition 
\begin{align}
    \sigma \geq \frac{\sigma_{\min}(\XXstar)}{\kappa\sqrt{d} } \exp(-\mu \sigma_{\min}(\XXstar) 3^d ),
\end{align}
guarantees that $ T \le \frac{1}{2} \cdot 6^d $.
This finishes the proof.
\end{proof}

 \end{document}